\documentclass[conference]{IEEEtran}
\IEEEoverridecommandlockouts
\usepackage{cite}
\usepackage{amsmath,amssymb,amsfonts}
\usepackage{graphicx}
\usepackage{textcomp}
\usepackage{xcolor}

\usepackage{url}
\usepackage{placeins}
\usepackage{hyperref}
\usepackage{float}
\usepackage{tabularx,colortbl}
\usepackage{ifthen}
\usepackage{caption,subcaption}
\usepackage{amsmath}
\usepackage{booktabs}
\renewcommand{\thetable}{\Roman{table}}
\usepackage{cleveref}

\usepackage{enumitem}
\usepackage{stfloats}
\usepackage{wrapfig}
\usepackage{verbatim}
\usepackage{tabularx}
\usepackage{booktabs}
\usepackage{multirow}
\usepackage{caption}
\usepackage{flushend}

\usepackage{amsthm}
\usepackage{algorithm}
\usepackage{algorithmicx}
\usepackage{algpseudocode}
\usepackage{adjustbox}
\usepackage{titlesec}

\def\BibTeX{{\rm B\kern-.05em{\sc i\kern-.025em b}\kern-.08em
    T\kern-.1667em\lower.7ex\hbox{E}\kern-.125emX}}
\begin{document}

\title{Learning Conformal Abstention Policies for Adaptive Risk Management in Large Language and Vision-Language Models\\
}
\author{Sina Tayebati$^1$, Divake Kumar$^1$, Nastaran Darabi$^1$, Dinithi Jayasuriya$^1$, Ranganath Krishnan$^2$, Amit Ranjan Trivedi$^1$ \vspace{2mm} \\
$^1$University of Illinois at Chicago, $^2$Intel Labs}

\maketitle

\begin{abstract}
Large Language and Vision-Language Models (LLMs/VLMs) are increasingly used in safety-critical applications, yet their opaque decision-making complicates risk assessment and reliability. Uncertainty quantification (UQ) helps assess prediction confidence and enables abstention when uncertainty is high. Conformal prediction (CP), a leading UQ method, provides statistical guarantees but relies on static thresholds, which fail to adapt to task complexity and evolving data distributions, leading to suboptimal trade-offs in accuracy, coverage, and informativeness. To address this, we propose \textit{learnable conformal abstention}, integrating reinforcement learning (RL) with CP to optimize abstention thresholds dynamically. By treating CP thresholds as adaptive actions, our approach balances multiple objectives, minimizing prediction set size while maintaining reliable coverage. Extensive evaluations across diverse LLM/VLM benchmarks show our method outperforms Least Ambiguous Classifiers (LAC) and Adaptive Prediction Sets (APS), improving accuracy by up to 3.2\%, boosting AUROC for hallucination detection by 22.19\%, enhancing uncertainty-guided selective generation (AUARC) by 21.17\%, and reducing calibration error by 70\%–85\%. These improvements hold across multiple models and datasets while consistently meeting the 90\% coverage target, establishing our approach as a more effective and flexible solution for reliable decision-making in safety-critical applications. The code is available at \textcolor{magenta}{\emph{\url{https://github.com/sinatayebati/vlm-uncertainty}}}.
\end{abstract}


\section{Introduction}
Large Language and Vision-Language Models (LLMs/VLMs) are rapidly becoming indispensable in safety-critical applications, from autonomous systems \cite{zhou2024vision} to healthcare diagnostics \cite{yildirim2024multimodal}. Their ability to process and interpret information across visual and textual modalities presents unprecedented opportunities for complex decision-making. However, their internal workings remain opaque, making it challenging to identify biases, vulnerabilities, and unintended consequences, which hinders effective risk assessment and mitigation. Traditional risk management frameworks, designed for static systems with well-defined rules, struggle to keep pace with the evolving capabilities and emergent behaviors of these models \cite{abdali2024securing}. As decision support systems increasingly rely on LLM/VLM, equipping them with robust mechanisms to identify and manage their prediction risks has become crucial.

Uncertainty quantification (UQ) of LLM/VLM has therefore gained significant attention for assessing prediction reliability and enabling abstention--allowing models to defer decisions when uncertainty is high. However, state-of-the-art UQ methods like conformal prediction (CP), while providing statistical guarantees, rely on static thresholds that fail to adapt to varying task complexities or evolving data distributions. Abstention strategies built on these methods therefore remain inflexible, treating abstention as a binary choice--predict or abstain \cite{yadkori2024mitigating}--without adapting to context. Consequently, state-of-the-art methods such as Least Ambiguous Classifiers (LAC) \cite{sadinle2019least} tend to produce overly narrow prediction sets, sacrificing coverage, while Adaptive Prediction Sets (APS) \cite{romano2020classification} generate excessively large sets. 

To address these limitations, we propose a framework for \textit{learnable conformal abstention}, where models dynamically adjust abstention decisions based on task complexity and evolving data distributions. By integrating reinforcement learning with conformal prediction, our approach enables adaptive thresholding that surpasses static methods in accuracy, coverage, and reliability. Extensive evaluations across diverse benchmarks demonstrate its effectiveness in improving risk management, selective abstention, and overall decision-making in safety-critical LLM/VLM applications. In particular, our learned policy boosts hallucination detection by up to 22\%, improves uncertainty-guided selective generation by more than 20\% in certain scenarios, and reduces expected calibration error by 70\%-85\% compared to standard conformal baselines. Notably, it also sustains at least 90\% coverage while reducing the average prediction set size.

\section{Background}\label{sec:related_work}

\textbf{Uncertainty Quantification (UQ) of Prediction Models:} Several approaches have been explored to capture and manage uncertainty in machine learning models. Conformal prediction \cite{vovk2005algorithmic, balasubramanian2014conformal, angelopoulos2021gentle}  provides a distribution-free, model-agnostic framework for generating prediction sets with statistical guarantees, with advances such as inductive, split, and cross-conformal prediction addressing different calibration strategies. Extensions include handling distribution shifts, sequential data, and active learning. Evidential learning captures uncertainty by modeling distributions over parameters or predictions, with applications in deep evidential regression, image classification, and reinforcement learning, and connections to belief function theory. Bayesian deep learning offers alternatives like Monte Carlo dropout \cite{gal2016dropout}, variational inference \cite{blundell2015weight}, deep ensembles \cite{lakshminarayanan2017simple}, and trainable calibration \cite{krishnan2020improving}, with recent work focusing on scalable inference for large models and flexible posterior distributions using normalizing flows. Additionally, calibration techniques such as test-time augmentation, temperature scaling, and Platt scaling further refine confidence estimates \cite{gawlikowski2023survey, kumar2019verified}. However, these methods primarily focus on UQ rather than actionable risk mitigation.

\textbf{Conformal Prediction (CP):} We characterize the proposed framework within CP, a distribution-free, model-agnostic UQ approach \cite{vovk2005algorithmic, balasubramanian2014conformal, angelopoulos2021gentle}. CP transforms model uncertainty estimates into statistically rigorous measures, generating a prediction set that includes the true label with a predefined error probability. The set’s size reflects model uncertainty, with larger sets indicating higher uncertainty. Formally, given a classification model \(f\) mapping input \(X\) to one of \(K\) classes \(Y = \{1, \ldots, K\}\), CP constructs a prediction set \(C(X_t) \subseteq Y\) for a test instance \(X_t\) satisfying:  
\[
P(Y_t \in C(X_t)) \geq 1 - \alpha,
\]
where \(\alpha \in (0,1)\) is the target error rate. Coverage is determined using a calibration dataset \(D_{\mathrm{cal}} = \{(X_c^{(i)}, Y_c^{(i)})\}_{i=1}^n\). A conformal score function \(s(X, Y) \in \mathbb{R}\) is defined, where higher values indicate greater uncertainty. Calibration scores are computed as \(s_1, \ldots, s_n\), and the threshold \(\hat{q}\) is set using the \((1-\alpha)\)-quantile:  
\[
\hat{q} = \text{quantile}_{\frac{\lceil(n + 1)(1 - \alpha)\rceil}{n}} \{s_1, \ldots, s_n\}.
\]
The final prediction set for \(X_t\) is:  
\[
C(X_t) = \{Y' \in Y : s(X_t, Y') \leq \hat{q}\}.
\]
A common uncertainty heuristic is the softmax score, which estimates class probabilities but often misaligns with true uncertainty. Two CP-based scoring functions have been proposed to address this: \textit{(i) Least Ambiguous Classifiers (LAC)} \cite{ye2024benchmarking, sadinle2019least} use:  
\(
s(X, Y) = 1 - f(X)_Y,
\)
where \(f(X)_Y\) is the softmax score for \(Y\). LAC minimizes prediction set size but may yield overly narrow or broad sets.  
\textit{(ii) Adaptive Prediction Sets (APS)} \cite{ye2024benchmarking, romano2020classification} define:  
\(
s(X, Y) = \sum_{\{Y' \in Y : f(X)_{Y'} \geq f(X)_Y\}} f(X)_{Y'},
\)
summing softmax scores of classes ranked above or equal to \(Y\). APS mitigates LAC’s limitations but often produces larger prediction sets, reducing informativeness.

\section{Learning Conformal Abstention Policies} \label{Conformal_Abstention}
We propose a novel framework for learning an abstention policy that leverages conformal prediction to generate uncertainty-aware prediction sets with statistical guarantees. The proposed framework, conformalized abstention policy (CAP), allows three possible outcomes per query: a single prediction, a set of plausible predictions, or abstention, balancing informativeness and risk based on prediction confidence. We formulate this as a reinforcement learning (RL) problem, optimizing the CP hyperparameters (\(\alpha\) and \(\beta\)) as actions using the REINFORCE \cite{zhang2021sample}.

First, to quantify the uncertainty of LLM/VLM responses, we use a nonconformity measure based on softmax probabilities. Given an input \(\mathbf{x}\), the model produces logits \(\ell = [\ell_1, \ell_2, \ldots, \ell_K]\) for \(K\) classes, which are converted to probabilities \(p_i(\mathbf{x})\) using softmax. Using a calibration dataset \(\mathcal{D}_{\mathrm{cal}} = \{(X_{\mathrm{cal}}^{(i)}, Y_{\mathrm{cal}}^{(i)})\}_{i=1}^n\), where \(y_i\) is the ground-truth label, we compute the \emph{nonconformity score} for each sample as  
\(
\mathrm{score}(\mathbf{x}_i) = 1 - p_{y_i}(\mathbf{x}_i).
\)
This score quantifies nonconformity, with higher values indicating greater uncertainty. Traditional conformal prediction defines a single threshold as the \((1-\alpha)\)-quantile of the calibration scores:  
\[
\hat{q} = \mathrm{Quantile}_{1-\alpha}\bigl(\{ s_i \}_{i=1}^n\bigr).
\]
We extend this by introducing two thresholds, \(\hat{q}_{\mathrm{predict}}\) and \(\hat{q}_{\mathrm{abstain}}\), computed as:  
\[
\hat{q}_{\mathrm{predict}} = \text{quantile} \left( \{s_1, \ldots, s_n\}, \frac{\left\lceil (n + 1)(1 - \alpha) \right\rceil}{n} \right),
\]
\[
\hat{q}_{\mathrm{abstain}} = \text{quantile} \left( \{s_1, \ldots, s_n\}, \frac{\left\lceil (n + 1)(1 - \beta) \right\rceil}{n} \right).
\]

These thresholds partition the nonconformity score of each test sample into three regimes:
\begin{enumerate}[noitemsep, topsep=0pt, partopsep=0pt]
    \item \(\mathrm{score}(\mathbf{x}) < \hat{q}_{\mathrm{predict}} \;\;\Rightarrow\;\) \emph{Single best prediction}.
    \item \(\hat{q}_{\mathrm{predict}} \;\le\; \mathrm{score}(\mathbf{x}) \;<\; \hat{q}_{\mathrm{abstain}} \;\;\Rightarrow\;\) \emph{Set prediction}.
    \item \(\mathrm{score}(\mathbf{x}) \;\ge\; \hat{q}_{\mathrm{abstain}} \;\;\Rightarrow\;\) \emph{Abstain}.
\end{enumerate}

\textbf{Action Probabilities and Stochastic Decisions:}  
We extend the deterministic three-regime decision with a \emph{stochastic policy} that maps the nonconformity score to action probabilities. Let:  
\(
    s(\mathbf{x}) = \mathrm{score}(\mathbf{x}) = 1 - \max_i\, p_i(\mathbf{x}).
\)
Using thresholds \(\hat{q}_{\mathrm{predict}}\) and \(\hat{q}_{\mathrm{abstain}}\), the action probabilities are defined as:  
\begin{align*}
    p_{\mathrm{single}}\bigl(s(\mathbf{x})\bigr) 
    &= \sigma\bigl(\,-c\,\bigl[s(\mathbf{x}) - \hat{q}_{\mathrm{predict}}\bigr]\bigr), 
    \label{eq:psingle}\\
    p_{\mathrm{abstain}}\bigl(s(\mathbf{x})\bigr) 
    &= \sigma\bigl(\,c\,\bigl[s(\mathbf{x}) - \hat{q}_{\mathrm{abstain}}\bigr]\bigr), 
\end{align*}
where \(\sigma(z) = 1/(1 + e^{-z})\) is the sigmoid function, and \(c > 0\) is a scaling constant. The probability of a set prediction is:  
\(
    p_{\mathrm{set}}\bigl(s(\mathbf{x})\bigr) = 1 - p_{\mathrm{single}}\bigl(s(\mathbf{x})\bigr) - p_{\mathrm{abstain}}\bigl(s(\mathbf{x})\bigr).
\)
For each test point, we stochastically select from \(\{\text{single}, \text{set}, \text{abstain}\}\) based on \(\{p_{\mathrm{single}}, p_{\mathrm{set}}, p_{\mathrm{abstain}}\}\), capturing model uncertainty.

\textbf{Reinforcement Learning and Abstention Policy:}  \label{rl_abstention}
To dynamically adjust the confidence levels \((\alpha, \beta)\) for optimal performance, we employ a policy-based reinforcement learning approach using the REINFORCE algorithm. The policy network \(\pi_{\theta}(\alpha, \beta)\) learns a distribution over these parameters. We treat \(\alpha\) and \(\beta\) as \emph{actions} sampled from a multivariate Gaussian distribution defined by the policy network. Let:  
\(
    (\boldsymbol{\mu}_{\theta}, \boldsymbol{\sigma}_{\theta}) 
    \;=\; f_{\theta}(\mathbf{s}),
\)
where \(\mathbf{s}\) is the current \emph{state}, and \(f_{\theta}\) is a neural network mapping \(\mathbf{s}\) to the mean and standard deviation vectors \((\boldsymbol{\mu}_{\theta}, \boldsymbol{\sigma}_{\theta})\). Specifically, 
\[
    \alpha \sim \mathcal{N}\bigl(\mu_{\theta}^{(\alpha)}, \sigma_{\theta}^{(\alpha)\,2}\bigr), 
    \quad
    \beta \sim \mathcal{N}\bigl(\mu_{\theta}^{(\beta)}, \sigma_{\theta}^{(\beta)\,2}\bigr).
\]
At each iteration, the process involves: (1) sampling \(\alpha, \beta\) from the learned distribution, (2) computing the thresholds \(\hat{q}_{\mathrm{predict}}, \hat{q}_{\mathrm{abstain}}\) on a calibration set, (3) evaluating performance on a test set, and (4) using the performance-based \emph{cost} as a \emph{reward} signal to update \(\pi_{\theta}\) via REINFORCE.

\textbf{Cost Function and Reward Design:}  
We define a scalar cost function \( C(\alpha, \beta) \) to balance multiple objectives: maximizing accuracy while ensuring well-calibrated uncertainty estimates, avoiding unnecessary set predictions or abstentions, and maintaining coverage guarantees. Let \( \mathrm{acc} \) be the fraction of correct predictions, \( \mathrm{abstention} \) the fraction of abstained samples, and \( \mathrm{avgSet} \) the average prediction set size. Additionally, \( \mathrm{coverage} = 1 - \mathrm{abstention} \), and \( \mathrm{div} \) is an entropy-based term quantifying the balance among single predictions, set predictions, and abstentions. The cost function is defined as:  
\begin{align*}
C(\alpha, \beta) &= (1 - \mathrm{acc}) + \lambda_{1}\,\mathrm{avgSet} + \lambda_{2}\,\mathrm{abstention} \\
&\quad - \lambda_{3}\,\mathrm{coverage} - \lambda_{4}\,\mathrm{div}.
\end{align*}
where \( \lambda_{1}, \ldots, \lambda_{4} \) are hyperparameters controlling trade-offs between these objectives. The corresponding \emph{reward} is simply the negative cost:  
    $R(\alpha, \beta) = -C(\alpha, \beta)$.

\textbf{REINFORCE Update:} We employ a Monte Carlo policy gradient method to update the policy parameters by maximizing the expected reward \(J(\theta) = \mathbb{E}_{\tau \sim \pi_{\theta}} [R(\tau)]\), where \(\tau\) is a trajectory of states and actions, and \(R(\tau)\) is the corresponding reward. Each \emph{episode} corresponds to evaluating the doubly conformalized prediction with a specific set of actions on the test set. For a batch of sampled actions \(\{\alpha_t, \beta_t\}\), the gradient of the expected reward is:
\begin{equation*}
    \nabla_{\theta} J(\theta)
    \;=\; 
    \mathbb{E}_{(\alpha,\beta)\sim \pi_{\theta}} 
    \bigl[
        \nabla_{\theta} \log \pi_{\theta}(\alpha,\beta) \cdot R(\alpha,\beta)
    \bigr].
\end{equation*}
This expectation is approximated by sampling trajectories (or averaging over a minibatch of \(\alpha, \beta\)) and updating the policy parameters as:
\begin{equation*}
\theta_{t+1} = \theta_t + \eta \cdot R_t \nabla_{\theta} \log \pi_{\theta}(a_t|s_t),
\end{equation*}
where \(\eta\) is the learning rate, \(R_t\) is the reward after action \(a_t\) in state \(s_t\), and \(\log \pi_{\theta}(a_t|s_t)\) is the log-probability of taking \(a_t\) under the policy. The learned \(\hat{\alpha}\) and \(\hat{\beta}\) minimize the cost with respect to coverage, set size, accuracy, and abstentions, coupling conformal prediction thresholds to an RL objective and enabling principled trade-offs between predictive certainty and abstention for LLM outputs. \cref{alg:conformal_rl} in Appendix \ref{rl_algorithm} summarizes the training of our proposed adaptive conformal method and abstention policy. 

\section{Experiments and Results}
We conducted a thorough empirical evaluation to benchmark our proposed CAP framework against the comparative least ambiguous set-valued classifier (LAC) and adaptive prediction sets (APS). The experiments focus on multiple-choice question answering (MCQA) tasks, assessing six key metrics: confidence ranking for hallucination detection, uncertainty-guided selective generation, coverage, set size, calibration, and accuracy. This evaluation systematically measures the effectiveness of the proposed abstention policy and the reliability of uncertainty estimates.

\subsection{Experimental Settings}

\textbf{Datasets:}~~We used a diverse collection of ten benchmark LLM/VLM datasets. These datasets are designed for multiple-choice question-answering (MCQA) across various reasoning tasks and uncertainty scenarios. For VLMs, we employ five datasets: (i) \textbf{MMBench} \cite{liu2025mmbench}, a multi-modal benchmark with 4,000 questions spanning perception and reasoning tasks, standardized to four options; (ii) \textbf{OODCV-VQA} \cite{zhao2022ood}, focusing on out-of-distribution instance counting via its ``Digits" subset, expanded to four options; (iii) \textbf{ScienceQA} \cite{lu2022learn}, containing 3,952 image-based questions across natural and social sciences; (iv) \textbf{SEEDBench} \cite{li2023seed}, evaluating visual understanding (e.g., object localization) with 14,233 questions; and (v) \textbf{AI2D} \cite{kembhavi2016diagram}, featuring 15,000 diagram-based science questions extended to six options. All datasets are reformatted to multiple-choice questions (MCQA) with four or six options to assess uncertainty handling.  

For LLMs, we evaluate on five tasks: (i) \textbf{MMLU} \cite{hendrycks2020measuring}, a question-answering benchmark spanning 57 academic subjects; (ii) \textbf{CosmosQA} \cite{huang2019cosmos}, focusing on reading comprehension requiring contextual inference; (iii) \textbf{HellaSwag} \cite{zellers2019hellaswag}, assessing commonsense inference for event followup prediction; (iv) \textbf{HaluDial} \cite{li2023halueval}, evaluating dialogue response selection from knowledge-grounded conversations; and (v) \textbf{HaluSum} \cite{li2023halueval}, testing document summarization on news articles. Each dataset is standardized to six options (including ``I don’t know" and ``None of the above") to align with uncertainty-aware evaluation protocols. This selection ensures diverse assessment of LLM capabilities in knowledge recall, reasoning, and abstention under ambiguity.

\textbf{Models:}~~We evaluated on a diverse set of LLM/VLM models with parameter scales ranging from 2.7B to 34B. For VLMs, the main body of the paper includes results for the LLaVA-v1.6 series (34B, 13B, and 7B parameters) \cite{liu2023improved}. Additional state-of-the-art VLMs—such as the lightweight MoE-LLaVA-Phi2 2.7B \cite{lin2024moe}, Monkey-Chat 7B \cite{li2024monkey}, InternLM-XComposer2-VL 7B \cite{dong2024internlm}, Yi-VL 6B \cite{young2024yi}, CogAgent-VQA 7B \cite{hong2024cogagent}, MobileVLMV2 \autoref{ap_exp}.  

For LLMs, the main body presents results for the Yi 34B model \cite{young2024yi} and the Qwen series (7B and 14B parameters) \cite{bai2023qwen}. Results for the Llama-2 foundation model series (7B and 13B parameters) are included in \autoref{ap_exp}.

\begin{table*}[h!]
\centering
\scriptsize
\setlength{\tabcolsep}{5.5pt}
\caption{Comparison of CAP (Ours) with Least Ambiguous Classifiers (LAC) \cite{sadinle2019least} and Adaptive Prediction Sets (APS) \cite{romano2020classification}. Models include VLMs and LLMs, assessed across datasets using AUROC (Hallucination Detection) and AUARC (Uncertainty-Guided Selective Generation). Best values are in \textbf{bold}.}
\label{tab:auroc_auarc}
\begin{tabular}{lccccccccccccc}
\toprule
\textbf{Models} & \textbf{Method} & \multicolumn{6}{c}{\textbf{AUROC} ↑ (Hallucination Detection)} & \multicolumn{6}{c}{\textbf{AUARC} ↑ (Uncertainty guided selective generation)} \\
\cmidrule(lr){3-8} \cmidrule(l){9-14}
\textbf{VLMs} & & \textbf{MMB} & \textbf{OOD} & \textbf{SQA} & \textbf{SB} & \textbf{AI2D} & \textbf{Avg.} & \textbf{MMB} & \textbf{OOD} & \textbf{SQA} & \textbf{SB} & \textbf{AI2D} & \textbf{Avg.} \\
\midrule
LLaVA-v1.6-34B
& APS & 0.7173 & 0.6962 & 0.7244 & 0.5566 & 0.8404 & 0.7070 & 0.9583 & 0.9138 & 0.9275 & 0.9155 & 0.9283 & 0.9287 \\
& LAC & 0.7837 & 0.7003 & 0.8000 & 0.5626 & 0.8476 & 0.7388 & 0.9412 & 0.9021 & 0.9099 & 0.8830 & 0.9192 & 0.9111 \\
& Ours & \textbf{0.8041} & \textbf{0.7849} & \textbf{0.8606} & \textbf{0.6512} & \textbf{0.8989} & \textbf{0.8000} & \textbf{0.9791} & \textbf{0.9717} & \textbf{0.9813} & \textbf{0.9441} & \textbf{0.9913} & \textbf{0.9735} \\
\cmidrule(lr){1-14}
LLaVA-v1.6-13B
& APS & 0.4930 & 0.5901 & 0.5281 & 0.4854 & 0.7775 & 0.5748 & 0.9566 & 0.8759 & 0.9444 & 0.9307 & 0.9142 & 0.9244 \\
& LAC & 0.6835 & 0.5919 & 0.5990 & 0.5038 & 0.7475 & 0.6251 & 0.9258 & 0.8512 & 0.8945 & 0.8791 & 0.8956 & 0.8892 \\
& Ours & \textbf{0.6382} & \textbf{0.7070} & \textbf{0.6663} & \textbf{0.6103} & \textbf{0.8083} & \textbf{0.6860} & \textbf{0.9761} & \textbf{0.9592} & \textbf{0.9565} & \textbf{0.9343} & \textbf{0.9838} & \textbf{0.9620} \\
\cmidrule(lr){1-14}
LLaVA-v1.6-7B
& APS & 0.6961 & 0.3424 & 0.6093 & 0.5699 & 0.8247 & 0.6085 & 0.9575 & 0.8712 & 0.9147 & 0.9239 & 0.8952 & 0.9125 \\
& LAC & 0.6849 & 0.4836 & 0.5555 & 0.4988 & 0.6930 & 0.5832 & 0.9212 & 0.8125 & 0.8671 & 0.8730 & 0.8691 & 0.8686 \\
& Ours & \textbf{0.7049} & \textbf{0.5643} & \textbf{0.6165} & \textbf{0.5919} & \textbf{0.7626} & \textbf{0.6480} & \textbf{0.9662} & \textbf{0.9253} & \textbf{0.9399} & \textbf{0.9353} & \textbf{0.9725} & \textbf{0.9478} \\
\midrule
\midrule
\textbf{LLMs} & & \textbf{HSwg} & \textbf{HDial} & \textbf{CQA} & \textbf{HSum} & \textbf{MMLU} & & \textbf{HSwg} & \textbf{HDial} & \textbf{CQA} & \textbf{HSum} & \textbf{MMLU} \\
\midrule
Yi-34B
& APS & 0.9109 & 0.5089 & 0.8370 & 0.5643 & 0.5883 & 0.6819 & 0.9735 & 0.7334 & 0.9373 & 0.7864 & 0.8806 & 0.8622 \\
& LAC & 0.9487 & 0.5650 & 0.9287 & 0.4181 & 0.6832 & 0.7087 & 0.9700 & 0.7140 & 0.9336 & 0.7529 & 0.8590 & 0.8459 \\
& Ours & \textbf{0.9726} & \textbf{0.7011} & \textbf{0.9649} & \textbf{0.6209} & \textbf{0.7425} & \textbf{0.8004} & \textbf{0.9973} & \textbf{0.9554} & \textbf{0.9963} & \textbf{0.9343} & \textbf{0.9669} & \textbf{0.9700} \\
\cmidrule(lr){1-14}
Qwen-14B
& APS & 0.8442 & 0.5296 & 0.7852 & 0.2611 & 0.6426 & 0.6125 & 0.9828 & 0.8326 & 0.9732 & 0.6266 & 0.8554 & 0.8541 \\
& LAC & 0.9182 & 0.4799 & 0.9132 & 0.1269 & 0.5445 & 0.5965 & 0.9748 & 0.8015 & 0.9657 & 0.5737 & 0.8216 & 0.8275 \\
& Ours & \textbf{0.9397} & \textbf{0.6175} & \textbf{0.9286} & \textbf{0.3510} & \textbf{0.6450} & \textbf{0.6964} & \textbf{0.9924} & \textbf{0.9323} & \textbf{0.9923} & \textbf{0.7146} & \textbf{0.9494} & \textbf{0.9162} \\
\cmidrule(lr){1-14}
Qwen-7B
& APS & 0.5638 & 0.3437 & 0.6107 & 0.2612 & 0.4829 & 0.4525 & 0.6853 & 0.7645 & 0.9000 & 0.5113 & 0.7459 & 0.7214 \\
& LAC & 0.4646 & 0.3542 & 0.7777 & 0.1654 & 0.4643 & 0.4452 & 0.6603 & 0.7275 & 0.8825 & 0.4754 & 0.7133 & 0.6918 \\
& Ours & \textbf{0.6380} & \textbf{0.4958} & \textbf{0.8037} & \textbf{0.4429} & \textbf{0.6053} & \textbf{0.5971} & \textbf{0.9325} & \textbf{0.9213} & \textbf{0.9831} & \textbf{0.6817} & \textbf{0.9450} & \textbf{0.8927} \\
\bottomrule
\end{tabular}
\end{table*}

\textbf{Evaluation Metrics:} ~~CAP is evaluated using the following metrics that assess both prediction quality and UQ, capturing its ability to produce single predictions, set predictions, or abstentions. The same metrics are applied to baseline conformal methods, including APS and LAC, following \cite{kostumov2024uncertainty, ye2024benchmarking}:

\textit{Accuracy:} For a test input \(X_t\) with true label \(Y_t\), let \(C(X_t)\) denote the generated prediction set. If a single prediction \(\hat{Y}_t\) is produced (e.g., in confident scenarios under ATCP), accuracy is binary: 1 if \(\hat{Y}_t = Y_t\), and 0 otherwise. For set predictions, accuracy is computed fractionally, inversely proportional to the size of \(C(X_t)\) when \(Y_t \in C(X_t)\).  

\textit{Coverage:} Coverage measures the fraction of instances where the correct label is included in the model’s output—either as a single prediction or within a prediction set. In setups with abstention, it also accounts for instances where the model successfully avoids making an incorrect explicit guess. This metric ensures the ground truth is not excluded from predictions. A key aspect of conformal prediction is meeting a predefined coverage guarantee, set at 90\% in our experiments.  

\textit{Set Sizes (SS):} Set Sizes measure the average number of labels in prediction sets, excluding single predictions and abstentions. This metric reflects model uncertainty, with larger sets indicating higher uncertainty and smaller sets implying greater confidence.  

\begin{table}[t!]
\centering
\scriptsize
\setlength{\tabcolsep}{3pt}
\caption{Coverage (\%) evaluation: Comparison of CAP (Ours) with LAC \cite{sadinle2019least} and APS \cite{romano2020classification}. CAP meets the 90\% coverage guarantee, \underline{underlined}, in instances.}
\label{tab:coverage}
\begin{tabular}{lccccccc}
\toprule
\multirow{2}{*}{\textbf{Model}} & \multirow{2}{*}{\textbf{Method}} & \multicolumn{5}{c}{\textbf{Coverage (\%)} ↑} \\
\cmidrule(lr){3-8}
\textbf{VLMs} & & \textbf{MMB} & \textbf{OOD} & \textbf{SQA} & \textbf{SB} & \textbf{AI2D} & \textbf{Avg.} \\
\midrule
\multirow{3}{*}{LLaVA-v1.6-34B}
& APS  & 98.26 & 94.87 & 98.08 & 95.81 & 97.48 & 96.90 \\
& LAC  & 90.73 & 91.42 & 88.67 & 90.23 & 90.21 & 90.25 \\
& Ours & \underline{93.97} & \underline{93.25} & \underline{93.07} & \underline{91.41} & \underline{95.46} & \underline{93.43} \\
\cmidrule(lr){1-8}
\multirow{3}{*}{LLaVA-v1.6-13B}
& APS  & 98.99 & 96.20 & 99.29 & 97.36 & 98.86 & 98.14 \\
& LAC  & 90.18 & 91.00 & 89.28 & 89.84 & 90.47 & 90.15 \\
& Ours & \underline{95.57} & \underline{92.48} & \underline{92.06} & \underline{90.67} & \underline{95.14} & \underline{93.18} \\
\cmidrule(lr){1-8}
\multirow{3}{*}{LLaVA-v1.6-7B}
& APS  & 98.45 & 97.89 & 97.88 & 96.74 & 96.19 & 97.43 \\
& LAC  & 89.26 & 89.10 & 89.83 & 90.19 & 89.65 & 89.61 \\
& Ours & \underline{92.96} & \underline{91.63} & \underline{90.49} & \underline{91.23} & \underline{93.41} & \underline{91.94} \\
\midrule
\midrule
\textbf{LLMs} & & \textbf{HSwg} & \textbf{HDial} & \textbf{CQA} & \textbf{HSum} & \textbf{MMLU} & \textbf{Avg.} \\
\midrule
\multirow{3}{*}{Qwen-7B}
& APS  & 92.12 & 95.24 & 98.92 & 90.18 & 96.24 & 94.54 \\
& LAC  & 89.64 & 90.90 & 90.44 & 90.12 & 90.66 & 90.35 \\
& Ours & \underline{91.96} & \underline{91.70} & \underline{95.68} & \underline{90.17} & \underline{91.32} & \underline{92.16} \\
\cmidrule(lr){1-8}
\multirow{3}{*}{Qwen-14B}
& APS  & 99.82 & 94.22 & 99.46 & 90.56 & 95.72 & 95.96 \\
& LAC  & 91.98 & 90.42 & 92.10 & 89.70 & 90.46 & 90.93 \\
& Ours & \underline{94.88} & \underline{90.96} & \underline{95.66} & \underline{90.32} & \underline{91.62} & \underline{92.68} \\
\cmidrule(lr){1-8}
\multirow{3}{*}{Yi-34B}
& APS  & 99.88 & 95.24 & 99.68 & 92.08 & 97.30 & 96.84 \\
& LAC  & 93.90 & 90.02 & 94.40 & 89.32 & 89.78 & 91.49 \\
& Ours & \underline{96.48} & \underline{92.56} & \underline{96.40} & \underline{90.82} & \underline{93.34} & \underline{93.92} \\
\bottomrule
\end{tabular}
\vspace{-20pt}
\end{table}

\begin{table*}[h!]
\centering
\scriptsize
\setlength{\tabcolsep}{5.5pt}
\caption{Evaluation of accuracy (\%) and set sizes:  Comparative analysis of CAP (Ours) with standard Least Ambiguous set-valued Classifiers (LAC) \cite{sadinle2019least}, and Adaptive Prediction Sets (APS) \cite{romano2020classification} methods. The table highlights that our proposed method achieves the highest average accuracy across datasets, while maintaining a balance in set sizes that avoids overly narrow or broad predictions observed in the baseline methods. Highest accuracy values are in \textbf{bold} and balanced set size values are \underline{underlined}.}
\label{tab:accuracy_vs_ss}
\begin{tabular}{lccccccccccccc}
\toprule
\textbf{Models} & \textbf{Method} & \multicolumn{5}{c}{\textbf{Accuracy (\%)} ↑} & & \multicolumn{5}{c}{\textbf{SS} ↓} \\
\cmidrule(lr){3-8} \cmidrule(l){9-14}
\textbf{VLMs} & & \textbf{MMB} & \textbf{OOD} & \textbf{SQA} & \textbf{SB} & \textbf{AI2D} & \textbf{Avg.} & \textbf{MMB} & \textbf{OOD} & \textbf{SQA} & \textbf{SB} & \textbf{AI2D} & \textbf{Avg.} \\
\midrule
LLaVA-v1.6-34B
& APS & 87.73 & 87.42 & 84.38 & \textbf{81.72} & 83.22 & 84.89 & 2.6501 & 1.6744 & 2.7269 & 2.6556 & 2.6386 & 2.4691 \\
& LAC & 86.75 & 86.47 & 83.53 & 81.39 & 82.55 & 84.14 & 1.2499 & 1.3101 & 1.2883 & 1.5854 & 1.4683 & 1.3804 \\
& Ours & \textbf{88.57} & \textbf{88.19} & \textbf{86.46} & 81.64 & \textbf{88.36} & \textbf{86.64} & \underline{1.6519} & \underline{1.6210} & \underline{1.8447} & \underline{1.9937} & \underline{2.1755} & \underline{1.8574} \\
\cmidrule(lr){1-14}
LLaVA-v1.6-13B
& APS & 82.29 & 80.02 & 78.08 & \textbf{77.83} & 80.39 & 79.72 & 3.1275 & 2.6857 & 3.2180 & 3.1280 & 3.0165 & 3.0351 \\
& LAC & 81.75 & 80.47 & 77.91 & 77.37 & 79.95 & 79.49 & 1.5573 & 1.6842 & 1.6884 & 1.8606 & 1.6505 & 1.6882 \\
& Ours & \textbf{82.66} & \textbf{80.79} & \textbf{79.08} & 77.50 & \textbf{84.87} & \textbf{81.38} & \underline{2.6249} & \underline{2.2271} & \underline{2.1796} & \underline{2.2776} & \underline{2.3135} & \underline{2.3245} \\
\cmidrule(lr){1-14}
LLaVA-v1.6-7B
& APS & 81.36 & 81.14 & 74.98 & \textbf{76.96} & 77.44 & 78.38 & 3.1540 & 2.9613 & 3.0303 & 3.1102 & 2.9752 & 3.0462 \\
& LAC & 80.60 & 79.89 & 75.03 & 76.65 & 77.03 & 77.84 & 1.5811 & 1.7250 & 1.8690 & 1.9445 & 1.7617 & 1.7763 \\
& Ours & \textbf{82.19} & \textbf{81.20} & \textbf{75.34} & 76.34 & \textbf{81.83} & \textbf{79.38} & \underline{1.9890} & \underline{2.2982} & \underline{2.1912} & \underline{2.3464} & \underline{2.3663} & \underline{2.2382} \\
\midrule
\midrule
\textbf{LLMs} & & \textbf{HSwg} & \textbf{HDial} & \textbf{CQA} & \textbf{HSum} & \textbf{MMLU} & & \textbf{HSwg} & \textbf{HDial} & \textbf{CQA} & \textbf{HSum} & \textbf{MMLU} \\
\midrule
Yi-34B
& APS & 95.21 & 83.99 & 95.74 & 81.20 & 80.64 & 87.76 & 3.0254 & 2.0548 & 2.5868 & 1.8630 & 2.8206 & 2.4701 \\
& LAC & 93.90 & 83.17 & 94.40 & 80.98 & 80.44 & 86.98 & 1.0000 & 1.3992 & 1.0000 & 1.3934 & 1.5886 & 1.2762 \\
& Ours & \textbf{96.17} & \textbf{85.56} & \textbf{96.12} & \textbf{83.09} & \textbf{82.90} & \textbf{88.77} & \underline{1.4790} & \underline{2.0714} & \underline{1.5664} & \underline{1.8540} & \underline{2.1220} & \underline{1.8186} \\
\cmidrule(lr){1-14}
Qwen-14B
& APS & 93.75 & 81.91 & 93.95 & 62.86 & 74.43 & \textbf{81.38} & 3.0120 & 2.4050 & 2.7242 & 2.6036 & 2.9640 & 2.7418 \\
& LAC & 91.98 & 82.42 & 92.06 & \textbf{64.22} & 74.26 & 80.59 & 1.0000 & 1.4634 & 1.0008 & 2.3154 & 2.1026 & 1.5764 \\
& Ours & \textbf{94.02} & \textbf{83.09} & \textbf{94.32} & 57.59 & \textbf{76.13} & 81.03 & \underline{1.3774} & \underline{1.8742} & \underline{1.3270} & \underline{2.3764} & \underline{2.5508} & \underline{1.9012} \\
\cmidrule(lr){1-14}
Qwen-7B
& APS & 72.46 & 74.47 & 88.38 & 52.48 & 67.47 & 71.85 & 2.3844 & 2.9366 & 3.1336 & 3.0076 & 3.5344 & 3.1993 \\
& LAC & 72.12 & 75.63 & 87.65 & \textbf{52.91} & 68.07 & 71.68 & 2.0564 & 2.0014 & 1.1790 & 2.9220 & 2.4890 & 2.1296 \\
& Ours & \textbf{73.79} & \textbf{75.81} & \textbf{90.06} & 47.75 & \textbf{72.25} & \textbf{71.93} & \underline{2.6116} & \underline{2.8832} & \underline{1.9172} & \underline{2.5734} & \underline{3.1820} & \underline{2.6335} \\
\bottomrule
\end{tabular}
\end{table*}

\textit{Area Under the Receiver Operating Characteristic (AUROC):} AUROC \cite{davis2006relationship} curve evaluates the model’s ability to rank predictions by confidence. It measures how effectively the model distinguishes correct from incorrect predictions, with a higher AUROC indicating more reliable confidence-based ranking.  

\textit{Area Under the Accuracy-Rejection Curve (AUARC):} AUARC illustrates the trade-off between accuracy and the retained fraction of predictions after abstaining from uncertain ones. For the proposed framework, AUARC quantifies how well the model's uncertainty estimates align with true prediction difficulty. A higher AUARC indicates better identification and abstention from difficult cases while maintaining high accuracy on confident predictions.  

\textit{Expected Calibration Error (ECE)} \cite{naeini2015obtaining}: ECE quantifies how well the model’s confidence estimates align with empirical correctness rates. Given \(n_{\text{bins}}\) confidence bins, it is defined as:  
\[
ECE = \sum_{b=1}^{n_{\text{bins}}} \frac{\lvert B_b \rvert}{N}
     \,\bigl|\text{acc}(B_b) - \text{conf}(B_b)\bigr|,
\]
where \( B_b \) is the set of samples whose confidence scores fall into bin \( b \), \( \text{acc}(B_b) \) is the mean accuracy, and \( \text{conf}(B_b) \) is the mean predicted confidence in that bin. Lower ECE indicates better calibration, meaning the model’s confidence estimates closely match empirical correctness rates.

\textbf{Prompting Strategies:}~~For VLMs, we adapt the multiple-choice Question Answering (VQA) template from LLaVA \cite{liu2024visual}. Each prompt starts with the attached image, followed by the question text and any relevant hints. Six answer options are then listed line by line, each prefixed with a letter (A-F). The prompt ends with the explicit instruction: \textit{``Answer with the option's letter from the given choices directly."} To ensure compatibility, we use model-specific templates sourced from their official GitHub repositories. Templates for CogAgent and InternLM-XComposer2 are obtained from Hugging Face. Due to the common constraint of single-image input in many VLMs, we exclude few-shot demonstrations. For more details on prompt template refer to Appendix \ref{prompt_template}.

For LLMs, we use the \textit{Base Prompt} strategy, following \cite{zheng2023judging}. This method concatenates the question with all answer options as the input prompt. The LLM is instructed to output the correct option using the prefix \textit{``Answer:"}, ensuring a standardized and straightforward input format for evaluation. For more details on prompt template please refer to Appendix \ref{prompt_template}.

\begin{table}[h!]
\centering
\scriptsize
\setlength{\tabcolsep}{2pt}
\caption{Evaluation of Expected Calibration Error (ECE):  Comparative analysis of the proposed CAP framework (Ours) with standard LAC \cite{sadinle2019least}, and APS \cite{romano2020classification} methods. The results show that the CAP method achieves significantly lower ECE values, in \textbf{bold}, compared to baseline.}
\label{tab:ece}
\begin{tabular}{lccccccc}
\toprule
\multirow{2}{*}{\textbf{Model}} & \multirow{2}{*}{\textbf{Method}} & \multicolumn{5}{c}{\textbf{ECE} ↓} \\
\cmidrule(lr){3-8}
\textbf{VLMs} & & \textbf{MMB} & \textbf{OOD} & \textbf{SQA} & \textbf{SB} & \textbf{AI2D} & \textbf{Avg.} \\
\midrule
\multirow{3}{*}{LLaVA-v1.6-34B}
& APS & 0.1277 & 0.1261 & 0.2082 & 0.1356 & 0.2353 & 0.1666 \\
& LAC & 0.0738 & 0.1124 & 0.1143 & 0.1312 & 0.1626 & 0.1109 \\
& Ours & \textbf{0.0085} & \textbf{0.0302} & \textbf{0.0309} & \textbf{0.0342} & \textbf{0.0385} & \textbf{0.0285} \\
\cmidrule(lr){1-8}
\multirow{3}{*}{LLaVA-v1.6-13B}
& APS & 0.1593 & 0.2218 & 0.1902 & 0.1607 & 0.2747 & 0.2013 \\
& LAC & 0.1300 & 0.1698 & 0.1618 & 0.1759 & 0.1908 & 0.1657 \\
& Ours & \textbf{0.0218} & \textbf{0.0159} & \textbf{0.0445} & \textbf{0.0601} & \textbf{0.0252} & \textbf{0.0335} \\
\cmidrule(lr){1-8}
\multirow{3}{*}{LLaVA-v1.6-7B}
& APS & 0.1576 & 0.2439 & 0.2128 & 0.1704 & 0.2641 & 0.2098 \\
& LAC & 0.1314 & 0.1974 & 0.1865 & 0.1797 & 0.1987 & 0.1787 \\
& Ours & \textbf{0.0419} & \textbf{0.0252} & \textbf{0.0498} & \textbf{0.0581} & \textbf{0.0148} & \textbf{0.0380} \\
\midrule
\midrule
\textbf{LLMs} & & \textbf{HSwg} & \textbf{HDial} & \textbf{CQA} & \textbf{HSum} & \textbf{MMLU} & \textbf{Avg.} \\
\midrule
\multirow{3}{*}{Qwen-7B}
& APS & 0.3470 & 0.2978 & 0.2327 & 0.4099 & 0.4032 & 0.3381 \\
& LAC & 0.3222 & 0.2680 & 0.1479 & 0.4381 & 0.3474 & 0.3047 \\
& Ours & \textbf{0.0807} & \textbf{0.0265} & \textbf{0.0772} & \textbf{0.1409} & \textbf{0.0485} & \textbf{0.0748} \\
\cmidrule(lr){1-8}
\multirow{3}{*}{Qwen-14B}
& APS & 0.0901 & 0.1972 & 0.0996 & 0.2949 & 0.2729 & 0.1909 \\
& LAC & 0.0156 & 0.1644 & 0.0278 & 0.3360 & 0.2273 & 0.1542 \\
& Ours & \textbf{0.0134} & \textbf{0.0307} & \textbf{0.0266} & \textbf{0.1271} & \textbf{0.0170} & \textbf{0.0429} \\
\cmidrule(lr){1-8}
\multirow{3}{*}{Yi-34B}
& APS & 0.1111 & 0.3240 & 0.1554 & 0.2163 & 0.2479 & 0.2109 \\
& LAC & 0.0514 & 0.2718 & 0.1030 & 0.1887 & 0.1727 & 0.1575 \\
& Ours & \textbf{0.0522} & \textbf{0.1528} & \textbf{0.0990} & \textbf{0.0542} & \textbf{0.0337} & \textbf{0.0784} \\
\bottomrule
\end{tabular}
\end{table}

\subsection{Evaluation and Results}  
We evaluate VLMs and LLMs with parameter sizes ranging from 7B to 34B, with detailed results presented in the following sections.

\textbf{Hallucination Detection:}~We evaluated the effectiveness of conformal methods in detecting hallucinations in VLM/LLM responses to multiple choice QA tasks. Here, hallucination refers to confidently asserted yet incorrect and arbitrary claims made by the model. Detecting them is framed as a binary classification task: distinguishing correct from hallucinated (incorrect) responses using uncertainty estimates. AUROC serves as the primary evaluation metric \cite{farquhar2024detecting}, with higher scores indicating better separation of correct and incorrect predictions based on model confidence. As shown in \autoref{tab:auroc_auarc}, our proposed CAP method consistently achieves higher AUROC scores than APS and LAC across various models and datasets. Notably, this improvement reaches up to 10.17\% on average and 22.19\% in specific instances, demonstrating CAP’s effectiveness in reliably detecting hallucinations.

\begin{figure*}[h!]
    \centering
    \makebox[\textwidth]{\includegraphics[width=0.9\textwidth]{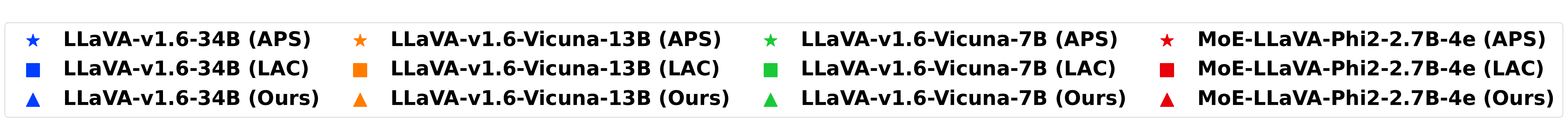}}
    \makebox[\textwidth]{
        \begin{tabular}{ccc}
            \includegraphics[width=0.30\textwidth]{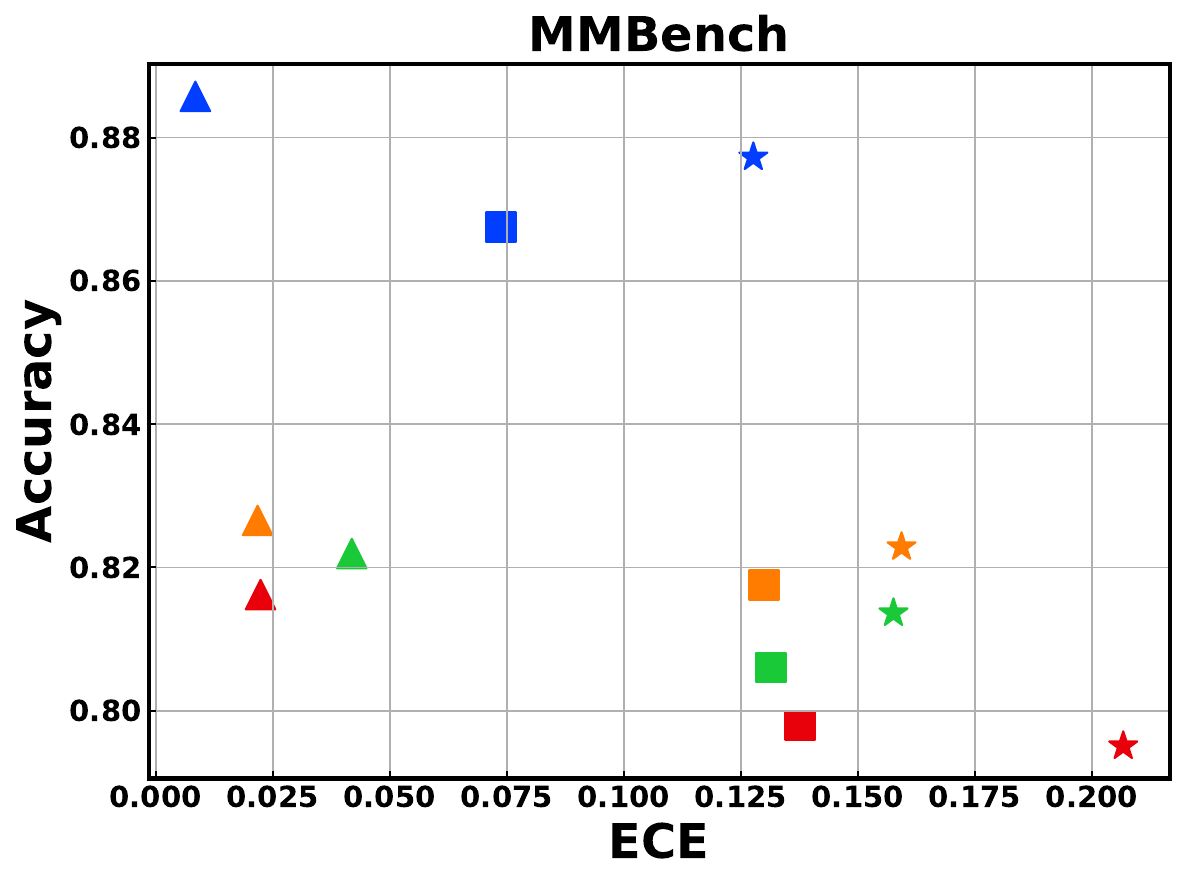} &
            \includegraphics[width=0.30\textwidth]{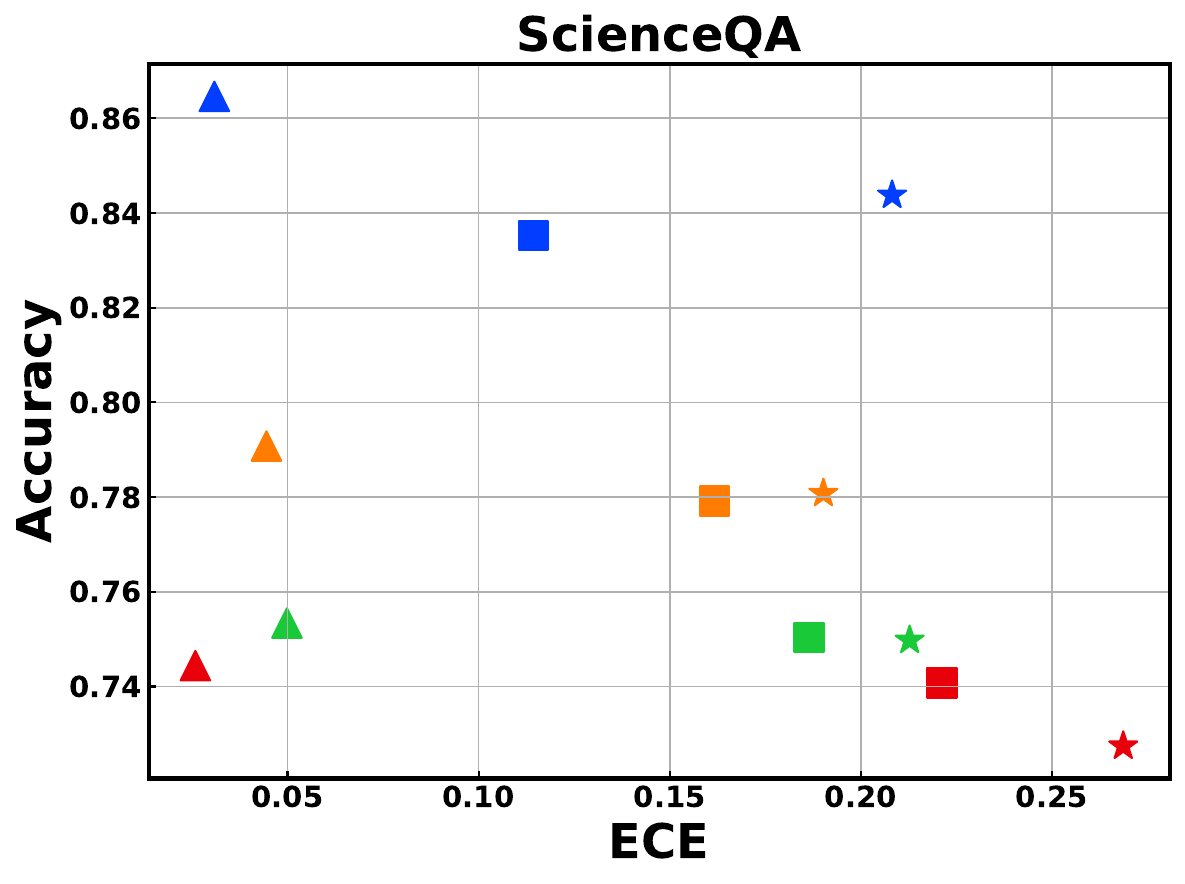} &
            \includegraphics[width=0.30\textwidth]{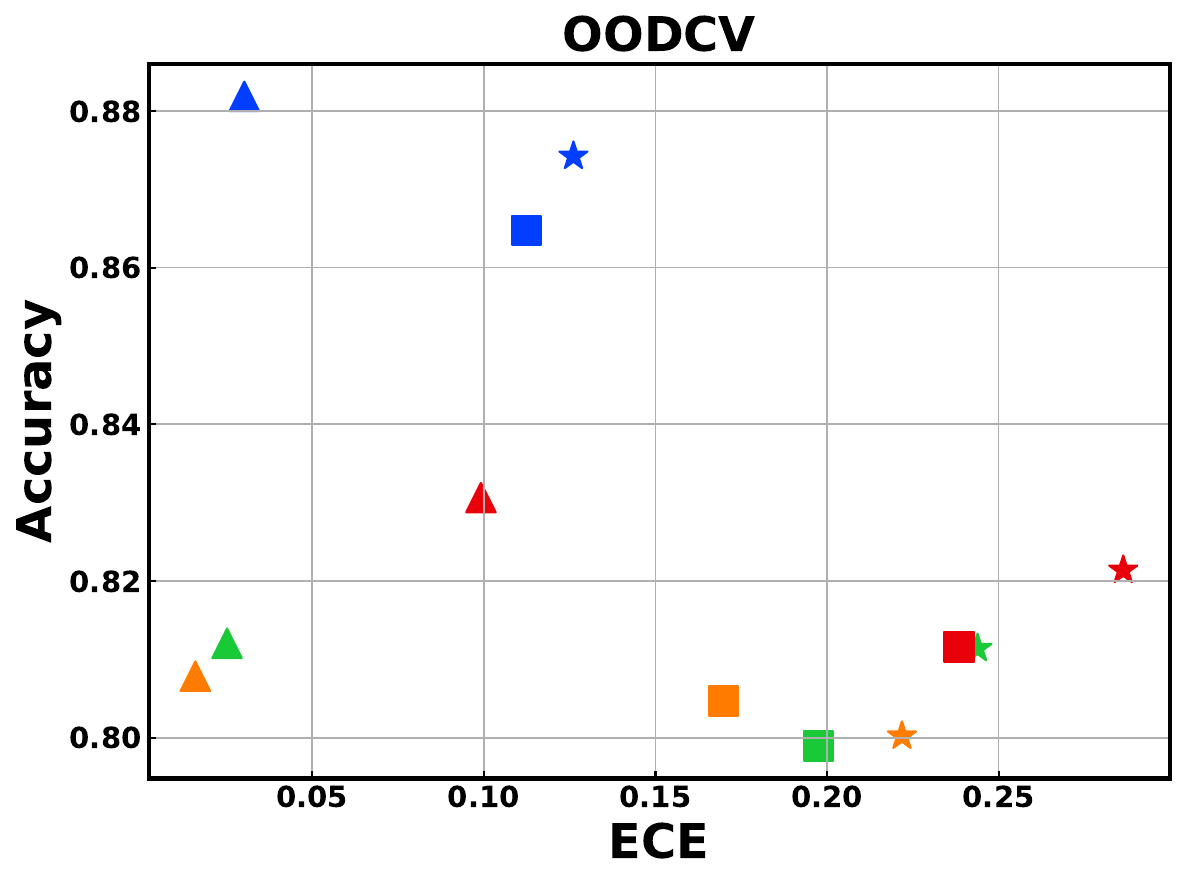} \\
        \end{tabular}
    }
\caption{Accuracy vs. Expected Calibration Error (ECE) comparison of CAP, APS, and LAC across various VLMs and five datasets: MMBench, ScienceQA, OODCV, SEEDBench, and AI2D. An ideal model has high accuracy and low ECE (upper-left). ATCP shows significant ECE improvement over baselines. Please refer to \autoref{fig:annotated_grid_ap} in Appendix \ref{accuracy_ece_appendix} for complete list of figures.}
    \label{fig:annotated_grid}
\end{figure*}

\begin{figure*}[h!]
    \centering
    \makebox[\textwidth]{\includegraphics[width=0.5\textwidth]{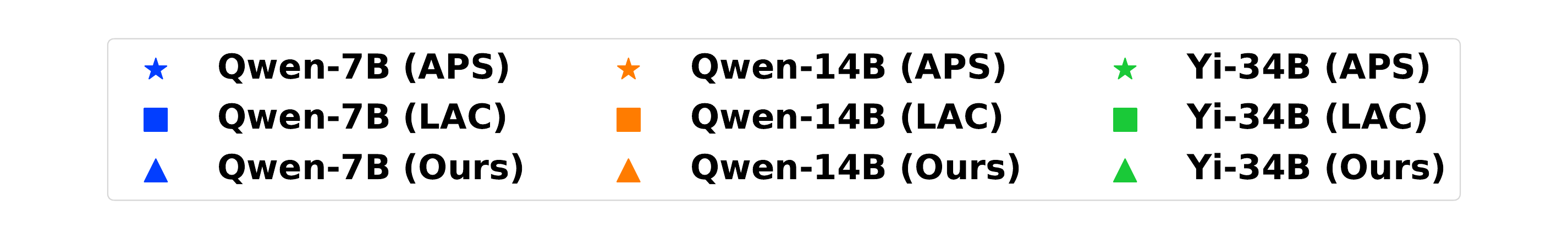}}
    \makebox[\textwidth]{
        \begin{tabular}{ccc}
            \includegraphics[width=0.30\textwidth]{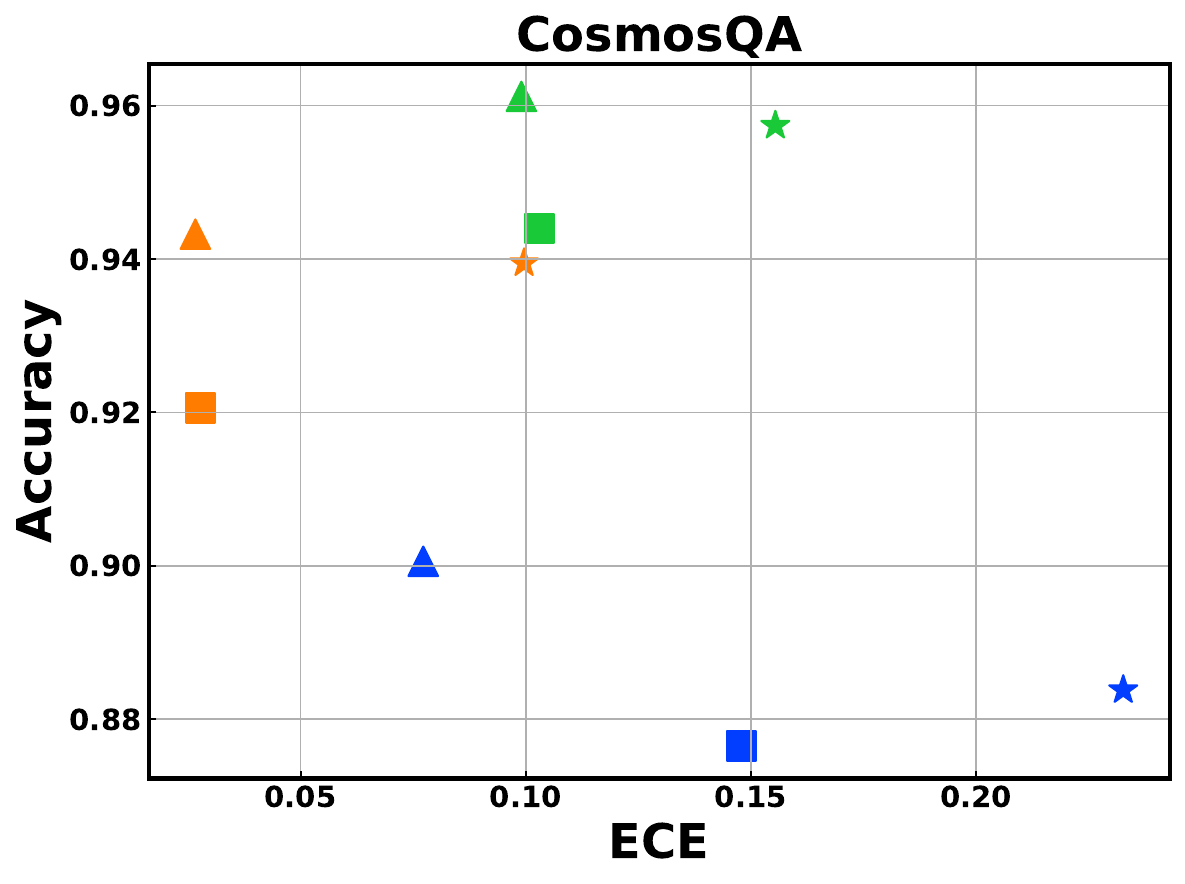} &
            \includegraphics[width=0.30\textwidth]{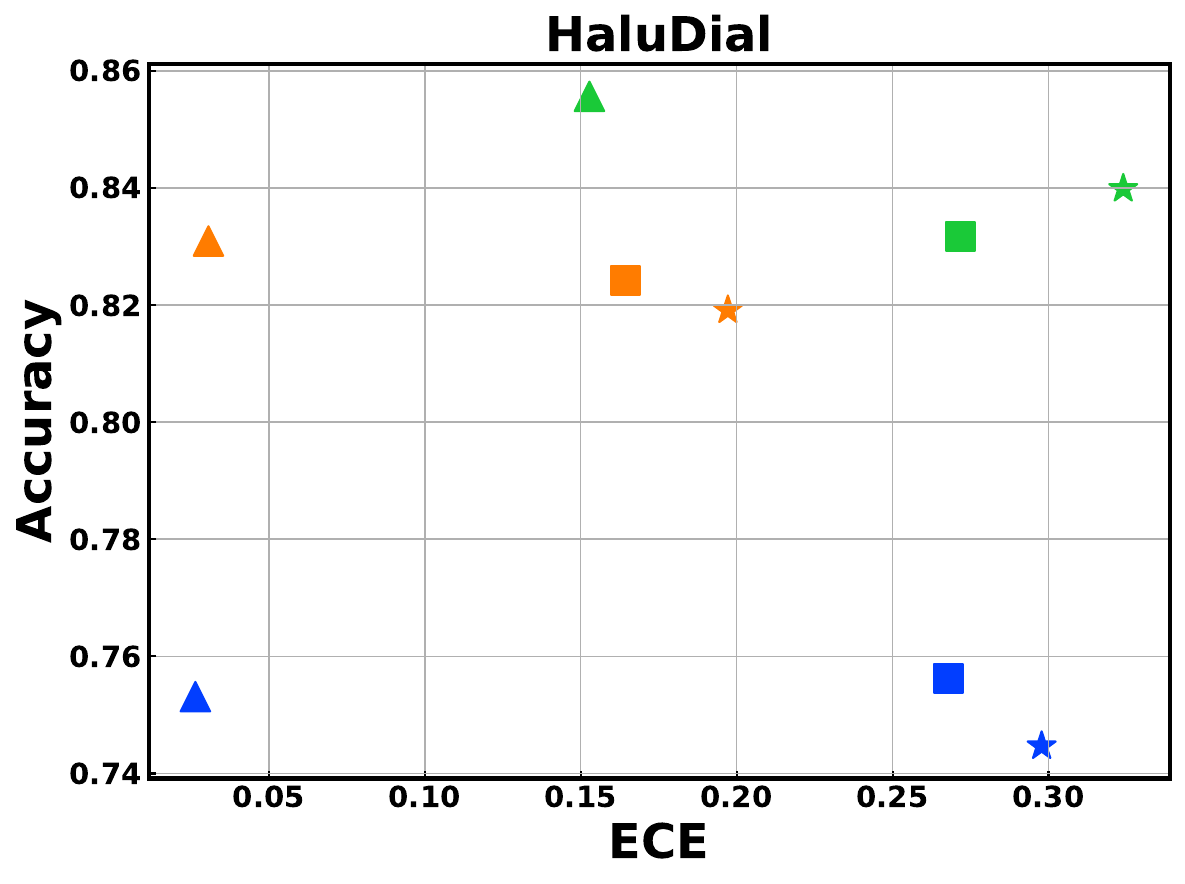} &
            \includegraphics[width=0.30\textwidth]{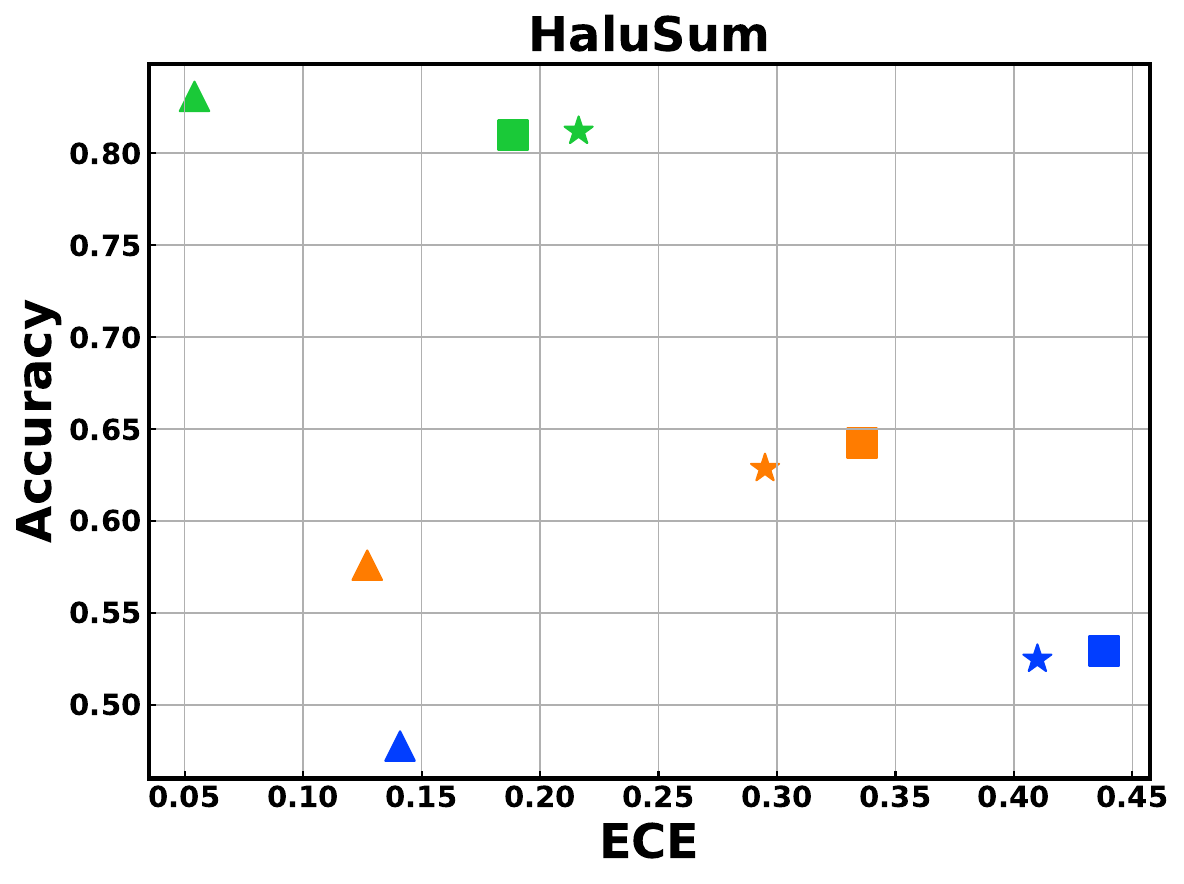} \\
        \end{tabular}
    }
    \caption{Accuracy versus Expected Calibration Error (ECE) comparison between ATCP, APS and LAC methods across different LLMs and five datasets i.e. CosmosQA, HaluDial, HaluSum, HellaSwag, MMLU. The ideal model should have high accuracy and low ECE, indicating accurate predictions with well calibrated uncertainty quantification (upper-left of the plot). The ECE of ATCP shows significant improvement compared to baseline methods. Please refer to \autoref{fig:annotated_grid_llm_ap} in Appendix \ref{accuracy_ece_appendix} for the complete list of figures.}
    \label{fig:annotated_grid_llm}
\end{figure*}

\textbf{Uncertainty-Guided Selective Generation:}~We assess the conformal model's ability to abstain from responses when uncertainty is high. Following \cite{farquhar2024detecting} and \cite{krishnan2024enhancing}, we evaluate this using AUARC, which quantifies the trade-off between accuracy and abstention. Higher AUARC indicates better alignment of uncertainty estimates with prediction difficulty, enabling effective abstention in challenging examples while maintaining accuracy on confident predictions. As shown in \autoref{tab:auroc_auarc}, our CAP method consistently outperforms APS and LAC in AUARC, with average improvements of up to 9.43\% and peak gains of 21.17\%. This demonstrates CAP’s effectiveness in allowing better abstention policies and selective generation.

\textbf{Coverage Guarantee:}~CAP ensures that the true label is included in the prediction set (or as a single prediction when confidence is high) with a predefined probability ($1-\alpha$, set to 90\%). This target represents the minimum probability of capturing the true label, corresponding to a maximum error rate of 10\%. As shown in \autoref{tab:coverage}, CAP consistently achieves at least 90\% coverage across all datasets and VLMs. Larger prediction sets indicate greater uncertainty while preserving statistical guarantees. Compared to baselines, CAP effectively balances \textit{coverage and prediction set size}. APS attains a higher coverage rate but at the cost of excessively large sets, introducing greater uncertainty and reducing practical utility (\autoref{tab:accuracy_vs_ss}). Conversely, LAC produces smaller sets but frequently falls short of the 90\% coverage target, compromising reliability (\autoref{tab:coverage}). CAP optimally bridges these extremes, reliably maintaining the 90\% guarantee while keeping prediction sets well-controlled. This balance ensures statistically valid and practically useful estimates.

\textbf{Accuracy and Set Size:}~Prediction set size is a key indicator of model uncertainty, with smaller sets generally reflecting lower uncertainty. However, as shown in \autoref{tab:accuracy_vs_ss}, the relationship between accuracy and set size is not always straightforward. While accuracy remains relatively consistent across APS, LAC, and our method, set sizes vary significantly. APS achieves competitive accuracy but often produces overly large prediction sets, indicating higher uncertainty and less informative outputs. In contrast, LAC generates the smallest sets but at the cost of compromised coverage rates and slightly lower accuracy in some cases (\autoref{tab:accuracy_vs_ss}).

\begin{figure*}[t!]
    \centering
    \makebox[\textwidth]{\includegraphics[width=0.4\textwidth]{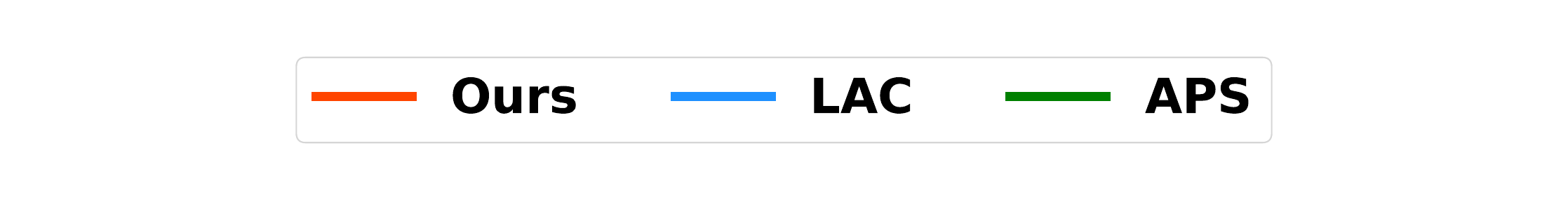}}
    \makebox[\textwidth]{
        \begin{tabular}{cccc}
            \includegraphics[width=0.24\textwidth]{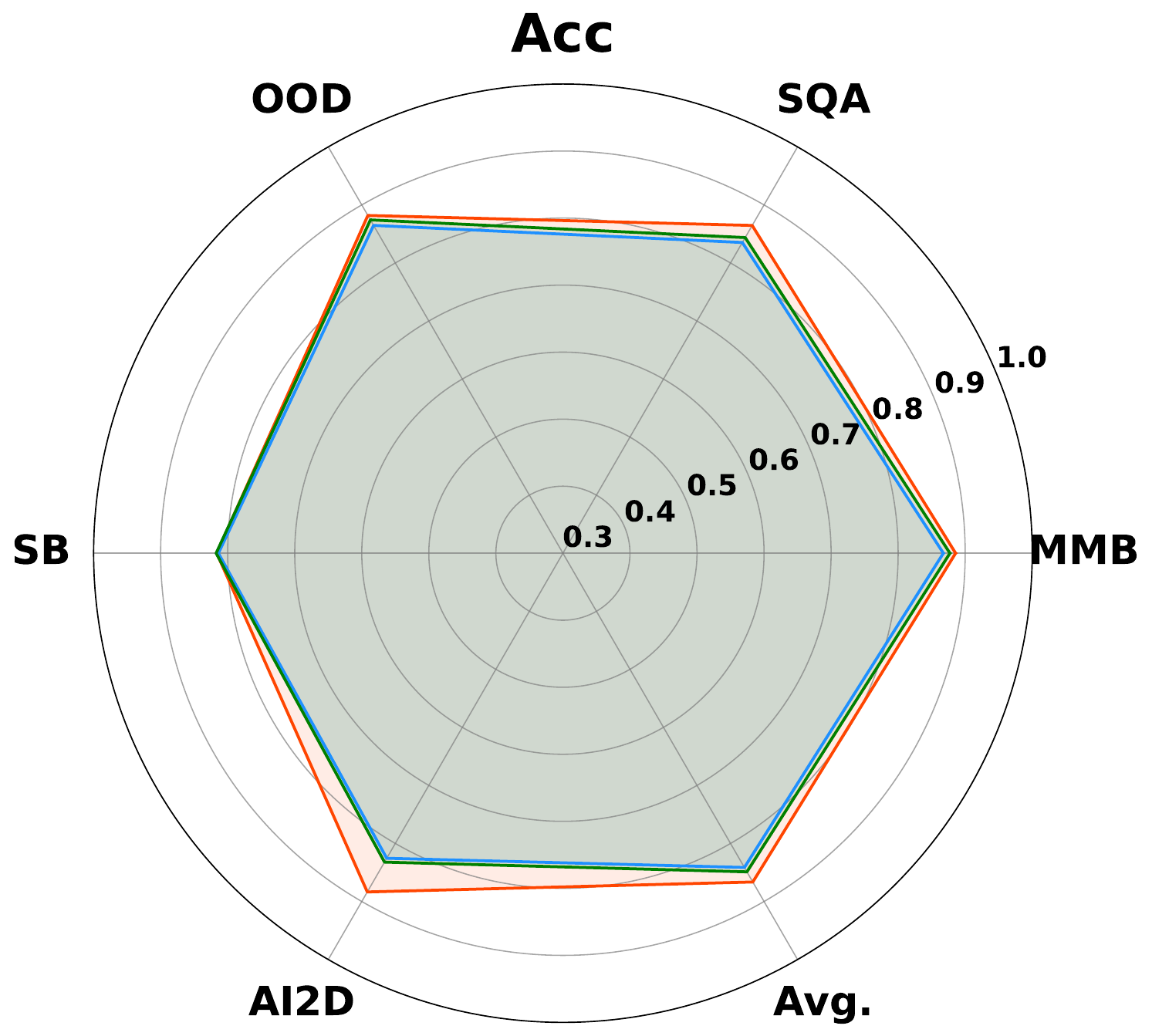} &
            \includegraphics[width=0.24\textwidth]{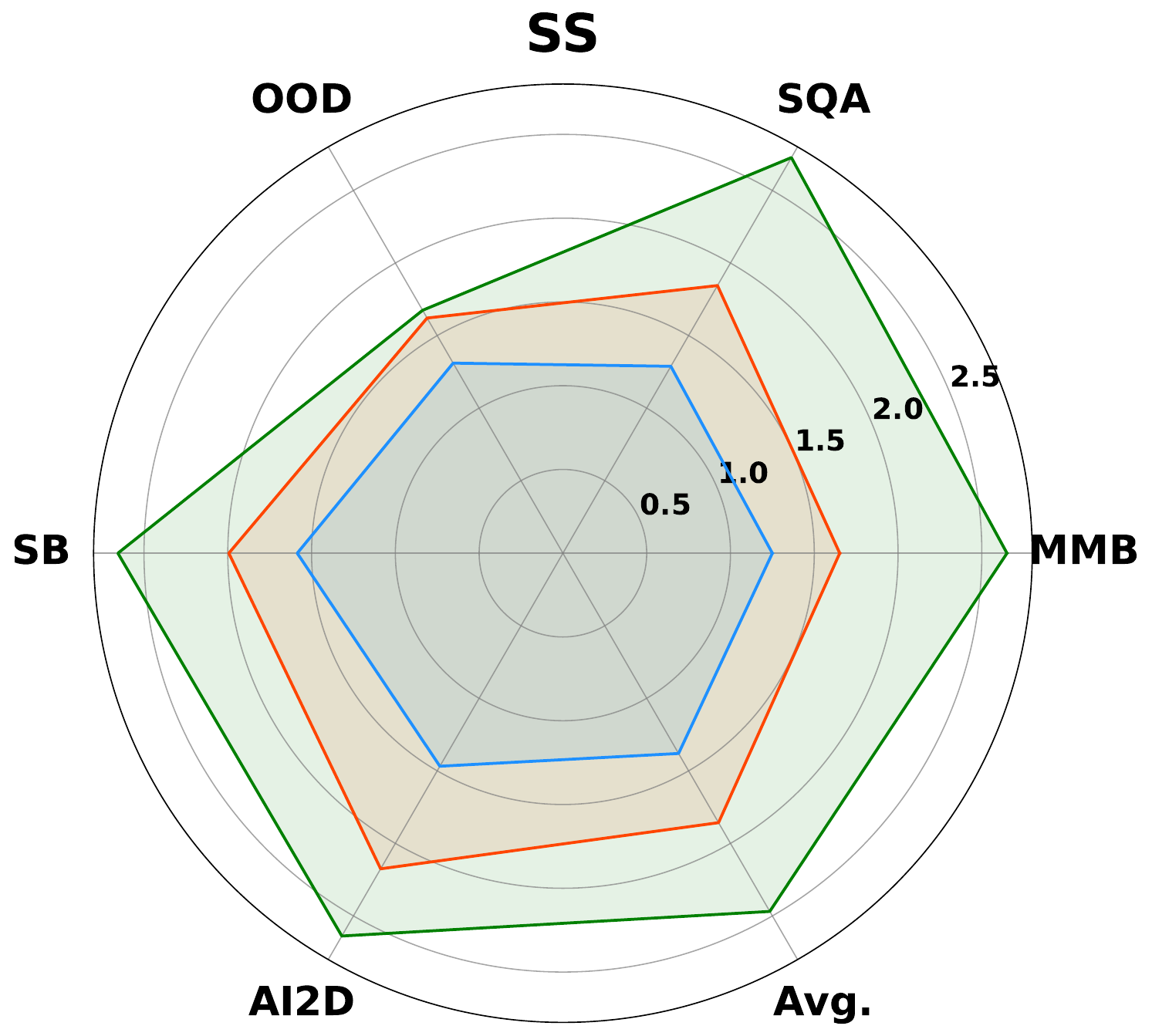} &
            \includegraphics[width=0.24\textwidth]{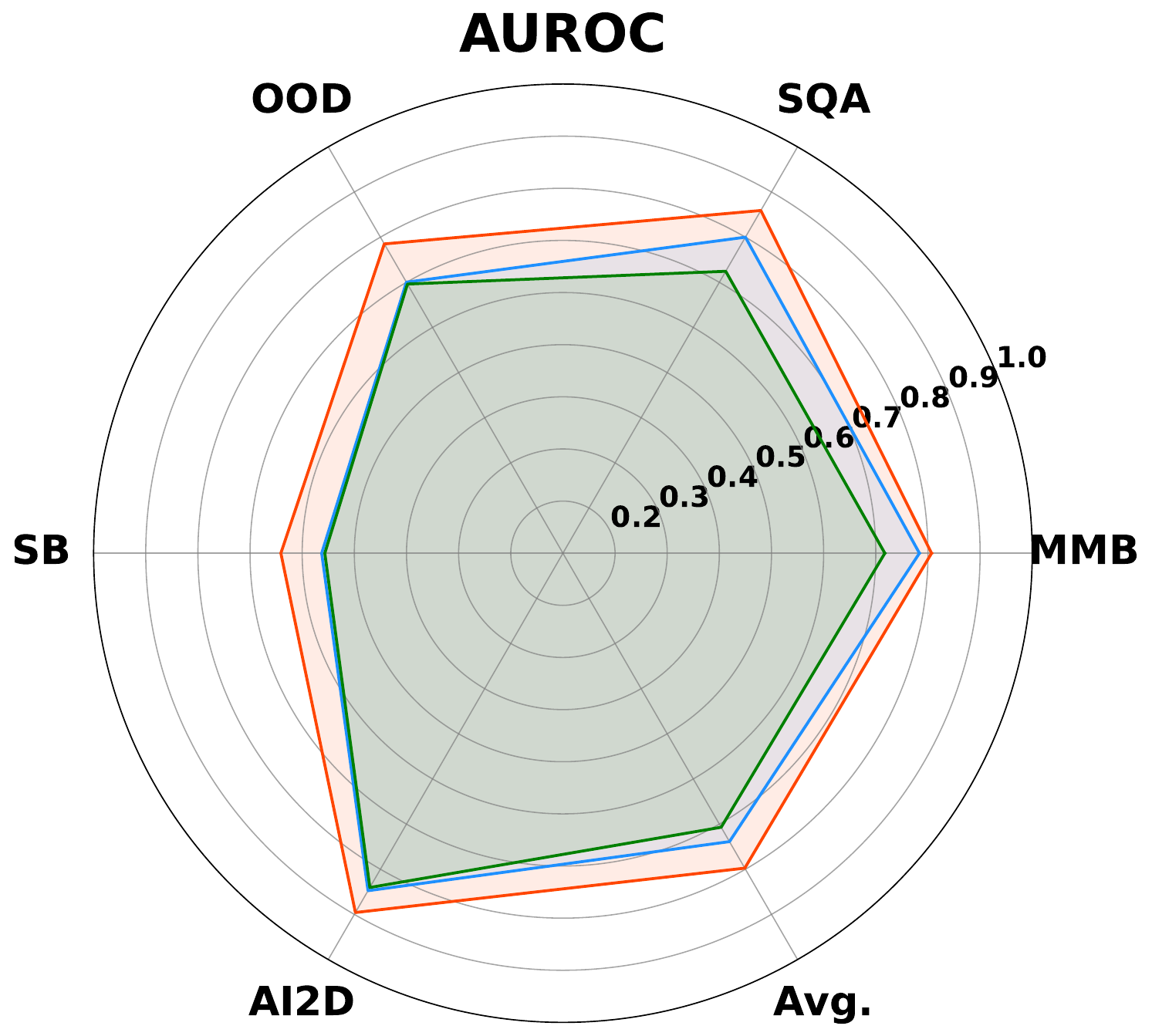} &
            \includegraphics[width=0.24\textwidth]{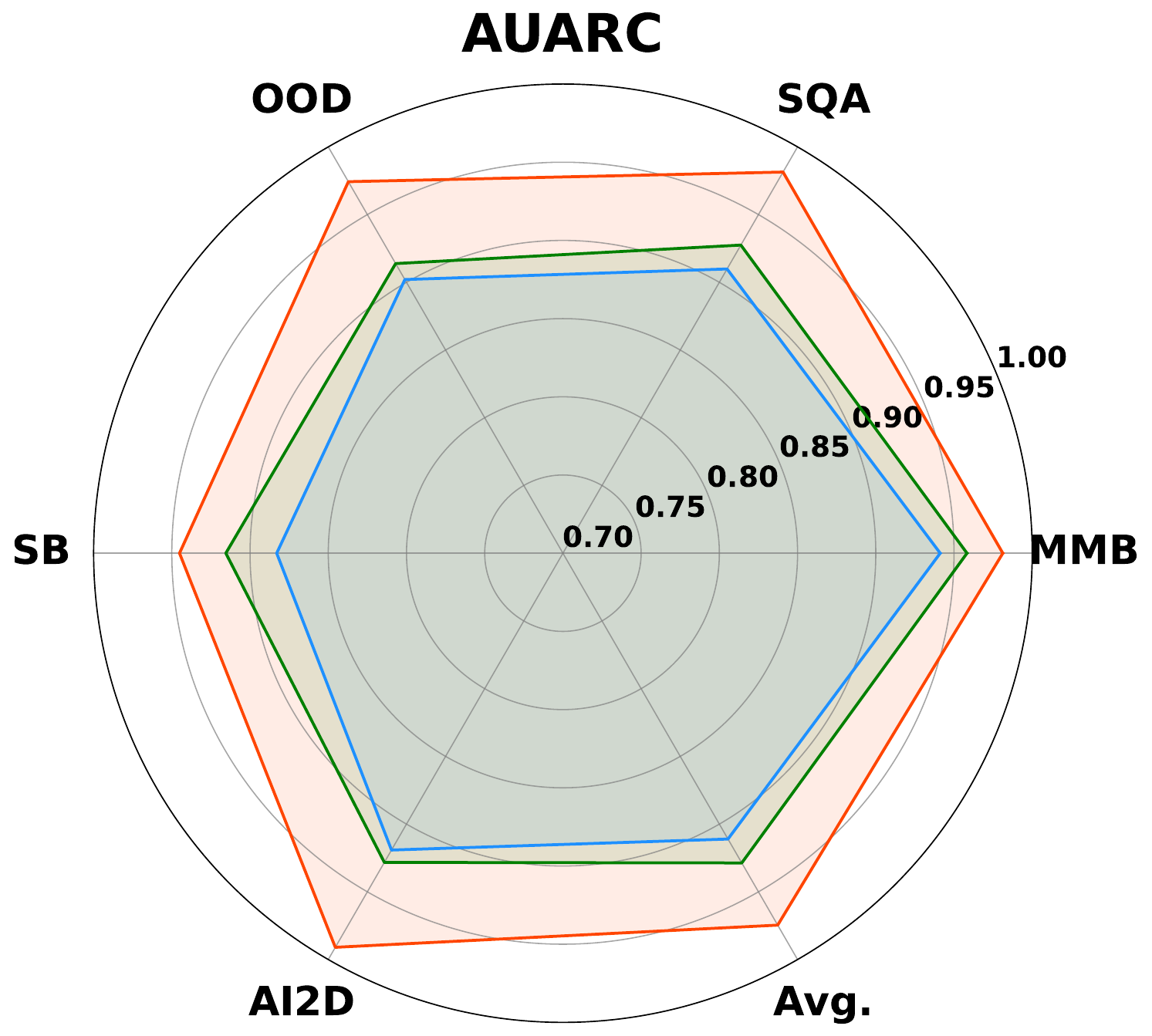} \\
        \end{tabular}
    }
    \makebox[\textwidth]{
        \begin{tabular}{cccc}
            \includegraphics[width=0.24\textwidth]{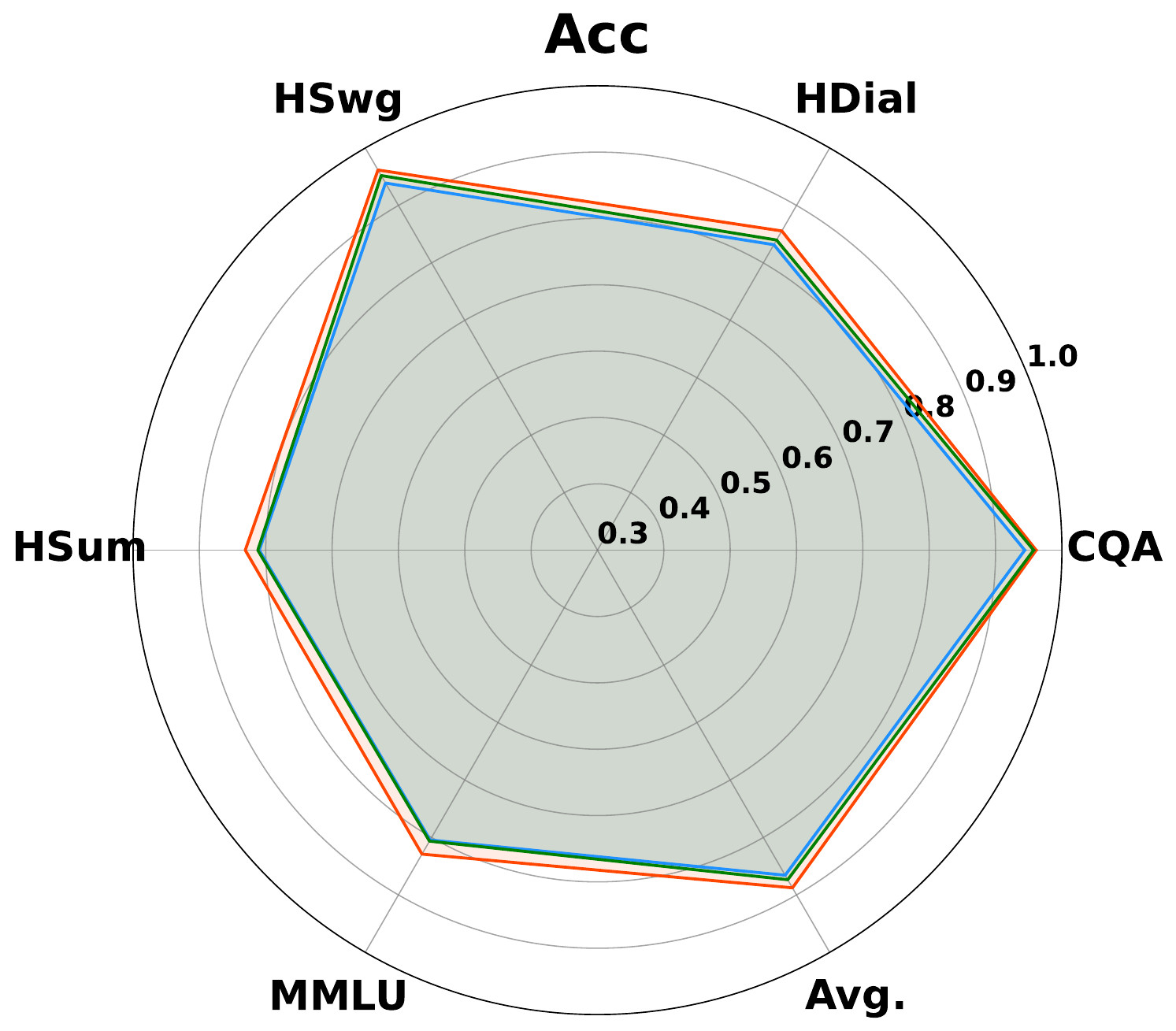} &
            \includegraphics[width=0.24\textwidth]{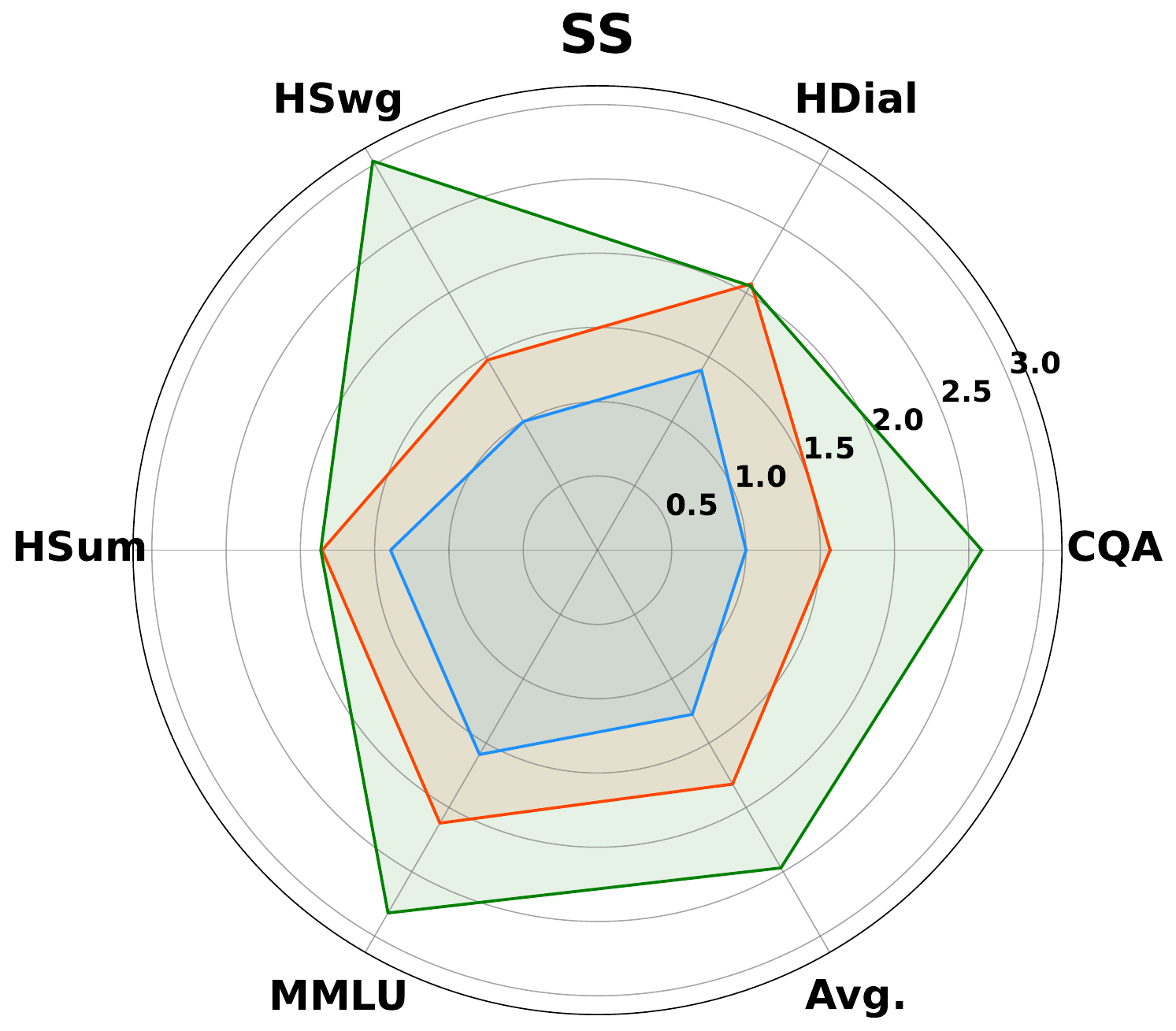} &
            \includegraphics[width=0.24\textwidth]{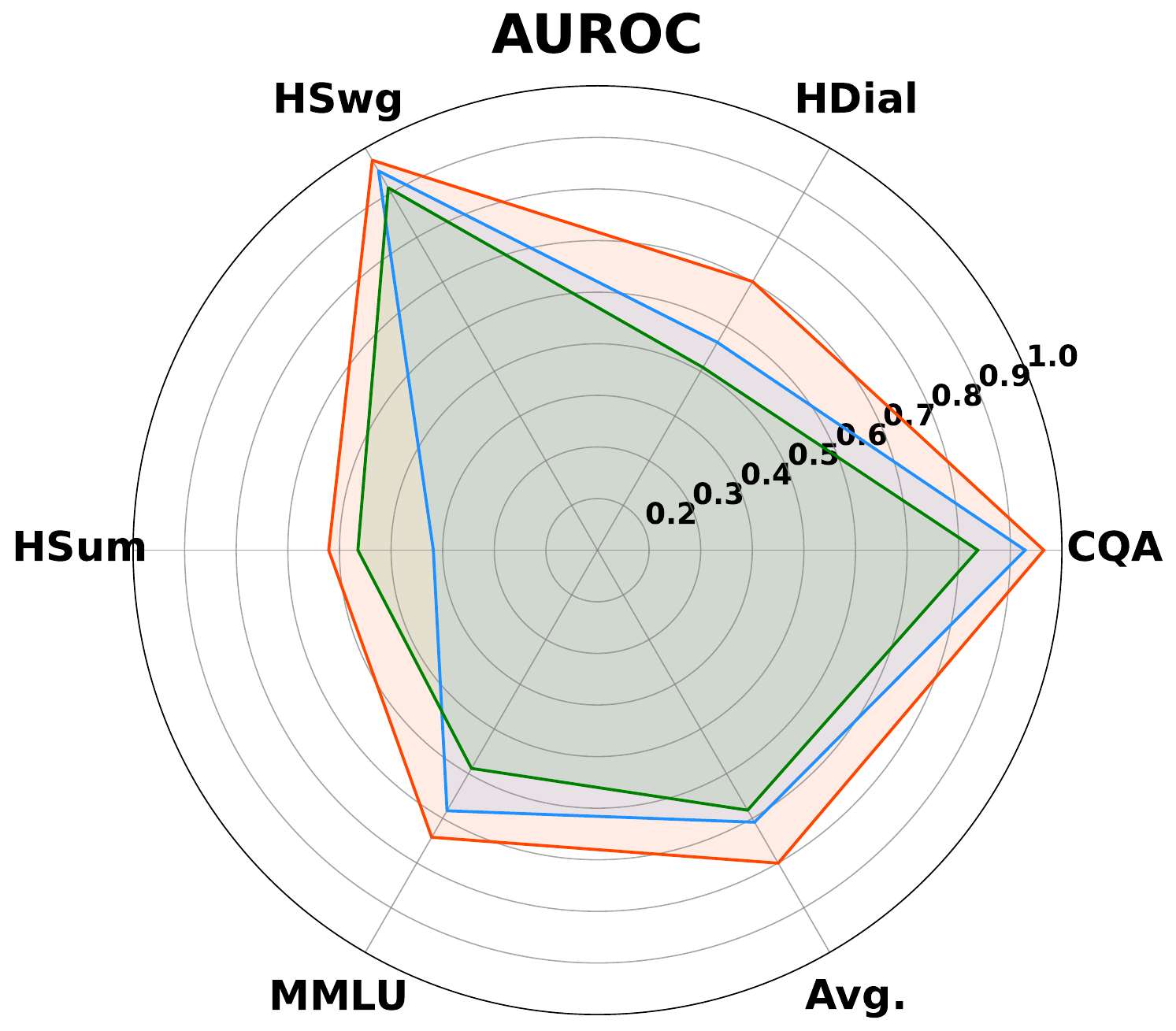} &
            \includegraphics[width=0.24\textwidth]{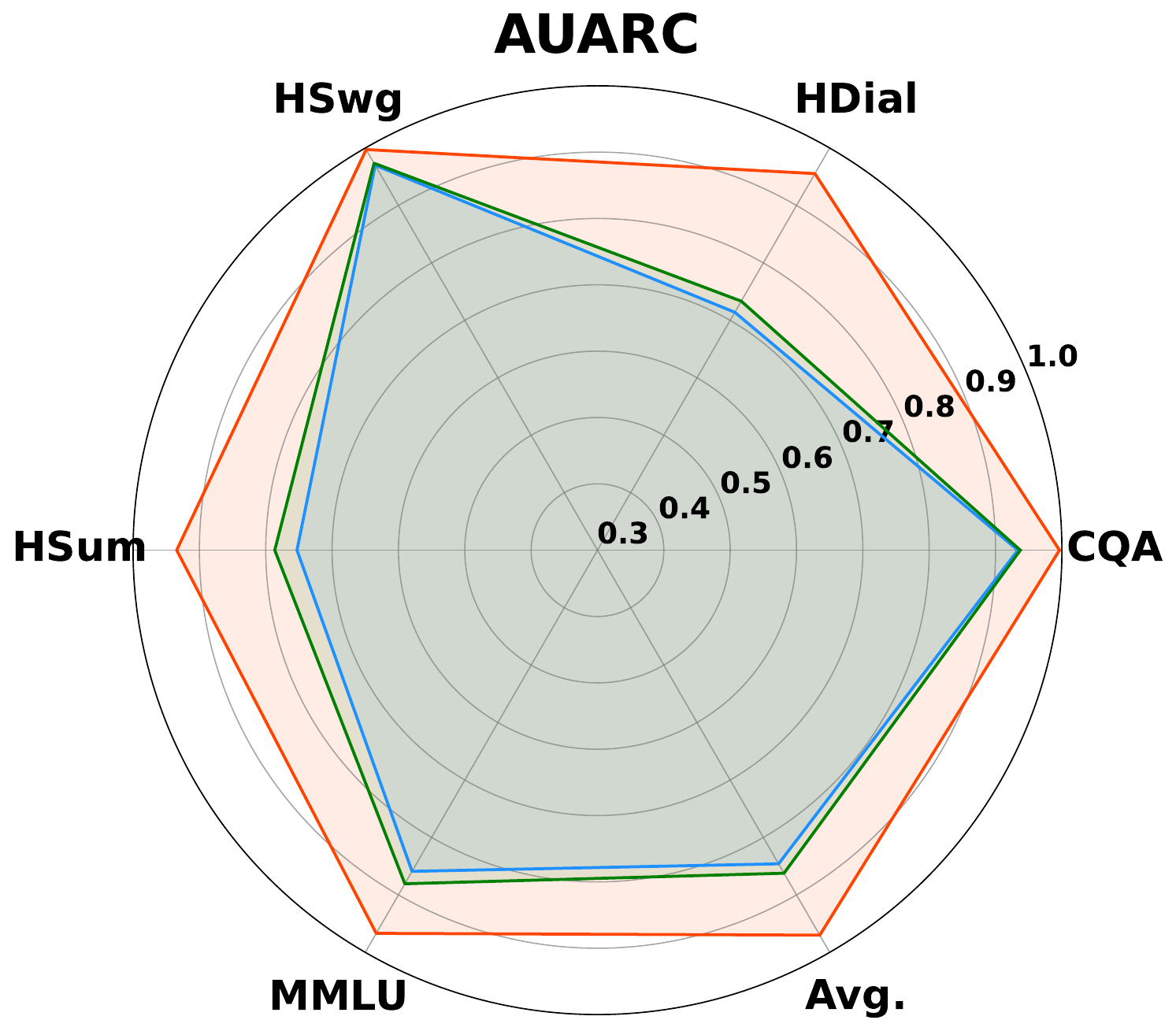} \\
        \end{tabular}
    }
    \caption{Performance comparison of CAP (Ours), APS, and LAC on Llava-v1.6-34B (VLM) and Yi-34B (LLM) across four metrics: (i) accuracy, (ii) set size, (iii) AUROC, and (iv) AUARC. Each figure shows model performance across ten benchmark datasets, illustrating the impact of conformal method on uncertainty metrics.}
    \label{fig:spider_all}
\end{figure*}

Our CAP method achieves the highest average accuracy across datasets by leveraging the trainable adaptive threshold mechanism. It generally outperforms APS in accuracy, with only minor exceptions. In terms of set sizes, CAP consistently strikes a balance--producing smaller, more controlled sets than APS while avoiding LAC’s overly narrow sets that lead to under-coverage. As shown in \autoref{tab:accuracy_vs_ss}, the \textit{balanced} set sizes of CAP are underlined. This ensures effective uncertainty management, which is a key strength of our approach. Importantly, these findings demonstrate that accuracy alone is insufficient for evaluating conformal methods. Significant variations in set sizes, despite similar accuracy, highlight the importance of set size as a distinct measure of performance and uncertainty.

\textbf{Accuracy and Expected Calibration Error:}~Calibration measures how well a model’s confidence estimates reflect actual correctness. We assess this using expected calibration error (ECE), where lower values indicate better alignment between confidence and accuracy. As shown in \autoref{tab:ece}, our CAP method consistently achieves significantly lower ECE than APS and LAC across all models and datasets. Notably, CAP improves calibration without compromising accuracy (\autoref{tab:accuracy_vs_ss}, \autoref{fig:annotated_grid}, \autoref{fig:annotated_grid_llm}). Compared to APS, CAP reduces ECE by an average of 82.9\% across all VLMs and LLMs (83.1\% for VLMs, 82.7\% for LLMs). Against LAC, CAP achieves an average ECE reduction of 74.8\% (74.4\% for VLMs, 75.2\% for LLMs), demonstrating superior calibration. The combination of lower ECE and improved accuracy underscores CAP's advantage: it delivers not only accurate predictions but also reliable uncertainty estimates. High-confidence predictions correspond to a higher likelihood of correctness, while uncertain predictions more accurately reflect the model’s limitations. \autoref{fig:annotated_grid} and \autoref{fig:annotated_grid_llm} illustrate CAP’s effectiveness, positioning it in the upper-left region of the accuracy-ECE space.

\textbf{Discussion and Limitations:}
Our extensive evaluations show that CAP can significantly improve over static uncertainty quantification methods by leveraging reinforcement learning to dynamically adjust thresholds, optimizing the trade-off between accuracy, coverage, and prediction set size. Empirically, CAP outperforms APS and LAC in hallucination detection (AUROC), selective generation (AUARC), and calibration error while maintaining valid coverage guarantees. As shown in \autoref{fig:spider_all}, CAP achieves higher accuracy than static baselines, balances prediction set sizes to prevent under- or over-coverage, and significantly improves uncertainty-aware metrics. Additional results across are provided in Appendix \ref{additional_results}.

However, integrating CP with RL introduces challenges. Learned policies may overfit, bias abstention strategies, or distort CP’s theoretical guarantees. CAP also introduces additional parameters and relies on well-tuned reward functions, which may require careful optimization for different data distributions. Extreme distribution shifts or limited calibration data can further impact performance if the calibration set fails to capture relevant uncertainty signals. These risks need a careful investigation and can be mitigated by enforcing distribution-aware regularization, calibrating policies through out-of-sample validation, and constraining reward functions to align with conformal principles. 

\section{Conclusion}
In this work, we propose a reinforcement learning-based approach to adaptively configure conformal prediction thresholds for selective abstention in large language and vision-language models. By dynamically adjusting the decision boundary between single-label, set-valued predictions, and abstentions, our method overcomes the limitations of static conformal approaches, such as rigid coverage–uncertainty trade-offs and suboptimal confidence calibration. Extensive evaluations across diverse tasks—from multiple-choice QA to image-based reasoning—demonstrate that our learned conformal abstention policy (CAP) outperforms APS and LAC, achieving higher accuracy, maintaining coverage guarantees, shrinking prediction sets, and reducing calibration error. Notably, CAP enhances hallucination detection and uncertainty-guided selective generation, highlighting the potential of coupling conformal prediction with adaptive policies for robust risk management in foundation models.


\bibliography{paper}
\bibliographystyle{IEEEtran}

\clearpage
\newpage
\appendices

\section{Formal Proof of Conformal Coverage Guarantee}
\label{sec:appendix_coverage}

We provide here a classic proof of the coverage property for standard (single-threshold) conformal prediction under i.i.d.\ assumptions. In the main text, this lays the foundation for our two-threshold extension (see Section~\ref{Conformal_Abstention}), where an additional threshold is introduced to distinguish between single-label predictions, set-valued predictions, and abstentions. Despite that extension, the core argument below underpins the claimed coverage guarantee at level \(1-\alpha\).

\newtheorem{theorem}{Theorem}
\begin{theorem}[Conformal Coverage Guarantee]
\label{thm:conformal_coverage}
Let \(\{(X_i, Y_i)\}_{i=1}^{n+1}\) be i.i.d.\ samples from an unknown distribution, partitioned into:
\begin{itemize}
    \item A \emph{calibration set} of size \(n\): \(\{(X_i, Y_i)\}_{i=1}^n\).
    \item A \emph{test point} \((X_{n+1}, Y_{n+1})\).
\end{itemize}
Suppose a \emph{nonconformity score function} \(s(\cdot,\cdot)\) assigns a real-valued score \(s(X_i, Y_i)\) to each calibration sample, capturing how ``atypical'' or ``nonconforming'' the pair \((X_i, Y_i)\) appears relative to a prediction model. Denoting
\[
s_i \;=\; s(X_i, Y_i), \quad i = 1,\ldots,n,
\]
let \(\hat{q}\) be the \((1-\alpha)\)-quantile of these calibration scores:
\[
\hat{q} \;=\; \mathrm{Quantile}\Bigl(\{s_1, \dots, s_n\},\; 1-\alpha\Bigr).
\]
Then we define the conformal prediction set for the test point \((X_{n+1}, \,\cdot\,)\) as
\[
C\bigl(X_{n+1}\bigr)
\;=\;
\Bigl\{\,y: s\bigl(X_{n+1}, y\bigr) \;\le\;\hat{q}\Bigr\}.
\]
Under the i.i.d.\ assumption, this set satisfies 
\[
\Pr\!\Bigl(Y_{n+1} \;\in\; C\bigl(X_{n+1}\bigr)\Bigr) \;\ge\; 1 - \alpha.
\]
\end{theorem}

\begin{proof}
Because the samples \(\{(X_i, Y_i)\}_{i=1}^{n+1}\) are assumed exchangeable (i.i.d.), any permutation of the \(n+1\) points is equally likely. Consider a random permutation \(\pi\) of the indices \(\{1,\dots,n+1\}\), and let 
\[
(\widetilde{X}_i, \widetilde{Y}_i) \;=\; \bigl(X_{\pi_i},\, Y_{\pi_i}\bigr)
\]
be the permuted data. We then treat the first \(n\) permuted samples as a calibration set, computing their nonconformity scores,
\[
\widetilde{s}_i \;=\; s\bigl(\widetilde{X}_i, \,\widetilde{Y}_i\bigr),
\quad
i=1,\dots,n,
\]
and defining
\[
\widetilde{q}
=
\mathrm{Quantile}\Bigl(\{\widetilde{s}_1,\dots,\widetilde{s}_n\},\,1-\alpha\Bigr).
\]
The point \((\widetilde{X}_{n+1}, \widetilde{Y}_{n+1})\) is then the ``test'' sample in this permuted view, for which the conformal set is
\[
C\bigl(\widetilde{X}_{n+1}\bigr)
=
\bigl\{\,
y
:
s\bigl(\widetilde{X}_{n+1}, y\bigr)
\le
\widetilde{q}
\bigr\}.
\]

We must show that 
\(
\Pr\bigl(\widetilde{Y}_{n+1} \in C(\widetilde{X}_{n+1})\bigr) \ge 1 - \alpha
\)
with respect to the randomness of both the original samples and the random permutation. Note that \(\widetilde{Y}_{n+1}\) belongs to \(C(\widetilde{X}_{n+1})\) precisely if its nonconformity score
\[
\widetilde{s}_{n+1} \;=\; s\bigl(\widetilde{X}_{n+1}, \widetilde{Y}_{n+1}\bigr)
\]
does not exceed the \((1-\alpha)\)-quantile \(\widetilde{q}\). Equivalently, \(\widetilde{s}_{n+1}\) is at most the \(\lceil (n+1)(1-\alpha)\rceil\)-th largest among \(\{\widetilde{s}_1,\dots,\widetilde{s}_{n+1}\}\).

By symmetry, \(\widetilde{s}_{n+1}\) is equally likely to appear in any rank among the \(n+1\) scores \(\widetilde{s}_1,\dots,\widetilde{s}_{n+1}\). Hence, the probability that \(\widetilde{s}_{n+1}\) falls \emph{above} that critical rank is at most \(\alpha\). Therefore,
\[
\begin{aligned}
    \Pr\!\Bigl(\widetilde{Y}_{n+1} \,\notin\, C(\widetilde{X}_{n+1})\Bigr)
    &\;\le\;\alpha, \\
    \text{and so} \quad
    \Pr\!\Bigl(\widetilde{Y}_{n+1} \,\in\, C(\widetilde{X}_{n+1})\Bigr)
    &\;\ge\;1-\alpha.
\end{aligned}
\]
Reversing the permutation \(\pi\) simply reverts the data to its original indexing. Because all permutations are equally likely, we conclude that, for the original test point \((X_{n+1}, Y_{n+1})\), 
\[
\Pr\!\bigl(Y_{n+1} \in C(X_{n+1})\bigr) \;\ge\; 1-\alpha.
\]

\end{proof}

\noindent
\textbf{Interpretation in the Context of Two-Threshold Policies.}  
Although Theorem~\ref{thm:conformal_coverage} is stated for a single threshold \(\hat{q}\), the rank-based argument holds equally under mild modifications when additional thresholds are introduced. In the main text, we exploit two thresholds to partition nonconformity scores into regimes that yield single-label predictions, set-valued predictions, or abstentions. The coverage requirement is preserved provided that the relevant thresholds are computed against \(\{\widetilde{s}_1,\dots,\widetilde{s}_n\}\) (the calibration scores) and remain within the same unified conformal scoring framework. As a result, the final coverage probability for the true label \(Y_{n+1}\) remains at least \(1-\alpha\), up to the statistical deviations governed by the i.i.d.\ assumption on \(\{(X_i, Y_i)\}\).

In our method (see Section~\ref{rl_abstention} in the main paper), we further \emph{optimize} these thresholds via reinforcement learning to improve accuracy, set size, and abstention outcomes. Nonetheless, the conformal criterion ensures that the proportion of samples for which the correct label lies outside the conformal set remains bounded by \(\alpha\).

\section{Training via Reinforcement Learning} \label{rl_algorithm}

\cref{alg:conformal_rl} summarizes the training of our proposed adaptive conformal environment and abstention policy.

\begin{algorithm}[h]
   \caption{Conformalized Abstention Policy with Reinforcement Learning}
   \label{alg:conformal_rl}
\begin{algorithmic}
   \State \textbf{Input:} Calibration dataset $\mathcal{D}_{\mathrm{cal}}$, LLM/VLM model $M$, learning rate $\eta$, policy network $\pi_\theta$, cost function $C(\alpha, \beta)$
   \State \textbf{Output:} Optimized thresholds $\hat{q}_{\mathrm{predict}}$, $\hat{q}_{\mathrm{abstain}}$
   \For{each episode}
       \State Sample $\alpha \sim \mathcal{N}(\mu_\theta^{(\alpha)}, \sigma_\theta^{(\alpha)2})$ and $\beta \sim \mathcal{N}(\mu_\theta^{(\beta)}, \sigma_\theta^{(\beta)2})$
       \State Compute nonconformity scores $s_i = 1 - p_{y_i}(\mathbf{x}_i)$ for all $(\mathbf{x}_i, y_i) \in \mathcal{D}_{\mathrm{cal}}$
       \State Calculate quantile thresholds:
       \State \hspace{0.5cm} $\hat{q}_{\mathrm{predict}} = \text{Quantile}(\{s_i\}, (n + 1)(1 - \alpha)/n)$
       \State \hspace{0.5cm} $\hat{q}_{\mathrm{abstain}} = \text{Quantile}(\{s_i\}, (n + 1)(1 - \beta)/n)$
       \For{each test sample $\mathbf{x}$}
           \State Compute $s(\mathbf{x}) = 1 - \max_i p_i(\mathbf{x})$
           \State Compute action probabilities:
           \State \hspace{0.5cm} $p_{\mathrm{single}} = \sigma(-c[s(\mathbf{x}) - \hat{q}_{\mathrm{predict}}])$
           \State \hspace{0.5cm} $p_{\mathrm{abstain}} = \sigma(c[s(\mathbf{x}) - \hat{q}_{\mathrm{abstain}}])$
           \State \hspace{0.5cm} $p_{\mathrm{set}} = 1 - p_{\mathrm{single}} - p_{\mathrm{abstain}}$
           \State Sample action $a \in \{\text{single}, \text{set}, \text{abstain}\}$ based on these probabilities
       \EndFor
       \State Evaluate performance and compute cost $C(\alpha, \beta)$
       \State Compute reward $R(\alpha, \beta) = -C(\alpha, \beta)$
       \State Update policy parameters:
       \State \hspace{0.5cm} $\theta \leftarrow \theta + \eta \cdot R(\alpha, \beta) \nabla_\theta \log \pi_\theta(\alpha, \beta)$
   \EndFor
\end{algorithmic}
\end{algorithm}

\section{Experimental Details} \label{ap_exp}
\subsection{Datasets}
This section provides details about the datasets used in our evaluation. We focus on two groups of datasets: one for Vision-Language Models (VLMs) on multiple-choice visual question answering (MCQA) tasks, and another for Language Models (LLMs) across multiple tasks. Below, we describe the VLM datasets in detail.

\underline{\textbf{Datasets for VLMs}} \hspace{1em}
For the evaluation of Vision-Language Models, we focus on multiple-choice visual question answering (MCQA) tasks. The following datasets are used, each addressing specific aspects of visual understanding and reasoning:

\textbf{Comprehensive Visual Understanding and Reasoning} \hspace{1em}
The \textbf{MMBench} dataset \cite{liu2025mmbench} evaluates a model's ability to perform tasks across 20 distinct capability dimensions, organized into two broad categories: perception and reasoning. It includes approximately 3,000 multiple-choice questions in the test set and 4,000 in the development set. Since the test set lacks ground truth answers, we use the development set for evaluation. Questions have between two to four answer options, and we standardize them to four options by adding randomly sampled incorrect answers when necessary.

\textbf{Out-of-Distribution Instance Counting} \hspace{1em}
The \textbf{OODCV-VQA} dataset \cite{zhao2022ood}, part of a safety evaluation benchmark, focuses on out-of-distribution instance counting tasks. We specifically use the "Digits" subset, where each question involves counting objects in images and has two answer options. To ensure consistency, we augment the options to four by randomly sampling incorrect digits not included in the original options.

\textbf{Scientific Reasoning with Visual Context} \hspace{1em}
The \textbf{ScienceQA} dataset \cite{lu2022learn} tests a model's ability to answer scientific questions across three subjects: natural science, language science, and social science. We use the validation and test portions, selecting only image-based questions with closed-choice answers. The number of options ranges from two to five, and we standardize them to four by adding or removing incorrect options as needed, resulting in 3,952 questions.

\textbf{Multimodal Scene and Instance Understanding} \hspace{1em}
The \textbf{SEED-Bench} dataset \cite{li2023seed} evaluates a model's capabilities across 12 dimensions, including spatial and temporal understanding. For our evaluation, we focus on dimensions 1-9, which are related to image modality and assess tasks such as scene understanding, instance identity, and instance location. We use 14,233 questions from this benchmark, each with four answer options.

\textbf{Diagram Understanding and Reasoning} \hspace{1em}
The \textbf{AI2D} dataset \cite{kembhavi2016diagram} contains over 5,000 diagrams from elementary school science topics, accompanied by more than 15,000 multiple-choice questions. These questions test a model's ability to understand and reason about information presented in diagrams. Each question already includes four answer options, so no modifications are required.

\textbf{Standardization of Options}  \hspace{1em}
To ensure consistency across all datasets, we append two additional choices ("I don’t know" and "None of the above") to the list of options for each question, expanding the total number of options to six. This provides a uniform evaluation framework for all tasks.

\underline{\textbf{Datasets for LLMs:}} \hspace{1em}
To comprehensively evaluate the capabilities of Language Models, we focus on five key natural language processing (NLP) tasks: question answering, reading comprehension, commonsense inference, dialogue response selection, and document summarization. Each task is formulated as a multiple-choice question answering (MCQA) task, where the model must select the correct answer from six possible options (A, B, C, D, E, and F). Below, we describe the datasets used for each task.

\textbf{Question Answering (QA)} \hspace{1em}
For the question answering task, we use the \textbf{MMLU} dataset \cite{hendrycks2020measuring}. MMLU evaluates an LLM's ability to leverage its extensive world knowledge to answer questions across 57 diverse subjects, including elementary mathematics, US history, computer science, and law. These subjects are grouped into four broad categories: humanities, social sciences, STEM, and others (e.g., business, health, and miscellaneous topics). We sample 2,500 instances from each category, resulting in a total of 10,000 questions for evaluation.

\textbf{Reading Comprehension (RC)} \hspace{1em}
The reading comprehension task assesses an LLM's ability to understand and analyze textual contexts, infer meanings, and draw conclusions based on the provided information. For this task, we use the \textbf{CosmosQA} dataset \cite{huang2019cosmos}. CosmosQA focuses on reasoning beyond explicit text spans, requiring models to interpret everyday narratives and infer implicit meanings. Since ground truth labels for the test set are unavailable, we sample 10,000 instances from the training and development sets for evaluation.

\textbf{Commonsense Inference (CI)} \hspace{1em}
Commonsense inference evaluates an LLM's ability to reason about relationships between concepts and events using background knowledge and commonsense understanding. We employ the \textbf{HellaSwag} dataset \cite{zellers2019hellaswag} for this task. HellaSwag focuses on natural language inference, requiring models to select the most plausible continuation of a given event description. Similar to CosmosQA, we sample 10,000 instances from the training and development sets of HellaSwag for evaluation.

\textbf{Dialogue Response Selection (DRS)} \hspace{1em} 
The dialogue response selection task tests an LLM's ability to understand conversational contexts and select appropriate responses that maintain coherence and relevance. For this task, we use the dialogue data from the \textbf{HaluEval} benchmark \cite{li2023halueval}, specifically the \textbf{HaluDial} subset. HaluDial is derived from OpenDialKG (Moon et al., 2019), a knowledge-grounded dialogue dataset, and consists of exactly 10,000 instances for evaluation.

\textbf{Document Summarization (DS)} \hspace{1em}
Document summarization evaluates an LLM's ability to comprehend and condense lengthy documents into concise summaries that capture the main ideas and key information. For this task, we use the summarization data from the \textbf{HaluEval} benchmark \cite{li2023halueval}, specifically the \textbf{HaluSum} subset. HaluSum is derived from the CNN/Daily Mail dataset (See et al., 2017), which focuses on summarizing news articles, and contains exactly 10,000 instances for evaluation.

\textbf{Standardization of Options} \hspace{1em}
To ensure consistency across all datasets, we standardize the number of answer options. While MMLU, CosmosQA, and HellaSwag originally provide four options per question, HaluDial and HaluSum include only two options. For the latter, we augment the options by randomly selecting additional choices from other questions within the same dataset. Additionally, we append two universal options, "I don’t know" and "None of the above," to every question, resulting in six possible options for all datasets.

\subsection{Prompting Templates} \label{prompt_template}
This section describes the prompting strategies and templates used for evaluating VLMs and LLMs. The templates are designed to ensure consistent and effective evaluation across different model families.

\textbf{Prompting Templates for VLMs:} \hspace{1em}
For Vision-Language Models, we adopt a standardized prompting strategy tailored for multiple-choice visual question answering (MCQA) tasks. The template is inspired by the approach used in LLaVA (Liu et al., 2024) and is designed to maximize compatibility across various VLM architectures. The prompt structure is as follows:

\begin{itemize}
    \item The prompt begins with an attached image, serving as the primary visual input for the model.
    \item The question text follows, optionally including a hint if available.
    \item Six answer options are presented line by line, each prefixed with its corresponding letter (A-F). Additional choices such as \textit{``I don't know''} and \textit{``None of the above''} are also included to account for uncertainty.
    \item The prompt concludes with an explicit instruction: \textit{``Answer with the option's letter from the given choices directly.''}
    \item For models requiring a specific multimodal token format, the image is prepended with a designated image token, such as \texttt{\textless image\textgreater} or model-specific tokens like \texttt{DEFAULT\_IMAGE\_TOKEN}, ensuring compatibility with different VLM architectures.
    \item Depending on the model type, the prompt is wrapped within a structured conversational template. Examples include Vicuna-style conversation for LLaVA, structured input for CogVLM, Yi-VL, and Qwen-VL, ensuring consistency in processing.
\end{itemize}

To accommodate the constraints of single-image input in many VLMs, we intentionally exclude few-shot demonstrations from the prompts. The templates are adapted for specific model families, including LLaVA, Yi-VL, Qwen, Monkey, MoE-LLaVA, mPLUG-Owl, and MobileVLM, using their respective official repositories. For CogAgent and InternLM-XComposer2, the templates are sourced from their Hugging Face repositories.

Below is the base prompt template format utilized in our experiments:

\begin{table}[h]
    \centering
    \begin{minipage}{0.45\textwidth}
    \fbox{%
    \begin{minipage}{\textwidth}
    \texttt{%
    \textbf{Image}: \{<Image>\} \\[6pt]
    \textbf{Question}: \{Question Text\} \\[6pt]
    \textbf{Hint}: \{Optional Hint Text\} \\[6pt]
    \textbf{Choices:} \\[4pt]
    A. \{Content of option A\} \\
    B. \{Content of option B\} \\
    C. \{Content of option C\} \\
    D. \{Content of option D\} \\
    E. I don’t know \\
    F. None of the above \\[6pt]
    \textbf{Answer with the option's letter from the given choices directly.}
    }
    \end{minipage}%
    }
    \end{minipage}
    \caption{This table presents the structured prompt template used for multiple-choice question answering in VLMs. Each prompt consists of an attached image, a question (optionally with a hint), and six answer choices, including uncertainty options (\textit{"I don't know"} and \textit{"None of the above"}). To maintain consistency across different VLM architectures, model-specific input tokens (e.g., \texttt{\textless image\textgreater} or \texttt{DEFAULT\_IMAGE\_TOKEN}) are included when necessary. The prompt concludes with a direct instruction for the model to answer using the letter corresponding to its chosen option.}
\end{table}

This template ensures a consistent format for evaluating VLMs across diverse datasets and tasks. The inclusion of six options (A-F) standardizes the evaluation process, while the explicit instruction at the end guides the model to provide a direct response.

\textbf{Prompting Templates for LLMs:} \hspace{1em}
For Language Models, we employ a \textbf{base prompting strategy} without any strategy such as shared instruction or task-specific instruction prompt in order maintain a standardized approach across evaluations. This prompt is designed to evaluate several model performances across multiple tasks, including question answering (QA), reading comprehension (RC), commonsense inference (CI), dialogue response selection (DRS), and document summarization (DS). The prompt template is designed to provide a consistent structure for all tasks while accommodating task-specific information. The structure of the base prompt is as follows:
\begin{itemize}
    \item The prompt begins with the task-specific context, dialogue, or document:
    \begin{itemize}
        \item For \textit{QA tasks}, no background information is included.
        \item For \textit{RC and CI tasks}, the keyword \textit{``Context''} introduces the relevant background information.
        \item For \textit{DRS tasks}, the keyword \textit{``Dialogue''} incorporates the dialogue history.
        \item For \textit{DS tasks}, the keyword \textit{``Document''} includes the document content.
    \end{itemize}
    \item The question is presented next, followed by a list of six answer options:
    \begin{itemize}
        \item Four standard options (A-D) with task-specific content.
        \item Two additional options: \textit{``I don’t know''} and \textit{``None of the above.''}
    \end{itemize}
    \item The model is instructed to provide the letter corresponding to the correct answer.
\end{itemize}

Below is the base prompt template format utilized in our experiments:

\begin{table}[h]
    \centering
    \begin{minipage}{0.45\textwidth}
    \fbox{%
    \begin{minipage}{\textwidth}
    \texttt{%
    \textbf{Context/Dialogue/Document}: \{The context or dialogue history or document corresponding to the following question\} \\[6pt]
    \textbf{Question}: \{Question\} \\[6pt]
    \textbf{Choices:} \\[4pt]
    A. \{Content of option A\} \\
    B. \{Content of option B\} \\
    C. \{Content of option C\} \\
    D. \{Content of option D\} \\
    E. I don’t know \\
    F. None of the above \\[6pt]
    \textbf{Answer with the option's letter from the given choices directly.}
    }
    \end{minipage}%
    }
    \end{minipage}
    \caption{This table presents the structured prompt template used for multiple-choice question answering in LLMs. In the QA setting, no additional background information is included. For the RC and CI tasks, the keyword "Context" is introduced to incorporate relevant background information. Similarly, the keywords "Dialogue" and "Document" are used for DRS and DS tasks, respectively, to integrate necessary context.}
\end{table}

This template ensures a standardized format for evaluating LLMs across diverse tasks. For instruction-finetuned LLMs, the entire prompt input is treated as the user's message, and the "apply\_chat\_template" function is used to transform the prompt into a chat format, ensuring compatibility with chat-based models.

\subsection{Additional Models Evaluated}
This appendix provides additional details on the Vision-Language Models (VLMs) and Large Language Models (LLMs) that complement those discussed in the main body of the paper. These models were evaluated to broaden the scope of our analysis across different architectures and parameter scales.

For \textbf{VLMs}, we include results for several additional models. \textbf{Monkey-Chat 7B} \cite{li2024monkey} is a vision-language model optimized for multimodal chat-based reasoning. \textbf{InternLM-XComposer2-VL 7B} \cite{dong2024internlm} enhances vision-language interaction through structured prompts, while \textbf{Yi-VL 6B} \cite{young2024yi} is a smaller variant of the Yi-VL series, designed for effective image-text understanding. \textbf{CogAgent-VQA 7B} \cite{hong2024cogagent} focuses on visual question answering with robust reasoning capabilities. \textbf{MobileVLMV2 7B} \cite{chu2024mobilevlm} is a lightweight VLM tailored for mobile and edge applications. Additionally, \textbf{mPLUG-Owl2 7B} \cite{ye2024mplug} offers strong image-text understanding capabilities, and \textbf{Qwen-VL-Chat 7B} \cite{bai2023qwen} is designed for dialogue-driven multimodal interactions.

For \textbf{LLMs}, we also present results for the \textbf{Llama-2 7B and 13B} models \cite{touvron2023llama}, which serve as foundation models with strong text generation and reasoning capabilities.

The inclusion of these models extends the scope of our evaluation, providing a comprehensive comparison across diverse architectures and parameter scales.

\subsection{Additional Results} \label{additional_results}
\subsubsection{Results of VLMs}
Additional results in \autoref{tab:auroc_auarc_ap_vlm}, \autoref{tab:coverage_ap_vl}, and \autoref{tab:acc_ap_vl} demonstrate the performance of multiple VLMs mentioned in \autoref{ap_exp} in terms of uncertainty quantification i.e. AUROC vs AUARC, coverage rate vs set size, and accuracy vs expected calibration error respectively. As shown in these tables, our CAP model outperforms other methods in hallucination detection and uncertainty guided selective generation while satisfying the minimum coverage rate of 90\% in all instances and maintaining the middle ground in set size balancing this for all cases.

\begin{table*}[h!]
\centering
\caption{Evaluation of uncertainty quantification: Comparative analysis of the proposed CAP (Ours) meth with standard Least Ambiguous set-valued Classifiers (LAC) \cite{sadinle2019least}, and Adaptive Prediction Sets (APS) \cite{romano2020classification} methods (the best values are in bold). The comparison includes different datasets and VLM models, with quality of uncertainty quantification evaluated using the Area Under the Receiver Operating Characteristic (AUROC) and the Area Under the Accuracy-Rejection Curve (AUARC). Best values are in bold.}
\label{tab:auroc_auarc_ap_vlm}
\begin{adjustbox}{width=\textwidth}
\begin{tabular}{llcccccccccccc}
\toprule
\multirow{2}{*}{\textbf{Model}} & \multirow{2}{*}{\textbf{Method}} & \multicolumn{6}{c}{\textbf{AUROC} ↑ (Hallucination Detection)} & \multicolumn{6}{c}{\textbf{AUARC} ↑ (Uncertainty guided selective generation)} \\
\cmidrule(lr){3-8} \cmidrule(lr){9-14}
& & \textbf{MMB} & \textbf{OOD} & \textbf{SQA} & \textbf{SB} & \textbf{AI2D} & \textbf{Avg.} & \textbf{MMB} & \textbf{OOD} & \textbf{SQA} & \textbf{SB} & \textbf{AI2D} & \textbf{Avg.} \\
\midrule
\multirow{3}{*}{Monkey-Chat-7B}
& APS & 0.6360 & 0.2994 & 0.4916 & 0.5304 & 0.7662 & 0.5447 & 0.9285 & 0.7640 & 0.8950 & 0.8579 & 0.8635 & 0.8618 \\
& LAC & 0.6855 & 0.4151 & 0.6501 & 0.4596 & 0.6716 & 0.5764 & 0.8988 & 0.7137 & 0.8646 & 0.8028 & 0.8413 & 0.8242 \\
& Ours & \textbf{0.7241} & \textbf{0.5182} & \textbf{0.6739} & \textbf{0.5550} & \textbf{0.7340} & \textbf{0.6410} & \textbf{0.9652} & \textbf{0.9174} & \textbf{0.9686} & \textbf{0.9335} & \textbf{0.9747} & \textbf{0.9519} \\
\cmidrule(lr){1-14}
\multirow{3}{*}{InternLM-XComposer2-VL-7B}
& APS & 0.6648 & 0.5000 & 0.7010 & 0.4731 & 0.6421 & 0.5962 & 0.9267 & 0.7999 & 0.9537 & 0.8642 & 0.8879 & 0.8865 \\
& LAC & 0.6861 & 0.5275 & 0.7524 & 0.4810 & 0.6429 & 0.6180 & 0.9001 & 0.7807 & 0.9301 & 0.8322 & 0.8624 & 0.8611 \\
& Ours & \textbf{0.7068} & \textbf{0.6295} & \textbf{0.7909} & \textbf{0.5773} & \textbf{0.7035} & \textbf{0.6816} & \textbf{0.9667} & \textbf{0.9219} & \textbf{0.9762} & \textbf{0.9261} & \textbf{0.9624} & \textbf{0.9507} \\
\cmidrule(lr){1-14}
\multirow{3}{*}{CogAgent-VQA-7B}
& APS & 0.6416 & 0.3448 & 0.4930 & 0.5274 & \textbf{0.5341} & 0.5082 & 0.9240 & 0.7469 & 0.8448 & 0.8741 & 0.7828 & 0.8345 \\
& LAC & 0.7003 & 0.3396 & 0.5693 & 0.4844 & 0.4245 & 0.5036 & 0.8996 & 0.7015 & 0.8130 & 0.8251 & 0.7483 & 0.7975 \\
& Ours & \textbf{0.7432} & \textbf{0.5175} & \textbf{0.6355} & \textbf{0.5346} & 0.4867 & \textbf{0.5835} & \textbf{0.9746} & \textbf{0.9264} & \textbf{0.9608} & \textbf{0.9471} & \textbf{0.9553} & \textbf{0.9528} \\
\cmidrule(lr){1-14}
\multirow{3}{*}{MobileVLM-v2-7B}
& APS & \textbf{0.7646} & 0.3836 & 0.5652 & 0.4153 & 0.4867 & 0.5231 & 0.9610 & 0.8712 & 0.9503 & 0.9296 & 0.8508 & 0.9126 \\
& LAC & 0.7168 & 0.3963 & 0.6777 & 0.4617 & 0.3539 & 0.5213 & 0.9307 & 0.8196 & 0.9133 & 0.8673 & 0.7866 & 0.8635 \\
& Ours & 0.7368 & \textbf{0.5214} & \textbf{0.6672} & \textbf{0.5695} & \textbf{0.4633} & \textbf{0.5916} & \textbf{0.9682} & \textbf{0.9169} & \textbf{0.9698} & \textbf{0.9194} & \textbf{0.9103} & \textbf{0.9369} \\
\cmidrule(lr){1-14}
\multirow{3}{*}{mPLUG-Owl2-7B}
& APS & 0.5347 & 0.4550 & 0.3855 & 0.3421 & 0.4862 & 0.4407 & 0.9625 & 0.8706 & 0.9111 & 0.9134 & 0.8628 & 0.9041 \\
& LAC & 0.6575 & 0.5069 & 0.4828 & 0.3692 & 0.3432 & 0.4719 & 0.9247 & 0.8383 & 0.8677 & 0.8447 & 0.7949 & 0.8541 \\
& Ours & \textbf{0.6920} & \textbf{0.6316} & \textbf{0.5766} & \textbf{0.5169} & \textbf{0.4792} & \textbf{0.5793} & \textbf{0.9650} & \textbf{0.9244} & \textbf{0.9415} & \textbf{0.9051} & \textbf{0.9066} & \textbf{0.9285} \\
\cmidrule(lr){1-14}
\multirow{3}{*}{Qwen-VL-Chat-7B}
& APS & 0.6230 & 0.4610 & 0.5156 & 0.4990 & 0.6786 & 0.5554 & 0.8882 & 0.6872 & 0.8052 & 0.7918 & 0.8536 & 0.8052 \\
& LAC & 0.6557 & 0.4057 & 0.5394 & 0.4624 & 0.6511 & 0.5429 & 0.8593 & 0.6425 & 0.7851 & 0.7616 & 0.8292 & 0.7755 \\
& Ours & \textbf{0.6907} & \textbf{0.5348} & \textbf{0.6079} & \textbf{0.5481} & \textbf{0.6990} & \textbf{0.6161} & \textbf{0.9600} & \textbf{0.9171} & \textbf{0.9313} & \textbf{0.9262} & \textbf{0.9688} & \textbf{0.9407} \\
\cmidrule(lr){1-14}
\multirow{3}{*}{Yi-VL-6B}
& APS & 0.6094 & 0.3616 & 0.5674 & 0.4747 & 0.4486 & 0.4923 & 0.9517 & 0.8790 & 0.9012 & 0.9023 & 0.8747 & 0.9018 \\
& LAC & 0.6785 & 0.4638 & 0.5780 & 0.4387 & 0.4246 & 0.5167 & 0.9198 & 0.8461 & 0.8606 & 0.8501 & 0.8276 & 0.8608 \\
& Ours & \textbf{0.7432} & \textbf{0.6284} & \textbf{0.6446} & \textbf{0.5471} & \textbf{0.5331} & \textbf{0.6193} & \textbf{0.9676} & \textbf{0.9228} & \textbf{0.9551} & \textbf{0.9187} & \textbf{0.9312} & \textbf{0.9391} \\
\cmidrule(lr){1-14}
MoE-LLaVA-Phi2-2.7B
& APS & 0.6359 & 0.5785 & 0.5248 & 0.4199 & 0.4282 & 0.5175 & 0.9446 & 0.7610 & 0.8522 & 0.8815 & 0.8061 & 0.8491 \\
& LAC & 0.6864 & 0.5614 & 0.4810 & 0.4849 & 0.4142 & 0.5256 & 0.9070 & 0.7360 & 0.8083 & 0.8298 & 0.7576 & 0.8077 \\
& Ours & \textbf{0.7360} & \textbf{0.7147} & \textbf{0.5329} & \textbf{0.5772} & \textbf{0.5352} & \textbf{0.6192} & \textbf{0.9655} & \textbf{0.9477} & \textbf{0.9412} & \textbf{0.9342} & \textbf{0.9284} & \textbf{0.9434} \\
\bottomrule
\end{tabular}
\end{adjustbox}
\end{table*}

\begin{table*}[h!]
\centering
\caption{Evaluation of coverage rate (\%)  and set size: Comparative analysis of the proposed CAP (Ours) meth with standard LAC \cite{sadinle2019least}, and APS \cite{romano2020classification} methods. The comparison includes different datasets and VLM models, show casing the satisfied coverage rate and balanced set sizes produced by our method with \underline{underlined} values.}
\label{tab:coverage_ap_vl}
\begin{adjustbox}{width=\textwidth}
\begin{tabular}{llcccccccccccc}
\toprule
\multirow{2}{*}{\textbf{Model}} & \multirow{2}{*}{\textbf{Method}} & \multicolumn{6}{c}{\textbf{Coverage (\%)} ↑} & \multicolumn{6}{c}{\textbf{SS} ↓} \\
\cmidrule(lr){3-8} \cmidrule(lr){9-14}
& & \textbf{MMB} & \textbf{OOD} & \textbf{SQA} & \textbf{SB} & \textbf{AI2D} & \textbf{Avg.} & \textbf{MMB} & \textbf{OOD} & \textbf{SQA} & \textbf{SB} & \textbf{AI2D} & \textbf{Avg.} \\
\midrule
\multirow{3}{*}{Monkey-Chat-7B}
& APS & 97.85 & 96.27 & 98.84 & 96.50 & 97.28 & 97.35 & 3.787 & 3.669 & 3.455 & 4.013 & 4.040 & 3.793 \\
& LAC & 89.45 & 88.75 & 90.44 & 89.22 & 90.98 & 89.77 & 1.611 & 2.181 & 1.656 & 2.505 & 2.346 & 2.060 \\
& Ours & \underline{93.33} & \underline{91.35} & \underline{94.69} & \underline{92.03} & \underline{94.36} & \underline{93.15} & \underline{2.383} & \underline{2.987} & \underline{2.567} & \underline{3.285} & \underline{3.017} & \underline{2.848} \\
\cmidrule(lr){1-14}
\multirow{3}{*}{InternLM-XComposer2-VL-7B}
& APS & 96.57 & 92.48 & 98.74 & 94.46 & 96.28 & 95.71 & 3.479 & 2.575 & 3.383 & 3.578 & 3.673 & 3.338 \\
& LAC & 89.17 & 88.96 & 89.58 & 89.90 & 89.87 & 89.50 & 1.966 & 1.819 & 1.443 & 2.584 & 2.358 & 2.034 \\
& Ours & \underline{93.51} & \underline{90.01} & \underline{92.97} & \underline{90.21} & \underline{92.43} & \underline{91.82} & \underline{2.763} & \underline{2.457} & \underline{1.926} & \underline{3.123} & \underline{2.902} & \underline{2.634} \\
\cmidrule(lr){1-14}
\multirow{3}{*}{CogAgent-VQA-7B}
& APS & 98.54 & 95.64 & 97.47 & 95.94 & 93.83 & 96.28 & 2.997 & 2.944 & 2.833 & 2.996 & 3.240 & 3.002 \\
& LAC & 90.68 & 90.37 & 90.14 & 89.36 & 90.65 & 90.24 & 1.665 & 1.971 & 1.895 & 1.975 & 2.640 & 2.030 \\
& Ours & \underline{94.15} & \underline{92.12} & \underline{93.53} & \underline{93.59} & \underline{94.25} & \underline{93.53} & \underline{2.175} & \underline{2.757} & \underline{2.506} & \underline{3.015} & \underline{3.652} & \underline{2.821} \\
\cmidrule(lr){1-14}
\multirow{3}{*}{MobileVLM-v2-7B}
& APS & 97.99 & 96.27 & 99.04 & 97.67 & 95.87 & 97.37 & 3.439 & 3.074 & 3.610 & 3.494 & 3.866 & 3.497 \\
& LAC & 89.63 & 90.86 & 89.07 & 89.49 & 90.23 & 89.86 & 1.629 & 2.153 & 1.625 & 2.106 & 2.925 & 2.088 \\
& Ours & \underline{92.78} & \underline{91.49} & \underline{94.18} & \underline{91.53} & \underline{90.23} & \underline{92.04} & \underline{2.159} & \underline{2.623} & \underline{2.329} & \underline{2.567} & \underline{3.448} & \underline{2.625} \\
\cmidrule(lr){1-14}
\multirow{3}{*}{mPLUG-Owl2-7B}
& APS & 99.27 & 95.08 & 98.18 & 97.09 & 95.81 & 97.09 & 3.365 & 2.485 & 3.346 & 3.431 & 3.379 & 3.201 \\
& LAC & 89.81 & 89.52 & 91.40 & 89.94 & 90.34 & 90.20 & 1.727 & 1.689 & 2.070 & 2.432 & 2.624 & 2.109 \\
& Ours & \underline{92.65} & \underline{91.28} & \underline{91.91} & \underline{91.94} & \underline{89.52} & \underline{91.46} & \underline{2.080} & \underline{2.062} & \underline{2.401} & \underline{2.753} & \underline{2.934} & \underline{2.446} \\
\cmidrule(lr){1-14}
\multirow{3}{*}{Qwen-VL-Chat-7B}
& APS & 96.21 & 93.46 & 92.01 & 92.97 & 96.70 & 94.27 & 3.413 & 3.589 & 3.349 & 3.692 & 3.796 & 3.568 \\
& LAC & 88.44 & 88.75 & 88.11 & 89.21 & 89.69 & 88.84 & 1.990 & 3.049 & 2.451 & 2.945 & 2.394 & 2.566 \\
& Ours & \underline{93.01} & \underline{90.37} & \underline{90.44} & \underline{91.06} & \underline{93.94} & \underline{91.76} & \underline{2.665} & \underline{3.673} & \underline{3.074} & \underline{3.504} & \underline{3.061} & \underline{3.195} \\
\cmidrule(lr){1-14}
\multirow{3}{*}{Yi-VL-6B}
& APS & 98.63 & 95.43 & 98.13 & 95.94 & 96.77 & 96.98 & 3.326 & 2.506 & 3.503 & 3.116 & 3.491 & 3.189 \\
& LAC & 90.22 & 89.94 & 89.78 & 89.84 & 91.01 & 90.16 & 1.621 & 1.536 & 2.009 & 2.106 & 2.514 & 1.957 \\
& Ours & \underline{93.38} & \underline{91.35} & \underline{92.41} & \underline{90.11} & \underline{90.99} & \underline{91.61} & \underline{2.082} & \underline{1.962} & \underline{2.574} & \underline{2.522} & \underline{2.927} & \underline{2.414} \\
\cmidrule(lr){1-14}
\multirow{3}{*}{MoE-LLaVA-Phi2-2.7B}
& APS  & 99.50 & 93.95 & 97.07 & 97.60 & 96.50 & 96.92 & 3.4961 & 2.2651 & 3.2969 & 3.3834 & 3.3425 & 3.1568 \\
& LAC  & 89.26 & 89.17 & 90.84 & 89.66 & 90.08 & 89.80 & 1.5843 & 1.5204 & 2.0976 & 2.0021 & 2.4891 & 1.9387 \\
& Ours & \underline{92.10} & \underline{91.63} & \underline{92.67} & \underline{91.34} & \underline{90.84} & \underline{91.72} &\underline{2.0461} & \underline{2.1280} & \underline{2.6631} & \underline{2.5178} & \underline{2.9515} & \underline{2.4613} \\
\bottomrule
\end{tabular}
\end{adjustbox}
\end{table*}

\begin{table*}[h!]
\centering
\caption{Evaluation of accuracy (\%)  and ECE: Comparative analysis of the proposed CAP (Ours) meth with standard LAC \cite{sadinle2019least}, and APS \cite{romano2020classification} methods. The comparison includes different datasets and VLM models, demonstrating the significant reduction in expected calibration error while improving overall accuracy.}
\label{tab:acc_ap_vl}
\begin{adjustbox}{width=\textwidth}
\begin{tabular}{llcccccccccccc}
\toprule
\multirow{2}{*}{\textbf{Model}} & \multirow{2}{*}{\textbf{Method}} & \multicolumn{6}{c}{\textbf{Accuracy (\%)} ↑} & \multicolumn{6}{c}{\textbf{ECE} ↓} \\
\cmidrule(lr){3-8} \cmidrule(lr){9-14}
& & \textbf{MMB} & \textbf{OOD} & \textbf{SQA} & \textbf{SB} & \textbf{AI2D} & \textbf{Avg.} & \textbf{MMB} & \textbf{OOD} & \textbf{SQA} & \textbf{SB} & \textbf{AI2D} & \textbf{Avg.} \\
\midrule
\multirow{3}{*}{Monkey-Chat-7B}
& APS & 81.40 & 76.75 & 79.27 & 72.33 & 73.58 & 76.67 & 0.2134 & 0.3583 & 0.2696 & 0.2825 & 0.3237 & 0.2895 \\
& LAC & 81.26 & 77.06 & 80.89 & 72.58 & 74.53 & 77.26 & 0.1480 & 0.3042 & 0.1857 & 0.2494 & 0.2608 & 0.2296 \\
& Ours & \textbf{84.03} & \textbf{78.61} & \textbf{82.37} & \textbf{74.57} & \textbf{78.70} & \textbf{79.66} & \textbf{0.0159} & \textbf{0.0336} & \textbf{0.0190} & \textbf{0.0336} & \textbf{0.0207} & \textbf{0.0246} \\
\cmidrule(lr){1-14}
\multirow{3}{*}{InternLM-XComposer2-VL-7B}
& APS & 76.72 & 77.04 & 81.73 & 70.80 & 72.52 & 75.76 & 0.1805 & 0.2179 & 0.1727 & 0.2203 & 0.2417 & 0.2066 \\
& LAC & 77.30 & 77.88 & 82.72 & 71.72 & 73.13 & 76.55 & 0.1284 & 0.1871 & 0.1073 & 0.2093 & 0.1776 & 0.1620 \\
& Ours & \textbf{78.46} & \textbf{78.07} & \textbf{84.10} & \textbf{72.15} & \textbf{75.02} & \textbf{77.56} & \textbf{0.0341} & \textbf{0.0173} & \textbf{0.0246} & \textbf{0.0593} & \textbf{0.0289} & \textbf{0.0328} \\
\cmidrule(lr){1-14}
\multirow{3}{*}{CogAgent-VQA-7B}
& APS & 81.07 & 76.86 & 75.98 & 76.03 & 67.99 & 75.59 & 0.2310 & 0.3614 & 0.3043 & 0.2327 & 0.3148 & 0.2889 \\
& LAC & 80.55 & 76.91 & 76.16 & 75.48 & 67.98 & 75.42 & 0.1608 & 0.3407 & 0.2340 & 0.2151 & 0.2912 & 0.2484 \\
& Ours & \textbf{83.29} & \textbf{79.72} & \textbf{79.25} & \textbf{76.18} & \textbf{69.60} & \textbf{77.61} & \textbf{0.0134} & \textbf{0.0470} & \textbf{0.0366} & \textbf{0.0246} & \textbf{0.0110} & \textbf{0.0265} \\
\cmidrule(lr){1-14}
\multirow{3}{*}{MobileVLM-v2-7B}
& APS & 80.79 & 74.86 & 78.11 & \textbf{74.30} & 63.37 & 74.29 & 0.1460 & 0.2230 & 0.1790 & 0.1691 & 0.2838 & 0.2002 \\
& LAC & 80.85 & 75.23 & 78.90 & 74.28 & 63.41 & 74.53 & 0.1202 & 0.2239 & 0.1363 & 0.1927 & 0.2932 & 0.1933 \\
& Ours & \textbf{82.12} & \textbf{75.38} & \textbf{79.66} & 73.95 & \textbf{64.77} & \textbf{75.18} & \textbf{0.0464} & \textbf{0.0503} & \textbf{0.0306} & \textbf{0.0780} & \textbf{0.0643} & \textbf{0.0539} \\
\cmidrule(lr){1-14}
\multirow{3}{*}{mPLUG-Owl2-7B}
& APS & 78.94 & 80.48 & 73.87 & \textbf{70.54} & 65.42 & 73.85 & 0.1578 & 0.1858 & 0.2260 & 0.1982 & 0.2588 & 0.2053 \\
& LAC & 78.67 & 79.91 & 74.53 & 69.76 & 64.91 & 73.55 & 0.1453 & 0.1574 & 0.2093 & 0.2384 & 0.2857 & 0.2072 \\
& Ours & \textbf{79.78} & \textbf{80.88} & \textbf{75.28} & 69.88 & \textbf{65.96} & \textbf{74.36} & \textbf{0.0473} & \textbf{0.0352} & \textbf{0.0439} & \textbf{0.0863} & \textbf{0.0668} & \textbf{0.0559} \\
\cmidrule(lr){1-14}
\multirow{3}{*}{Qwen-VL-Chat-7B}
& APS & 76.71 & 62.45 & 70.04 & 66.86 & 71.25 & 69.46 & 0.2350 & 0.3944 & 0.2487 & 0.2937 & 0.3211 & 0.2986 \\
& LAC & 76.78 & 63.91 & 71.12 & 67.42 & 72.10 & 70.27 & 0.1653 & 0.3684 & 0.2249 & 0.2739 & 0.2510 & 0.2567 \\
& Ours & \textbf{79.70} & \textbf{67.14} & \textbf{73.18} & \textbf{69.89} & \textbf{76.11} & \textbf{73.20} & \textbf{0.0167} & \textbf{0.0316} & \textbf{0.0403} & \textbf{0.0379} & \textbf{0.0087} & \textbf{0.0270} \\
\cmidrule(lr){1-14}
\multirow{3}{*}{Yi-VL-6B}
& APS & 80.54 & \textbf{81.23} & 74.40 & \textbf{74.59} & 67.98 & 75.75 & 0.1694 & 0.1720 & 0.2553 & 0.1837 & 0.2621 & 0.2085 \\
& LAC & 80.70 & 80.67 & 74.77 & 74.18 & 68.29 & 75.72 & 0.1263 & 0.1594 & 0.2136 & 0.2019 & 0.2565 & 0.1915 \\
& Ours & \textbf{81.82} & 81.05 & \textbf{76.03} & 73.99 & \textbf{69.72} & \textbf{76.52} & \textbf{0.0308} & \textbf{0.0332} & \textbf{0.0287} & \textbf{0.0650} & \textbf{0.0504} & \textbf{0.0416} \\
\cmidrule(lr){1-14}
\multirow{3}{*}{MoE-LLaVA-Phi2-2.7B}
& APS & 79.51 & 82.14 & 72.74 & 74.51 & 66.56 & 75.89 & 0.2067 & 0.2863 & 0.2687 & 0.2476 & 0.3266 & 0.2672 \\
& LAC & 79.80 & 81.16 & 74.08 & 74.86 & 66.66 & 75.71 & 0.1377 & 0.2385 & 0.2212 & 0.2100 & 0.2825 & 0.2180 \\
& Ours & \textbf{81.62} & \textbf{83.06} & \textbf{74.44} & \textbf{75.86} & \textbf{69.02} & \textbf{76.80} & \underline{0.0224} & \underline{0.0991} & \underline{0.0259} & \underline{0.0333} & \underline{0.0238} & \underline{0.0409} \\
\bottomrule
\end{tabular}
\end{adjustbox}
\end{table*}

\subsubsection{Results of LLMs}
Additional results in \autoref{tab:auroc_auarc_ap_llm}, \autoref{tab:coverage_ap_llm}, and \autoref{tab:acc_ap_llm} demonstrate the performance of Llama-2 series models (7B and 13B) discussed in \autoref{ap_exp} in terms of uncertainty quantification i.e. AUROC vs AUARC, coverage rate vs set size, and accuracy vs expected calibration error respectively. As shown in these tables, our CAP model outperforms other methods in hallucination detection and uncertainty guided selective generation while satisfying the minimum coverage rate of 90\% in all instances and maintaining the middle ground in set size balancing this for all cases.

\begin{table*}[h!]
\centering
\caption{Evaluation of uncertainty quantification: Comparative analysis of the proposed CAP (Ours) meth with standard Least Ambiguous set-valued Classifiers (LAC) \cite{sadinle2019least}, and Adaptive Prediction Sets (APS) \cite{romano2020classification} methods (the best values are in bold). The comparison includes different datasets and LLM models, with quality of uncertainty quantification evaluated using the Area Under the Receiver Operating Characteristic (AUROC) and the Area Under the Accuracy-Rejection Curve (AUARC). Best values are in bold.}
\label{tab:auroc_auarc_ap_llm}
\begin{adjustbox}{width=\textwidth}
\begin{tabular}{llcccccccccccc}
\toprule
\multirow{2}{*}{\textbf{Model}} & \multirow{2}{*}{\textbf{Method}} & \multicolumn{6}{c}{\textbf{AUROC} ↑ (Hallucination Detection)} & \multicolumn{6}{c}{\textbf{AUARC} ↑ (Uncertainty guided selective generation)} \\
\cmidrule(lr){3-8} \cmidrule(lr){9-14}
& & \textbf{HSwag} & \textbf{HDial} & \textbf{CQA} & \textbf{HSum} & \textbf{MMLU} & \textbf{Avg.} & \textbf{HSwag} & \textbf{HDial} & \textbf{CQA} & \textbf{HSum} & \textbf{MMLU} & \textbf{Avg.} \\
\midrule
\multirow{3}{*}{Llama2-7B}
& APS & 0.4884 & 0.4646 & 0.6378 & 0.6353 & 0.4495 & 0.5351 & 0.3473 & 0.3301 & 0.5296 & 0.2962 & 0.5774 & 0.4161 \\
& LAC & 0.4079 & 0.2623 & 0.5490 & 0.7205 & 0.3594 & 0.4598 & 0.3395 & 0.2891 & 0.5185 & 0.2923 & 0.5496 & 0.3978 \\
& Ours & \textbf{0.7066} & \textbf{0.7040} & \textbf{0.7724} & \textbf{0.7672} & \textbf{0.6324} & \textbf{0.7165} & \textbf{0.8681} & \textbf{0.8354} & \textbf{0.9599} & \textbf{0.9078} & \textbf{0.8935} & \textbf{0.8929} \\
\cmidrule(lr){1-14}
\multirow{3}{*}{Llama2-13B}
& APS & 0.6225 & 0.3460 & 0.5186 & 0.4092 & 0.4132 & 0.4619 & 0.5788 & 0.5065 & 0.7893 & 0.4709 & 0.7455 & 0.6182 \\
& LAC & 0.4685 & 0.2007 & 0.6377 & 0.3478 & 0.3808 & 0.4071 & 0.5591 & 0.4801 & 0.7710 & 0.4580 & 0.6950 & 0.5926 \\
& Ours & \textbf{0.6396} & \textbf{0.5043} & \textbf{0.7159} & \textbf{0.6255} & \textbf{0.5572} & \textbf{0.6085} & \textbf{0.9254} & \textbf{0.8134} & \textbf{0.9685} & \textbf{0.8986} & \textbf{0.9177} & \textbf{0.9047} \\
\bottomrule
\end{tabular}
\end{adjustbox}
\end{table*}

\begin{table*}[h!]
\centering
\caption{Evaluation of coverage rate (\%)  and set size: Comparative analysis of the proposed CAP (Ours) meth with standard LAC \cite{sadinle2019least}, and APS \cite{romano2020classification} methods. The comparison includes different datasets and LLM models, show casing the satisfied coverage rate and balanced set sizes produced by our method with \underline{underlined} values.}
\label{tab:coverage_ap_llm}
\begin{adjustbox}{width=\textwidth}
\begin{tabular}{llcccccccccccc}
\toprule
\multirow{2}{*}{\textbf{Model}} & \multirow{2}{*}{\textbf{Method}} & \multicolumn{6}{c}{\textbf{Coverage (\%)} ↑} & \multicolumn{6}{c}{\textbf{SS} ↓} \\
\cmidrule(lr){3-8} \cmidrule(lr){9-14}
& & \textbf{HSwag} & \textbf{HDial} & \textbf{CQA} & \textbf{HSum} & \textbf{MMLU} & \textbf{Avg.} & \textbf{HSwag} & \textbf{HDial} & \textbf{CQA} & \textbf{HSum} & \textbf{MMLU} & \textbf{Avg.} \\
\midrule
\multirow{3}{*}{Llama2-7B}
& APS & 90.02 & 90.44 & 91.78 & 89.72 & 92.50 & 90.89 & 3.346 & 3.257 & 2.661 & 3.227 & 3.319 & 3.162 \\
& LAC & 90.66 & 89.96 & 90.08 & 89.22 & 90.54 & 90.09 & 3.253 & 3.251 & 2.275 & 3.423 & 3.021 & 3.044 \\
& Ours & \underline{90.38} & \underline{90.42} & \underline{91.22} & \underline{89.78} & \underline{91.04} & \underline{90.56} & \underline{3.378} & \underline{3.252} & \underline{2.316} & \underline{3.360} & \underline{3.191} & \underline{3.099} \\
\cmidrule(lr){1-14}
\multirow{3}{*}{Llama2-13b}
& APS & 89.70 & 90.32 & 97.06 & 90.26 & 95.86 & 92.64 & 2.801 & 2.571 & 2.881 & 2.306 & 3.320 & 2.776 \\
& LAC & 89.88 & 90.62 & 90.52 & 89.98 & 89.18 & 90.03 & 2.497 & 2.535 & 1.568 & 2.117 & 2.578 & 2.259 \\
& Ours & \underline{90.11} & \underline{90.41} & \underline{94.40} & \underline{90.30} & \underline{93.62} & \underline{91.77} & \underline{3.071} & \underline{2.537} & \underline{2.465} & \underline{2.122} & \underline{3.104} & \underline{2.660} \\
\bottomrule
\end{tabular}
\end{adjustbox}
\end{table*}

\begin{table*}[h!]
\centering
\caption{Evaluation of accuracy (\%)  and ECE: Comparative analysis of the proposed CAP (Ours) meth with standard LAC \cite{sadinle2019least}, and APS \cite{romano2020classification} methods. The comparison includes different datasets and LLM models, demonstrating the significant reduction in expected calibration error while improving overall accuracy.}
\label{tab:acc_ap_llm}
\begin{adjustbox}{width=\textwidth}
\begin{tabular}{llcccccccccccc}
\toprule
\multirow{2}{*}{\textbf{Model}} & \multirow{2}{*}{\textbf{Method}} & \multicolumn{6}{c}{\textbf{Accuracy (\%)} ↑} & \multicolumn{6}{c}{\textbf{ECE} ↓} \\
\cmidrule(lr){3-8} \cmidrule(lr){9-14}
& & \textbf{HSwag} & \textbf{HDial} & \textbf{CQA} & \textbf{HSum} & \textbf{MMLU} & \textbf{Avg.} & \textbf{HSwag} & \textbf{HDial} & \textbf{CQA} & \textbf{HSum} & \textbf{MMLU} & \textbf{Avg.} \\
\midrule
\multirow{3}{*}{Llama2-7B}
& APS & 54.86 & 50.54 & 74.43 & 60.78 & 59.33 & 59.99 & 0.5720 & 0.5934 & 0.5085 & 0.6176 & 0.4894 & 0.5562 \\
& LAC & 55.05 & 50.24 & 74.10 & 59.23 & 59.11 & 59.55 & 0.5784 & 0.5927 & 0.4915 & 0.6127 & 0.4703 & 0.5491 \\
& Ours & \textbf{61.00} & \textbf{57.78} & \textbf{80.32} & \textbf{68.36} & \textbf{64.13} & \textbf{66.32} & \textbf{0.0606} & \textbf{0.0572} & \textbf{0.1953} & \textbf{0.1693} & \textbf{0.0414} & \textbf{0.1048} \\
\cmidrule(lr){1-14}
\multirow{3}{*}{Llama2-13b}
& APS & 70.38 & \textbf{66.84} & 83.42 & 72.59 & 65.48 & 71.74 & 0.4165 & 0.4207 & 0.3462 & 0.4666 & 0.3930 & 0.4086 \\
& LAC & 70.46 & 66.69 & 83.26 & 72.34 & 64.79 & 71.49 & 0.4183 & 0.4373 & 0.2808 & 0.4638 & 0.3444 & 0.3889 \\
& Ours & \textbf{73.31} & 64.28 & \textbf{85.89} & \textbf{73.44} & \textbf{68.11} & \textbf{73.00} & \textbf{0.0814} & \textbf{0.0721} & \textbf{0.1138} & \textbf{0.1559} & \textbf{0.0203} & \textbf{0.0887} \\
\bottomrule
\end{tabular}
\end{adjustbox}
\end{table*}

\subsubsection{Accuracy vs ECE:}  \label{accuracy_ece_appendix}
\autoref{fig:acc_ece_ap_vlm} shows the results of accuracy vs ECE achieved using CAP versus APS and LAC across multiple VLMs. Lower ECE values indicate better calibration, signifying that confidence scores are more reliable indicators of prediction accuracy. As shown in these figures, CAP was able to improve the accuracy while \textbf{significantly} reducing the expected calibration error. Moreover, \autoref{fig:acc_ece_ap_llm} shows the same trend in LLMs consistently reducing ECE while improving accuracy across all tasks and datasets.

\begin{figure*}[h!]
    \centering
    \makebox[\textwidth]{\includegraphics[width=0.9\textwidth]{figures/legend.pdf}}
    \makebox[\textwidth]{
        \begin{tabular}{ccc}
            \includegraphics[width=0.30\textwidth]{figures/mmbench_plot.pdf} &
            \includegraphics[width=0.30\textwidth]{figures/scienceqa_plot.pdf} &
            \includegraphics[width=0.30\textwidth]{figures/oodcv_plot.pdf} \\
        \end{tabular}
    }
    \makebox[\textwidth]{
        \begin{tabular}{cc}
            \includegraphics[width=0.30\textwidth]{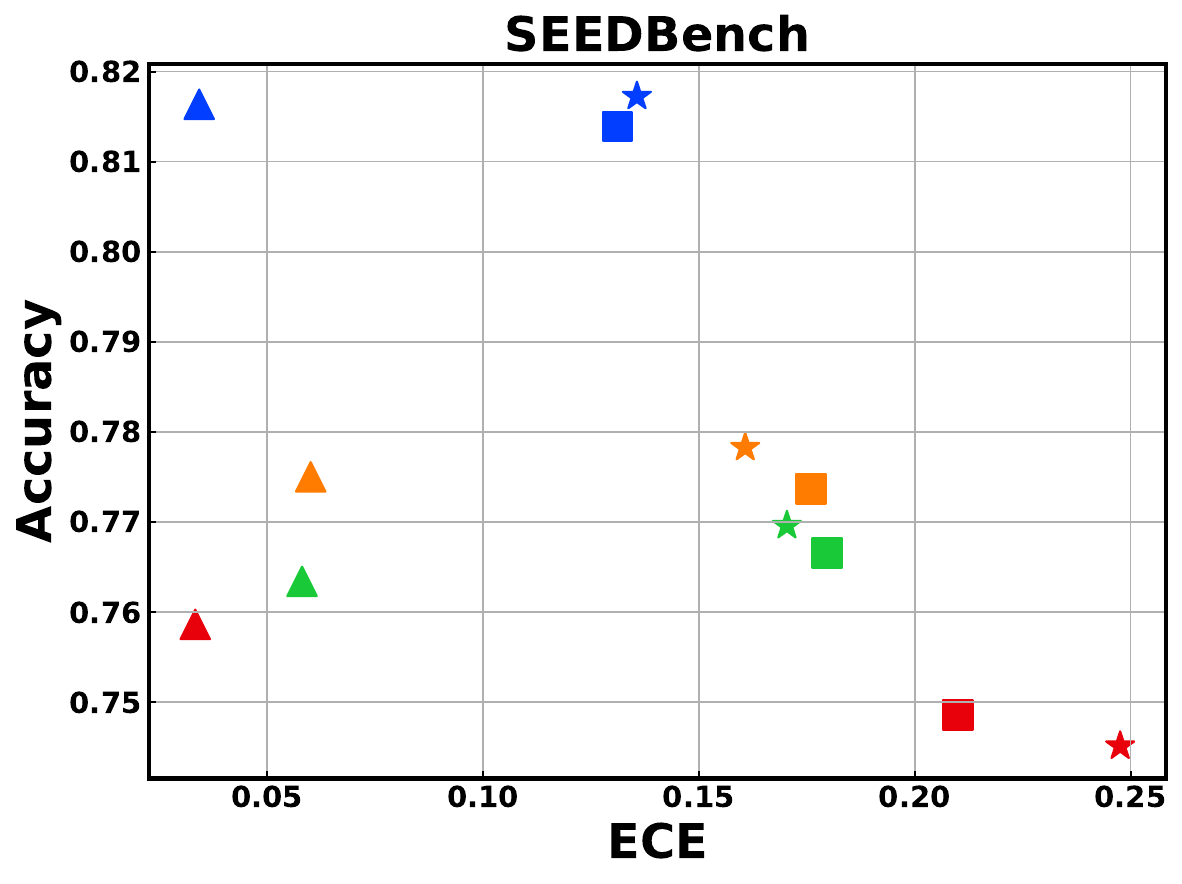} &
            \includegraphics[width=0.30\textwidth]{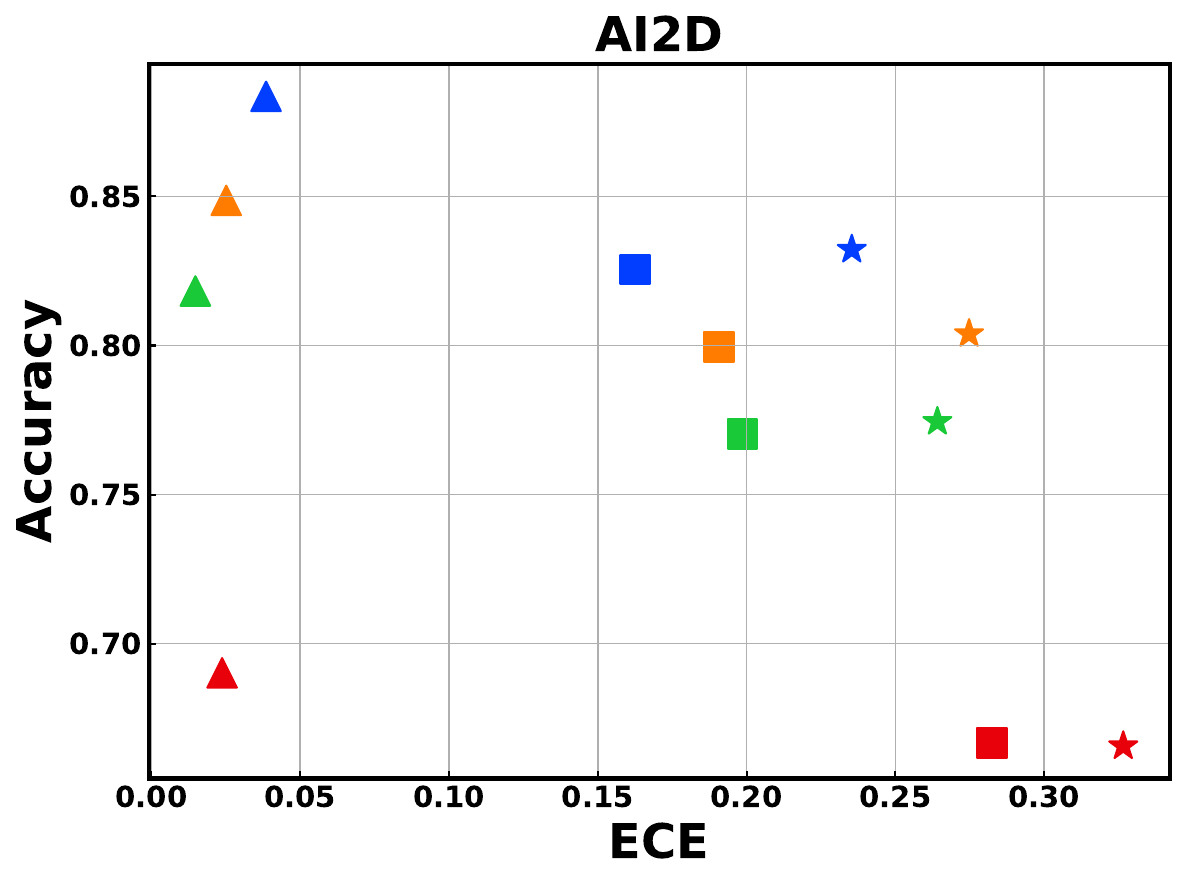} \\
        \end{tabular}
    }
\caption{Accuracy vs. Expected Calibration Error (ECE) comparison of CAP, APS, and LAC across various VLMs and five datasets: MMBench, ScienceQA, OODCV, SEEDBench, and AI2D. An ideal model has high accuracy and low ECE (upper-left). ATCP shows significant ECE improvement over baselines.}
    \label{fig:annotated_grid_ap}
\end{figure*}

\begin{figure*}[h!]
    \centering
    \makebox[\textwidth]{\includegraphics[width=\textwidth]{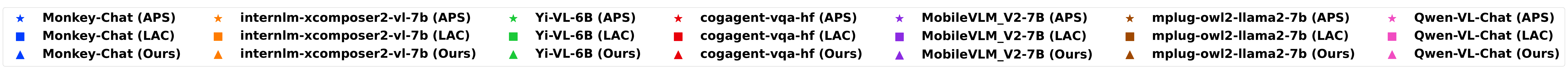}}
    \makebox[\textwidth]{
        \begin{tabular}{ccc}
            \includegraphics[width=0.30\textwidth]{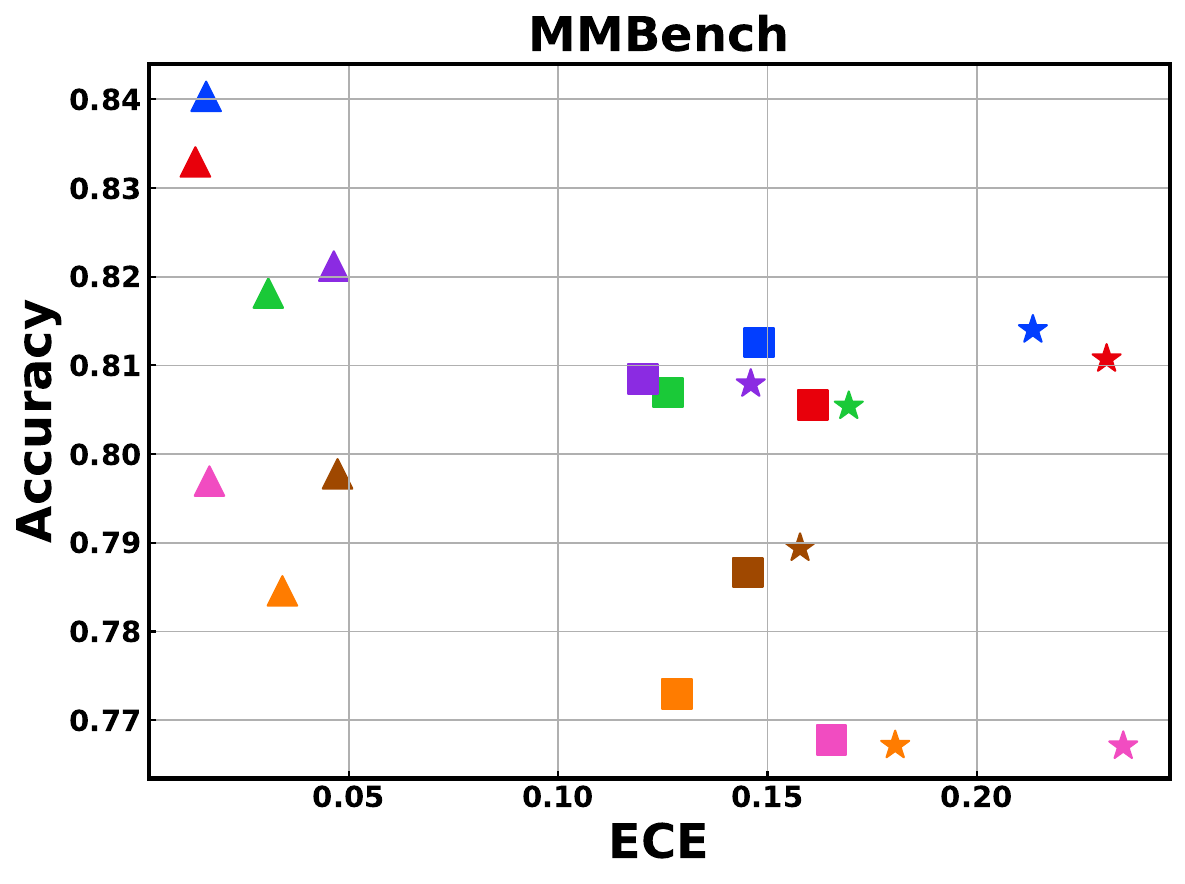} &
            \includegraphics[width=0.30\textwidth]{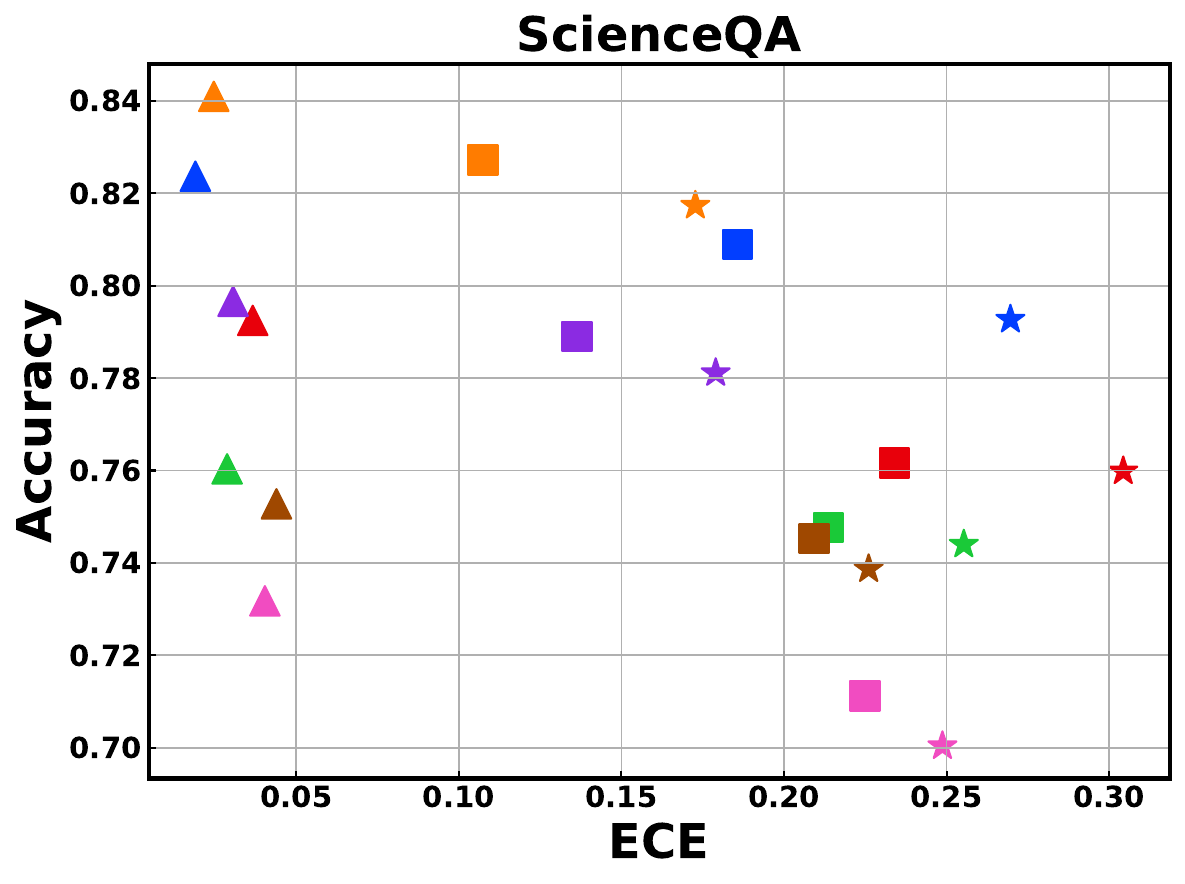} &
            \includegraphics[width=0.30\textwidth]{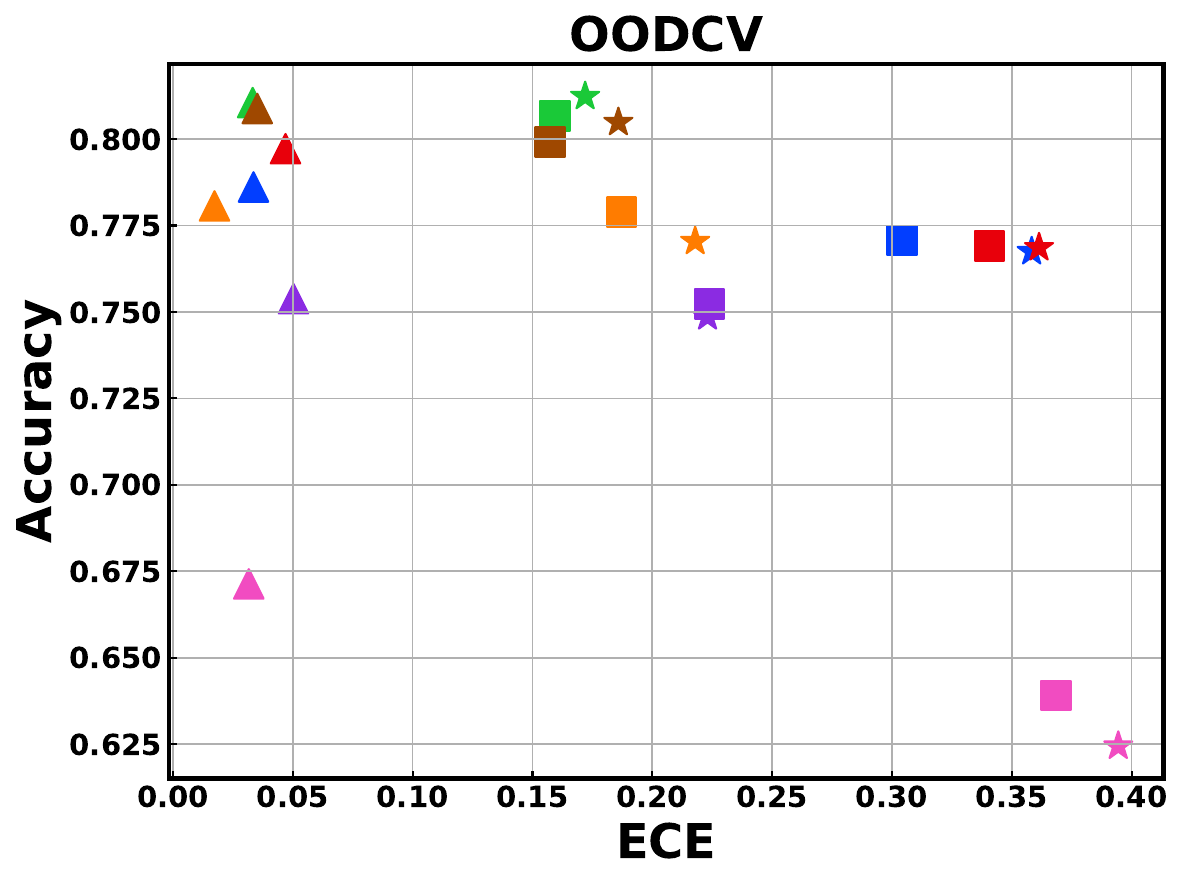} \\
        \end{tabular}
    }
    \makebox[\textwidth]{
        \begin{tabular}{cc}
            \includegraphics[width=0.30\textwidth]{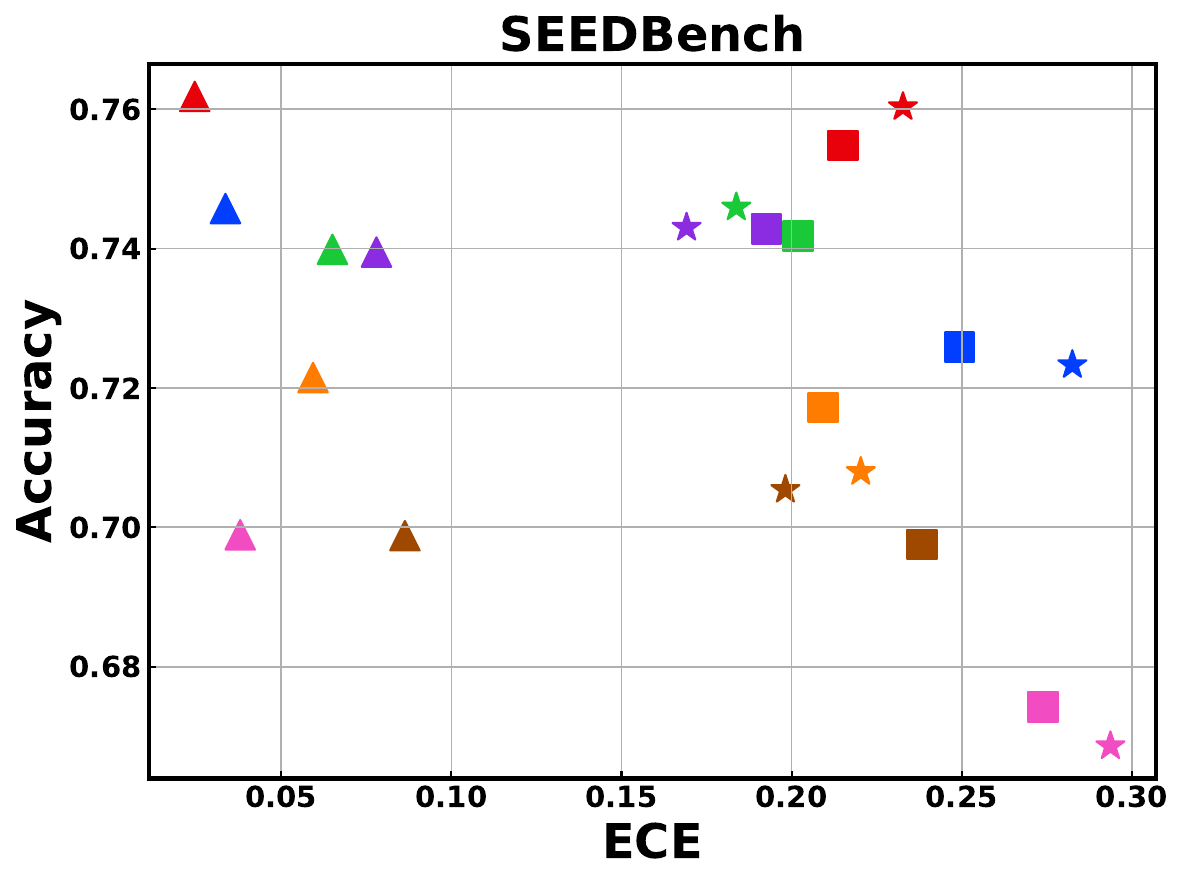} &
            \includegraphics[width=0.30\textwidth]{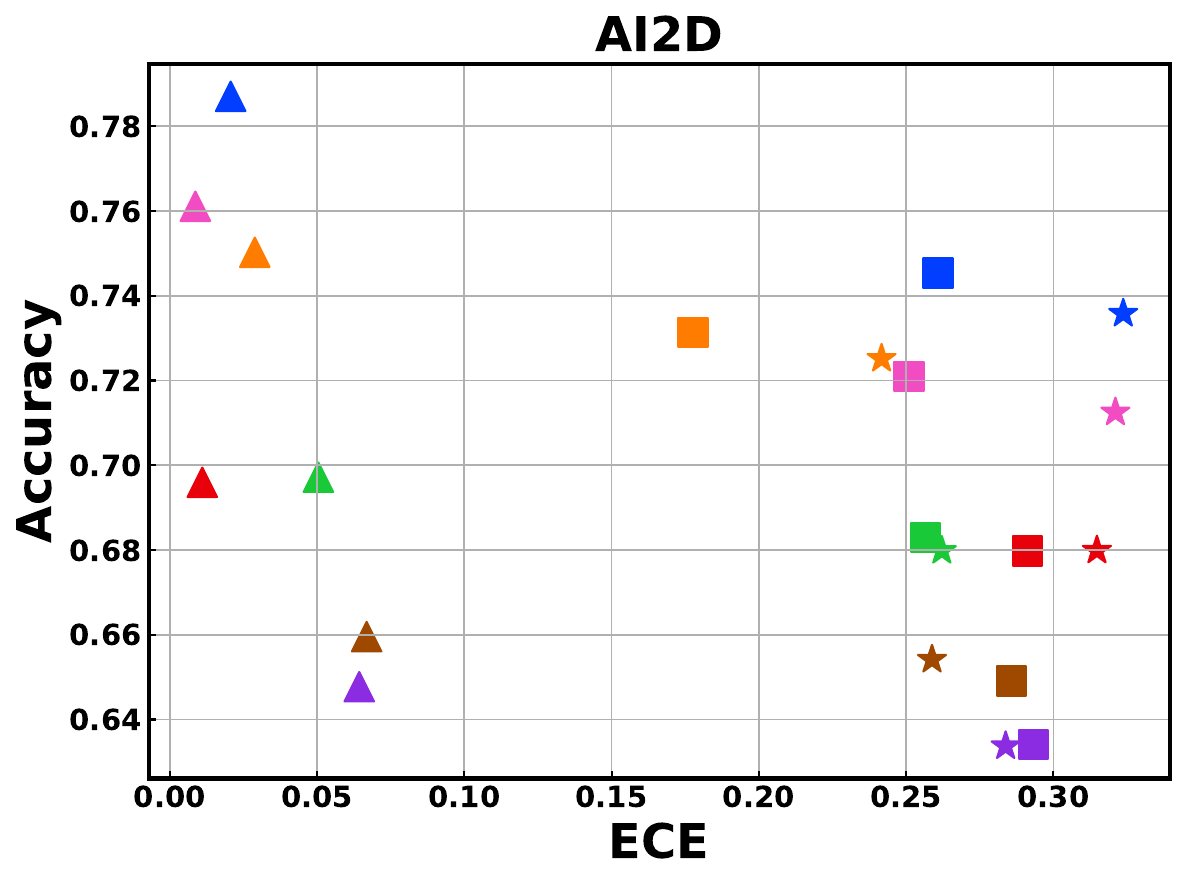} \\
        \end{tabular}
    }
    \caption{Accuracy versus Expected Calibration Error (ECE) comparison between CAP, APS and LAC methods across different VLMs and five datasets i.e. MMBench, ScienceQA, OODCV, SEEDBench, AI2D. The ideal model should have high accuracy and low ECE, indicating accurate predictions with well calibrated uncertainty quantification (upper-left of the plot). The ECE of ATCP shows significant improvement compared to baseline methods.}
    \label{fig:acc_ece_ap_vlm}
\end{figure*}

\begin{figure*}[h!]
    \centering
    \makebox[\textwidth]{\includegraphics[width=0.5\textwidth]{figures_llm/legend.pdf}}
    \makebox[\textwidth]{
        \begin{tabular}{ccc}
            \includegraphics[width=0.30\textwidth]{figures_llm/cosmosqa_plot.pdf} &
            \includegraphics[width=0.30\textwidth]{figures_llm/halu_dialogue_plot.pdf} &
            \includegraphics[width=0.30\textwidth]{figures_llm/halu_summarization_plot.pdf} \\
        \end{tabular}
    }
    \makebox[\textwidth]{
        \begin{tabular}{cc}
            \includegraphics[width=0.30\textwidth]{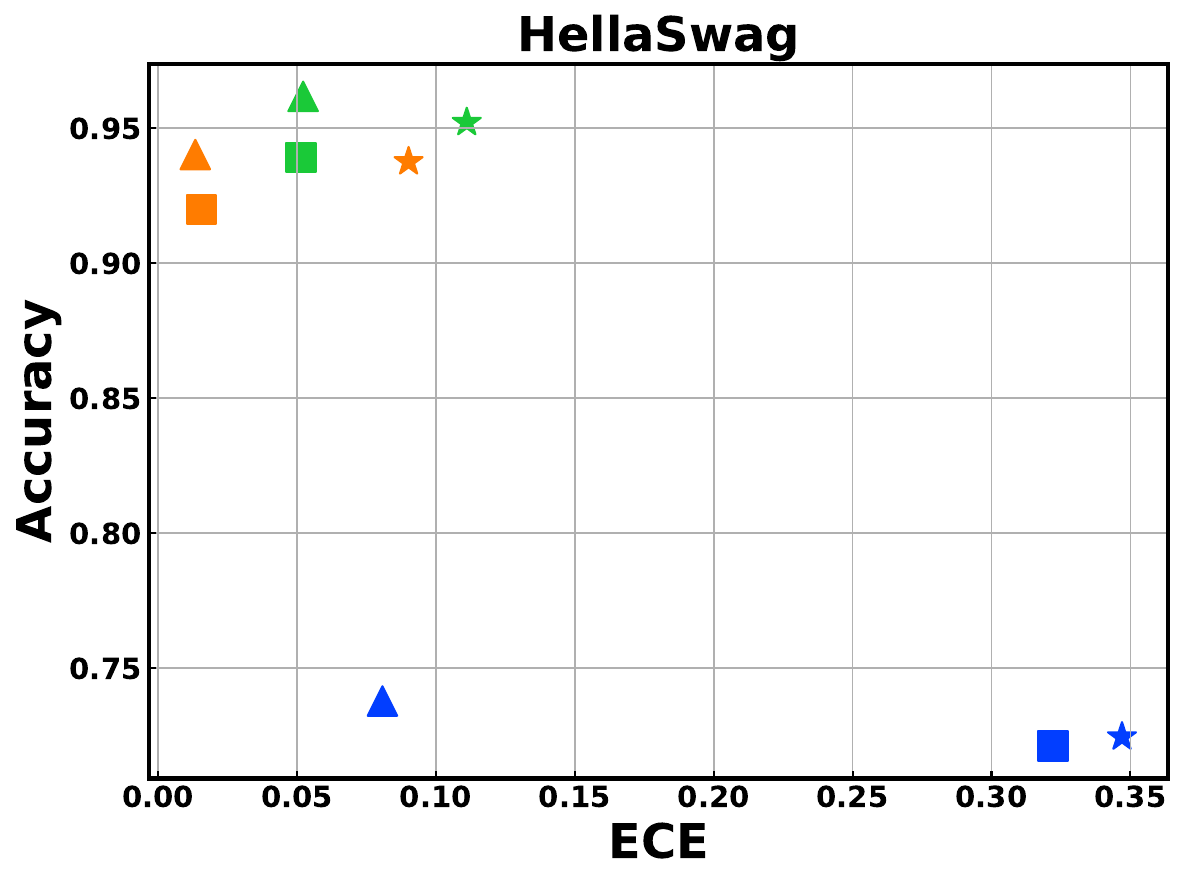} &
            \includegraphics[width=0.30\textwidth]{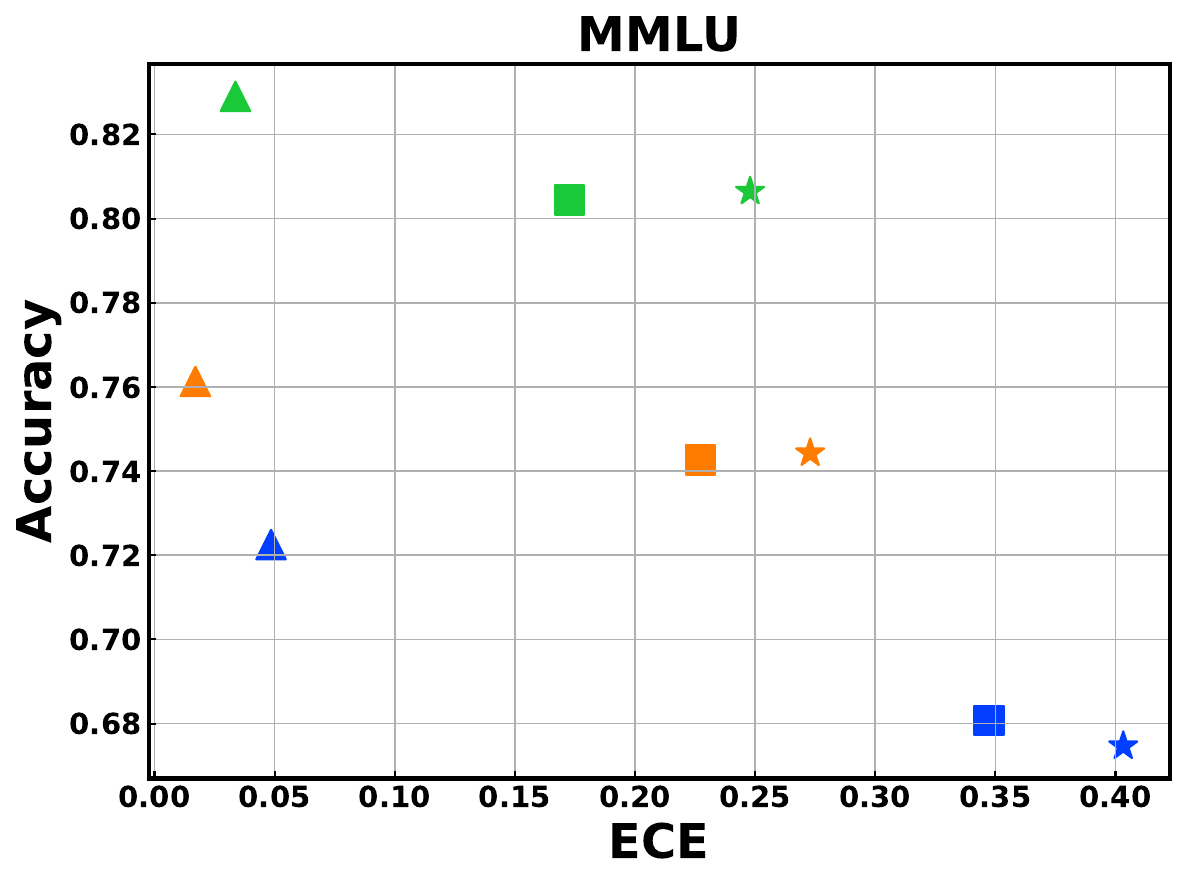} \\
        \end{tabular}
    }
    \caption{Accuracy versus Expected Calibration Error (ECE) comparison between ATCP, APS and LAC methods across different LLMs and five datasets i.e. CosmosQA, HaluDial, HaluSum, HellaSwag, MMLU. The ideal model should have high accuracy and low ECE, indicating accurate predictions with well calibrated uncertainty quantification (upper-left of the plot). The ECE of ATCP shows significant improvement compared to baseline methods.}
    \label{fig:annotated_grid_llm_ap}
\end{figure*}

\begin{figure*}[h!]
    \centering
    \makebox[\textwidth]{\includegraphics[width=0.6\textwidth]{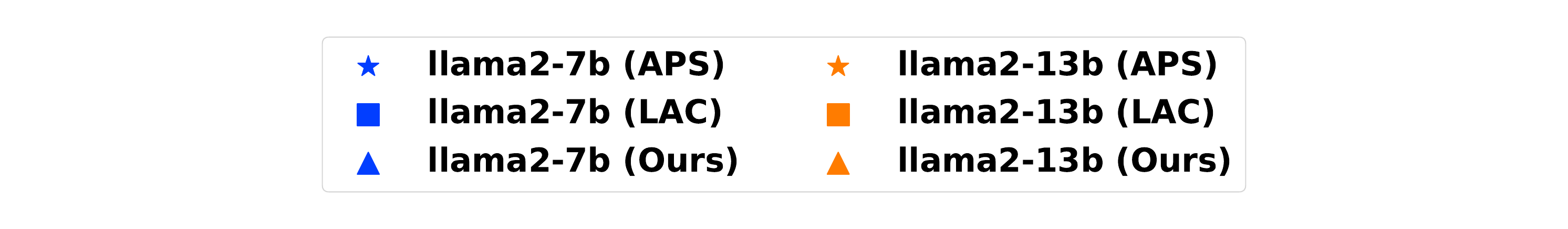}}
    \makebox[\textwidth]{
        \begin{tabular}{ccc}
            \includegraphics[width=0.30\textwidth]{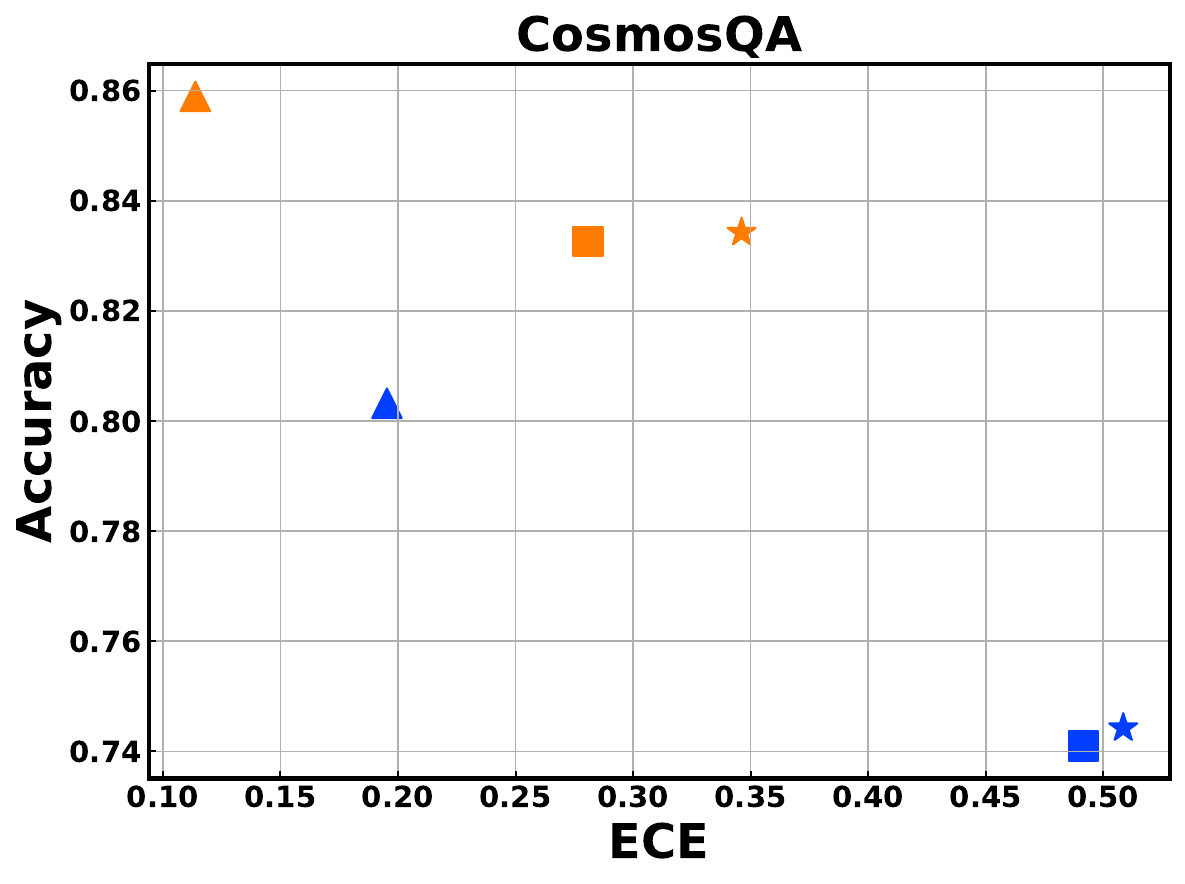} &
            \includegraphics[width=0.30\textwidth]{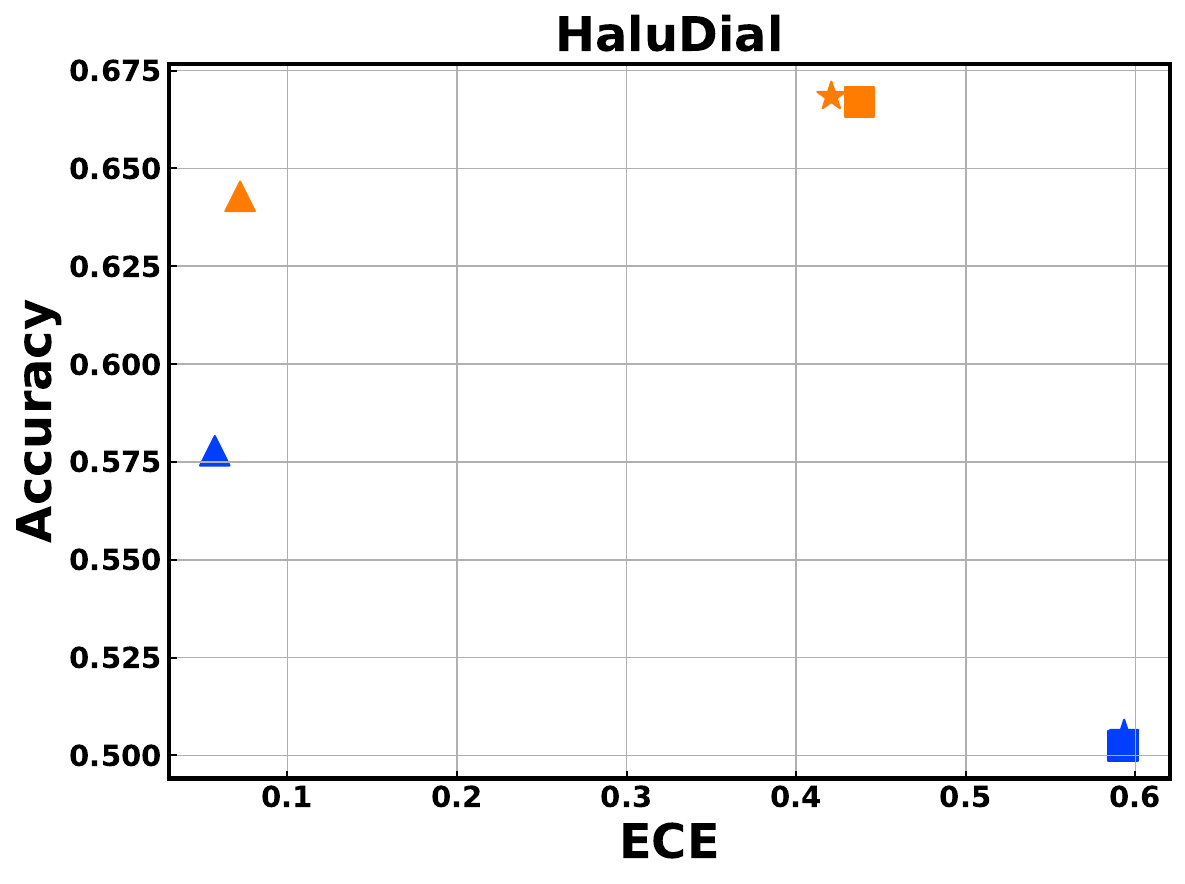} &
            \includegraphics[width=0.30\textwidth]{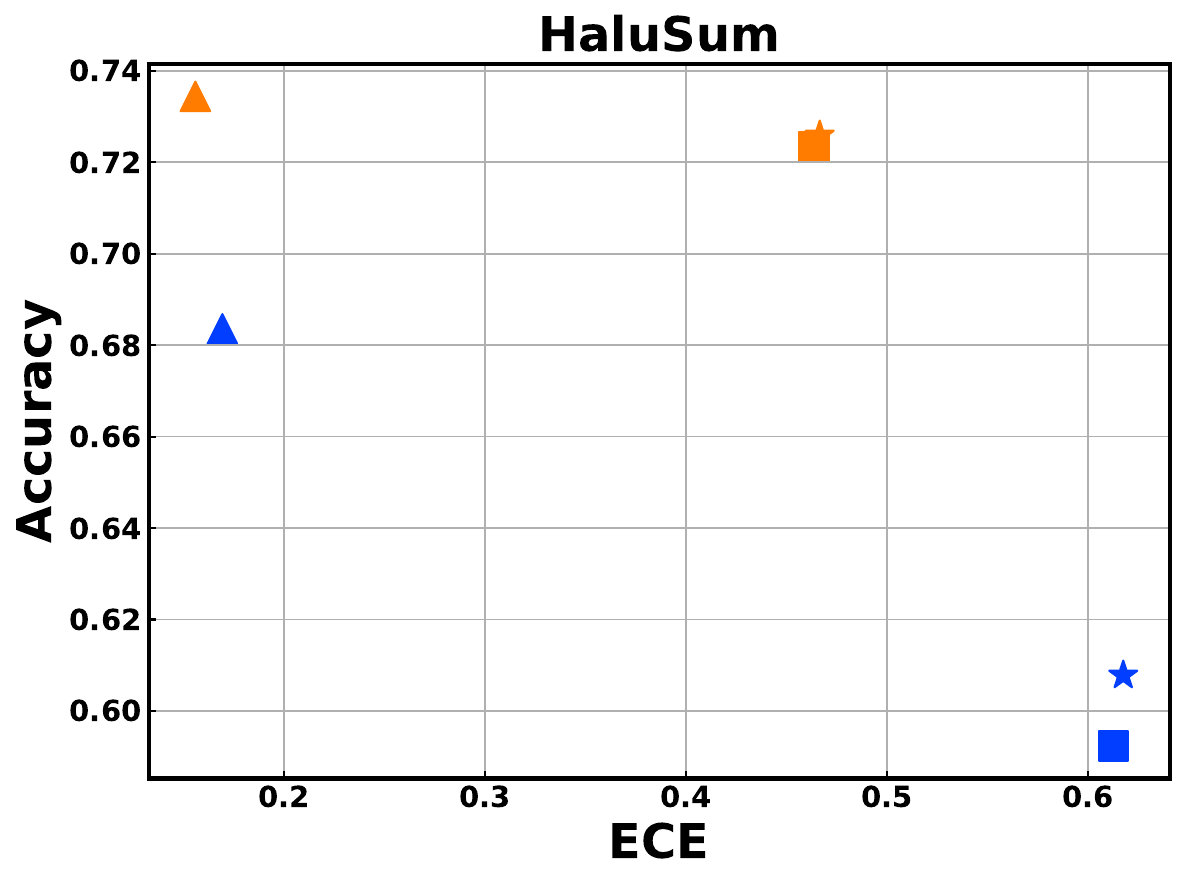} \\
        \end{tabular}
    }
    \makebox[\textwidth]{
        \begin{tabular}{cc}
            \includegraphics[width=0.30\textwidth]{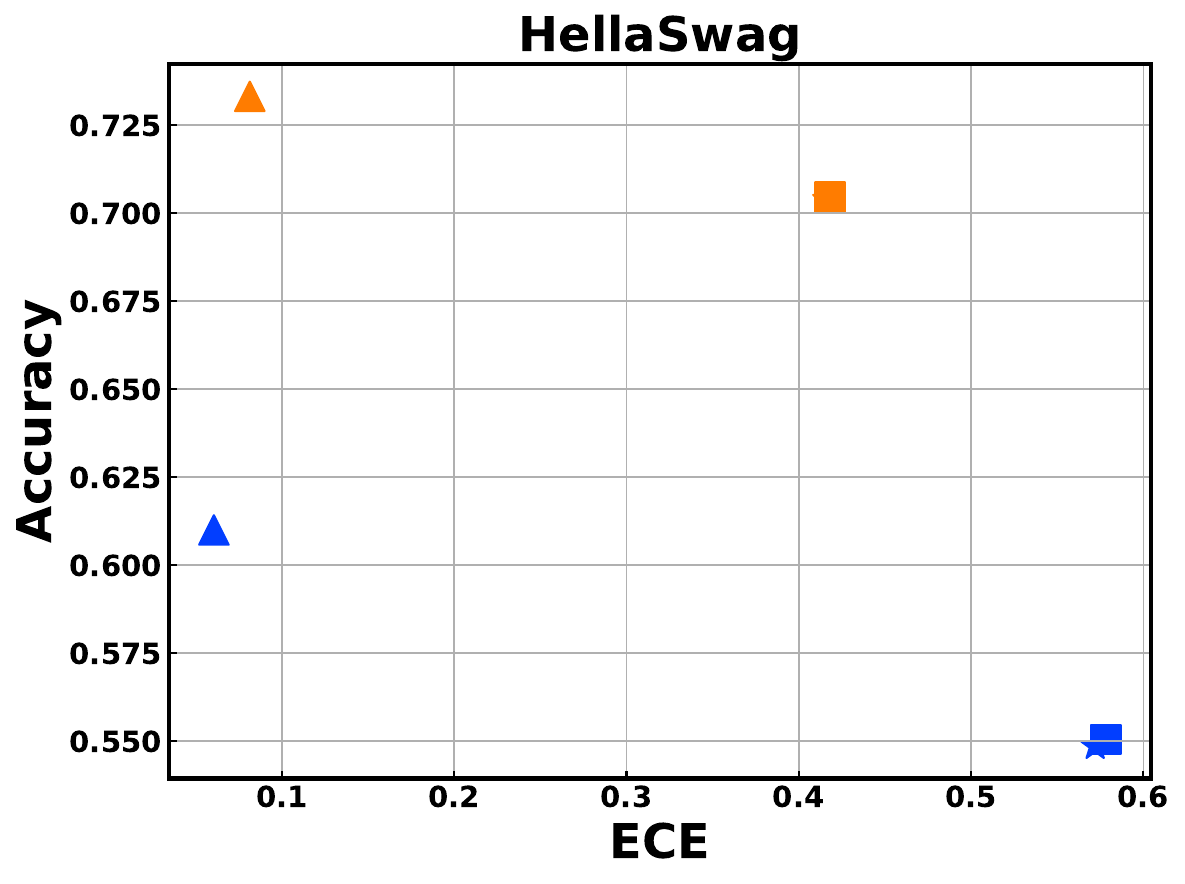} &
            \includegraphics[width=0.30\textwidth]{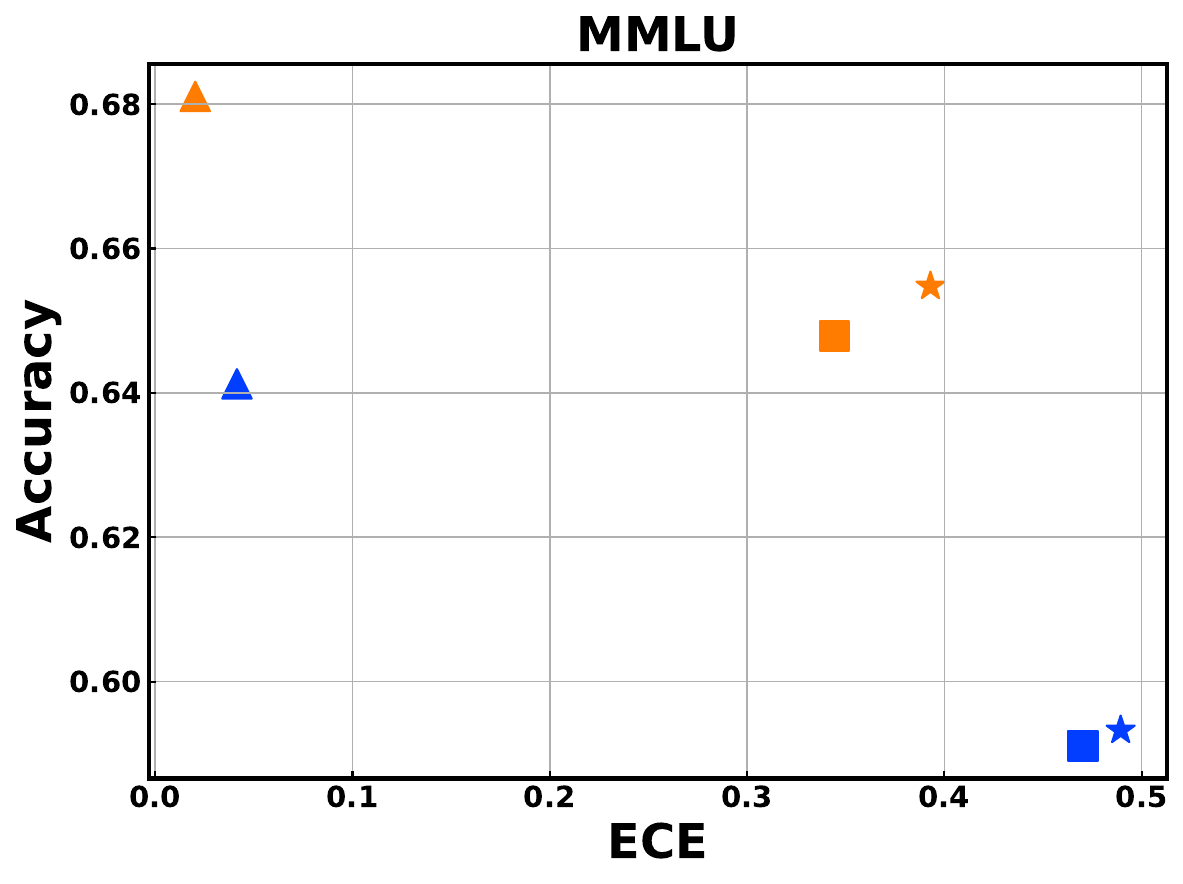} \\
        \end{tabular}
    }
    \caption{Accuracy versus Expected Calibration Error (ECE) comparison between ATCP, APS and LAC methods across different LLMs and five datasets i.e. CosmosQA, HaluDial, HaluSum, HellaSwag, MMLU. The ideal model should have high accuracy and low ECE, indicating accurate predictions with well calibrated uncertainty quantification (upper-left of the plot). The ECE of ATCP shows significant improvement compared to baseline methods.}
    \label{fig:acc_ece_ap_llm}
\end{figure*}

\subsubsection{Effect of Model Scale}
To examine the impact of model scale, we analyzed the performance of our CAP method across models of varying sizes. As shown in \autoref{fig:spider}, larger models generally achieve higher accuracy, with the most significant gains observed when scaling from 13B to 34B parameters. Prediction set size inversely correlates with model scale, as larger models produce smaller sets, reflecting greater precision and reduced uncertainty. Additionally, AUROC and AUARC improve consistently with increasing model scale, indicating that larger models are not only more accurate but also less prone to hallucinations and better at abstaining when uncertainty is high.

To examine the impact of model scale, we analyzed the performance of our CAP method across models of varying sizes. larger models generally achieve higher accuracy and produce smaller set sizes while showing better performance in avoiding hallucinations and uncertainty guided selective generation. As shown in \autoref{fig:spider}, \autoref{fig:spider_llm}, and \autoref{fig:spider_ap_llm}, the most significant gains observed when scaling the model size from 13B to 34B parameters. Prediction set size inversely correlates with model scale, as larger models produce smaller sets, reflecting greater precision and reduced uncertainty. Additionally in \autoref{fig:spider_ap_vlm}, we can see slight gains in all metrics comparing VLMs with 7B parameters against Yi-VL with 6B parameters. However, since size differences in VLMs in this particular benchmark are not very different, part of the gap in results between model would be associated with different finetuning methods used for these models and the pre-trained model under the hood.

\begin{figure*}[h!]
    \centering
    \makebox[\textwidth]{\includegraphics[width=0.7\textwidth]{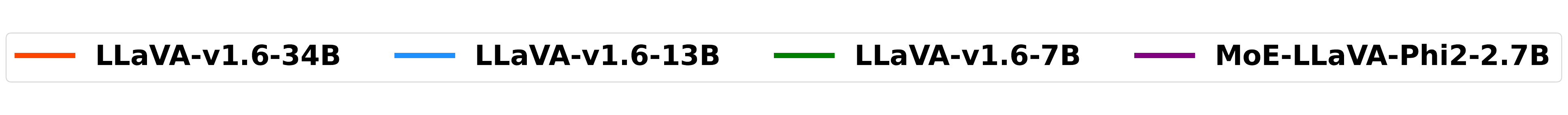}}
    \makebox[\textwidth]{
        \begin{tabular}{cccc}
            \includegraphics[width=0.24\textwidth]{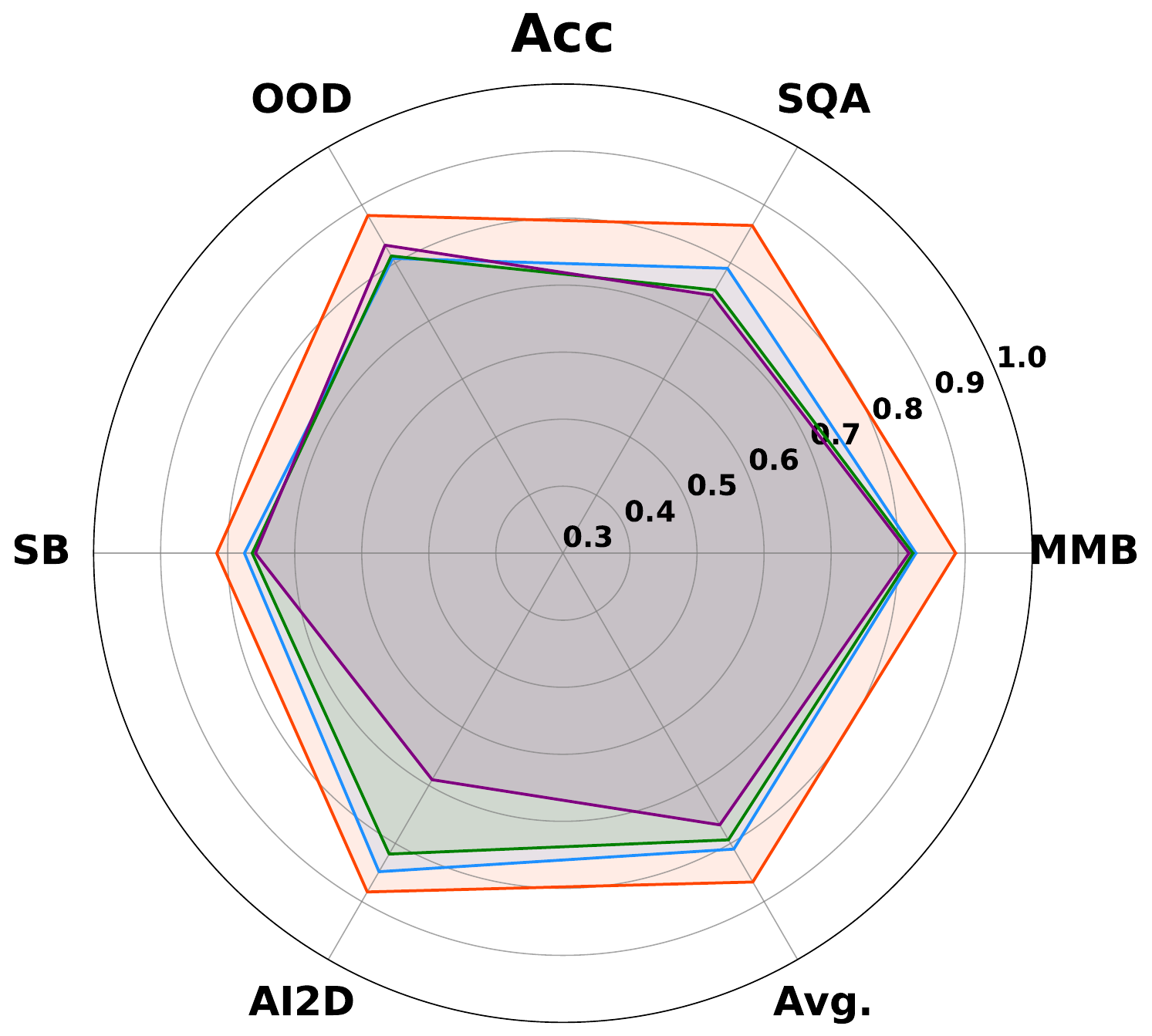} &
            \includegraphics[width=0.24\textwidth]{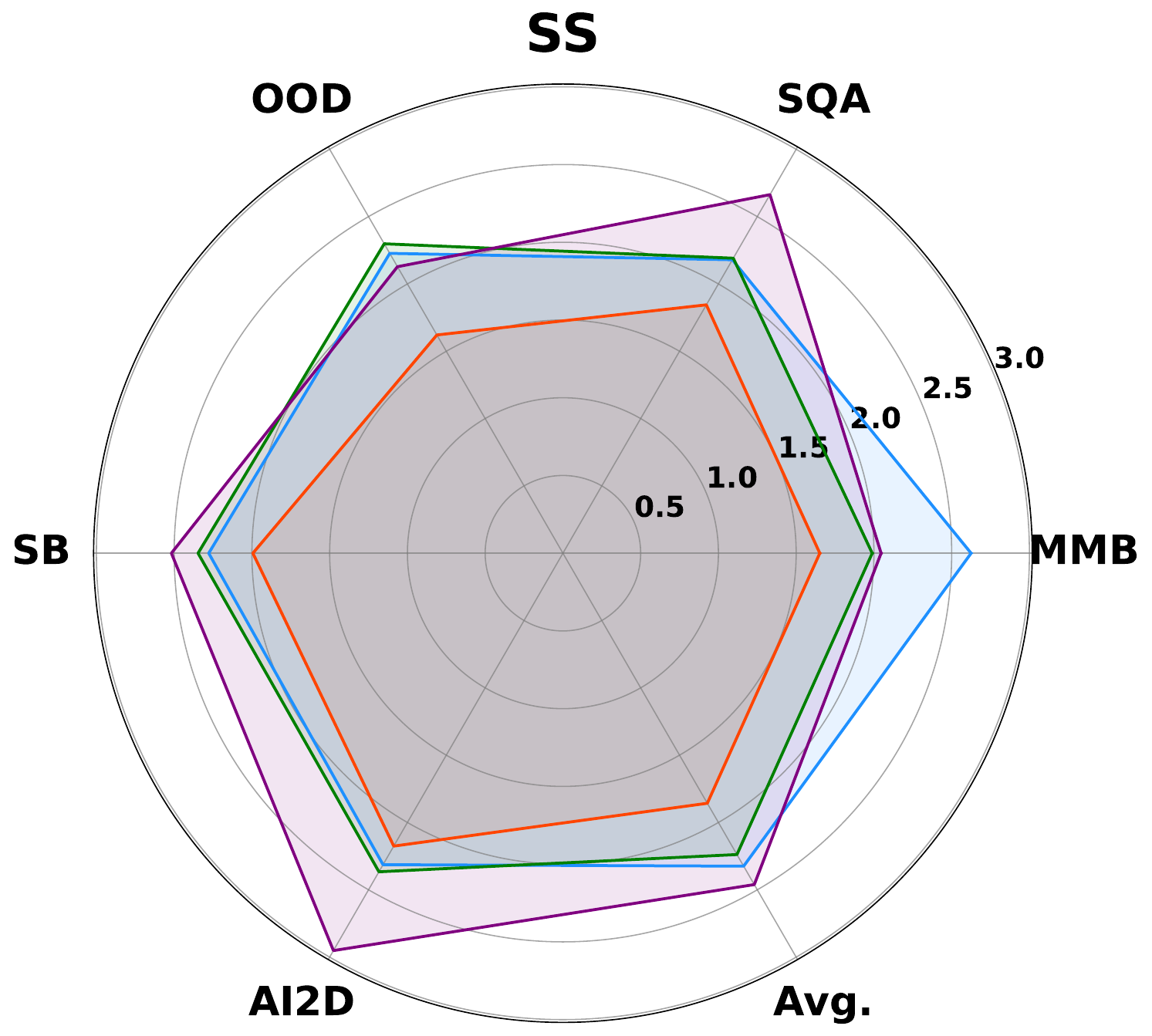} &
            \includegraphics[width=0.24\textwidth]{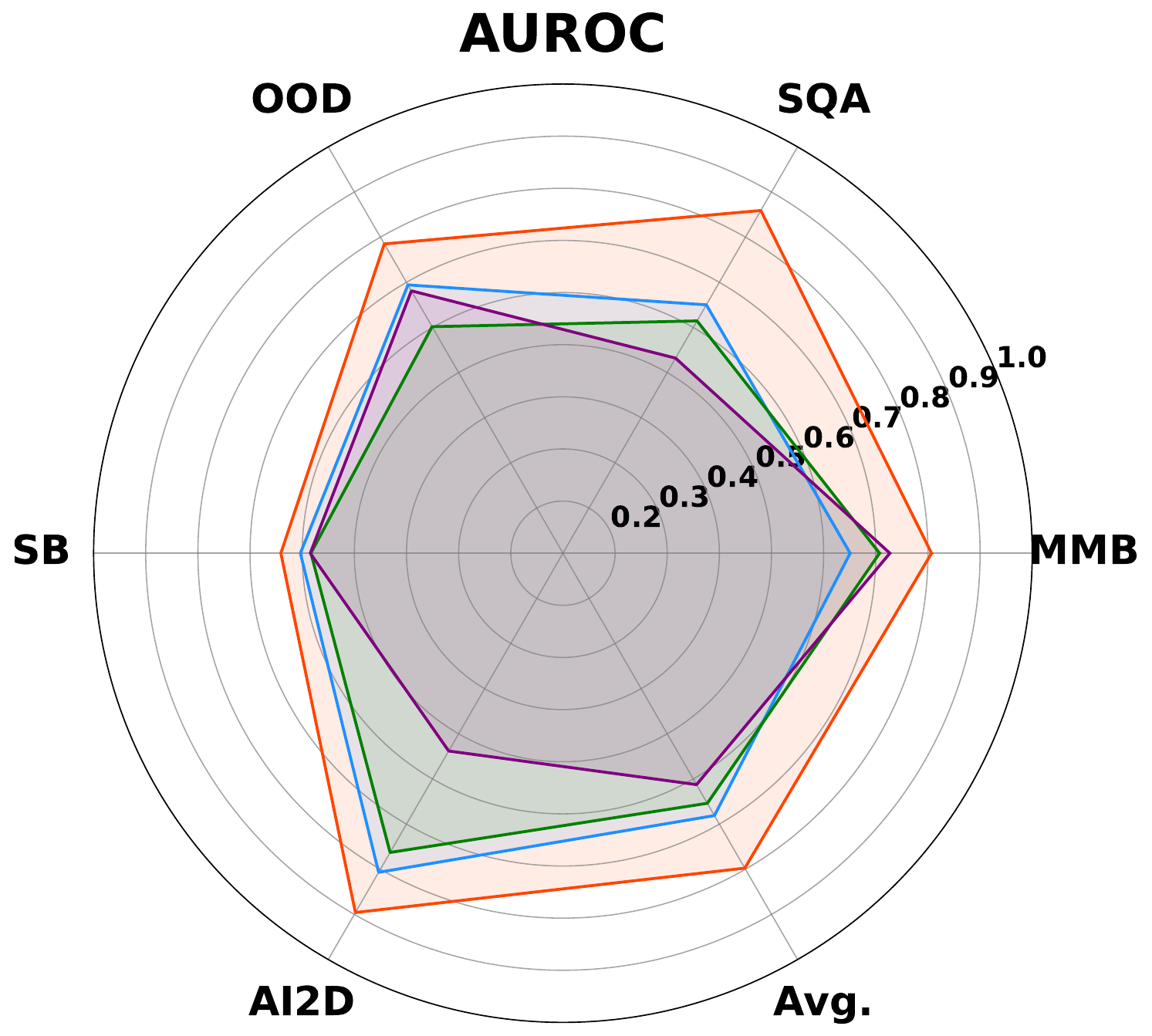} &
            \includegraphics[width=0.24\textwidth]{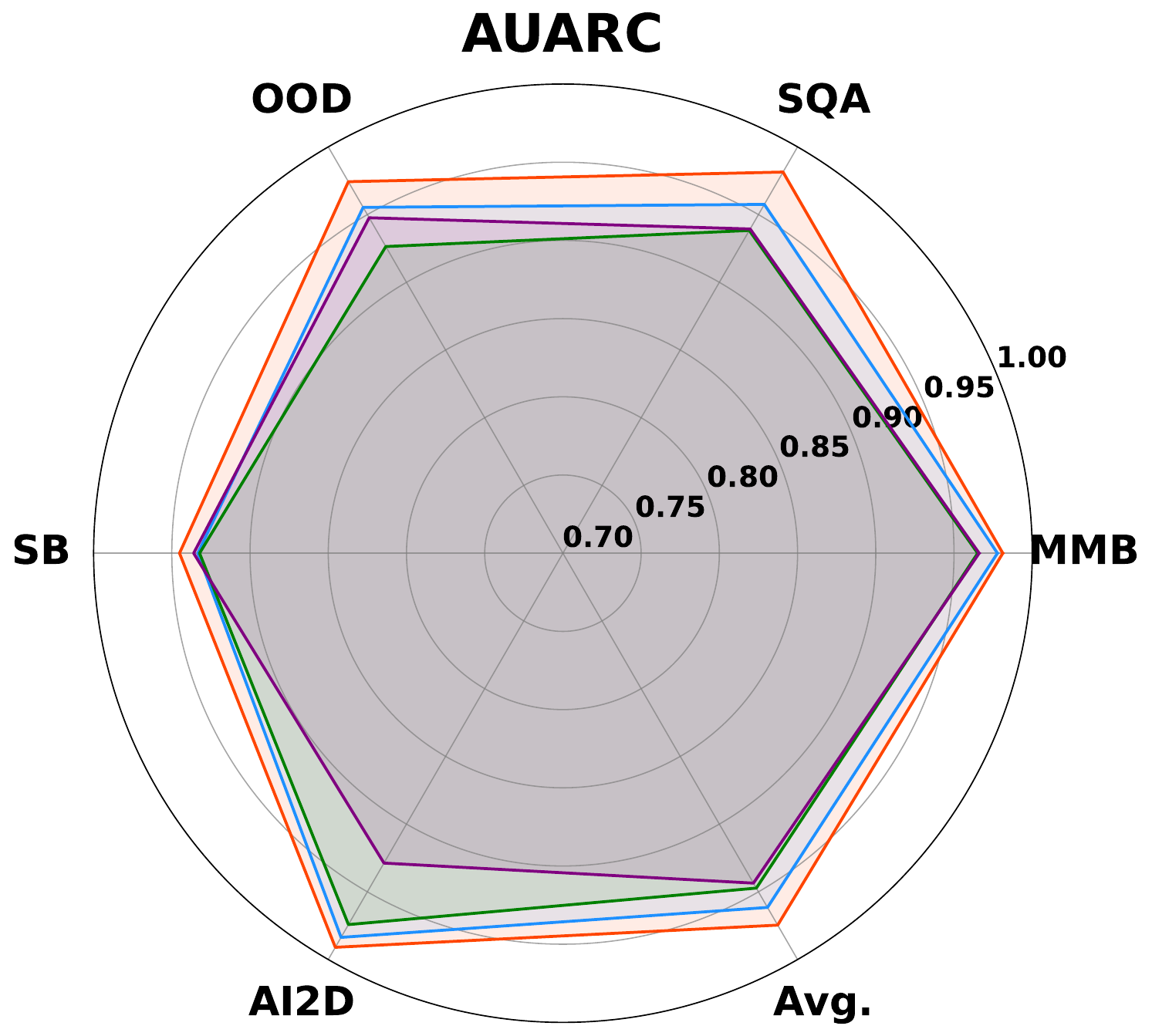} \\
        \end{tabular}
    }
    \caption{Performance comparison of VLMs with different model sizes (2.7B to 34B) across various metrics. Figures from left to right represents the performance of four models on one of the four metrics i) accuracy, ii) set size, iii) AUROC, and iv) AUARC respectively. In each figure, we have drawn the performance of models across five datasets in VLM benchmark. Each figure represents the effect of model scale (number of parameters) in its performance across different uncertainty metrics.}
    \label{fig:spider}
\end{figure*}

\begin{figure*}[h!]
    \centering
    \makebox[\textwidth]{\includegraphics[width=0.8\textwidth]{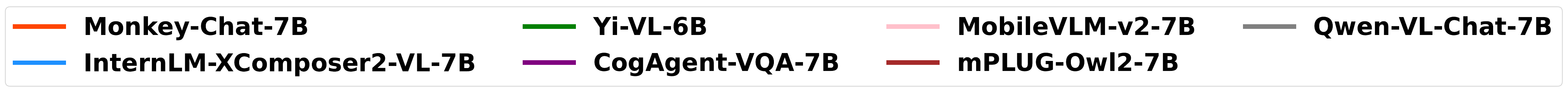}}
    \makebox[\textwidth]{
        \begin{tabular}{cccc}
            \includegraphics[width=0.24\textwidth]{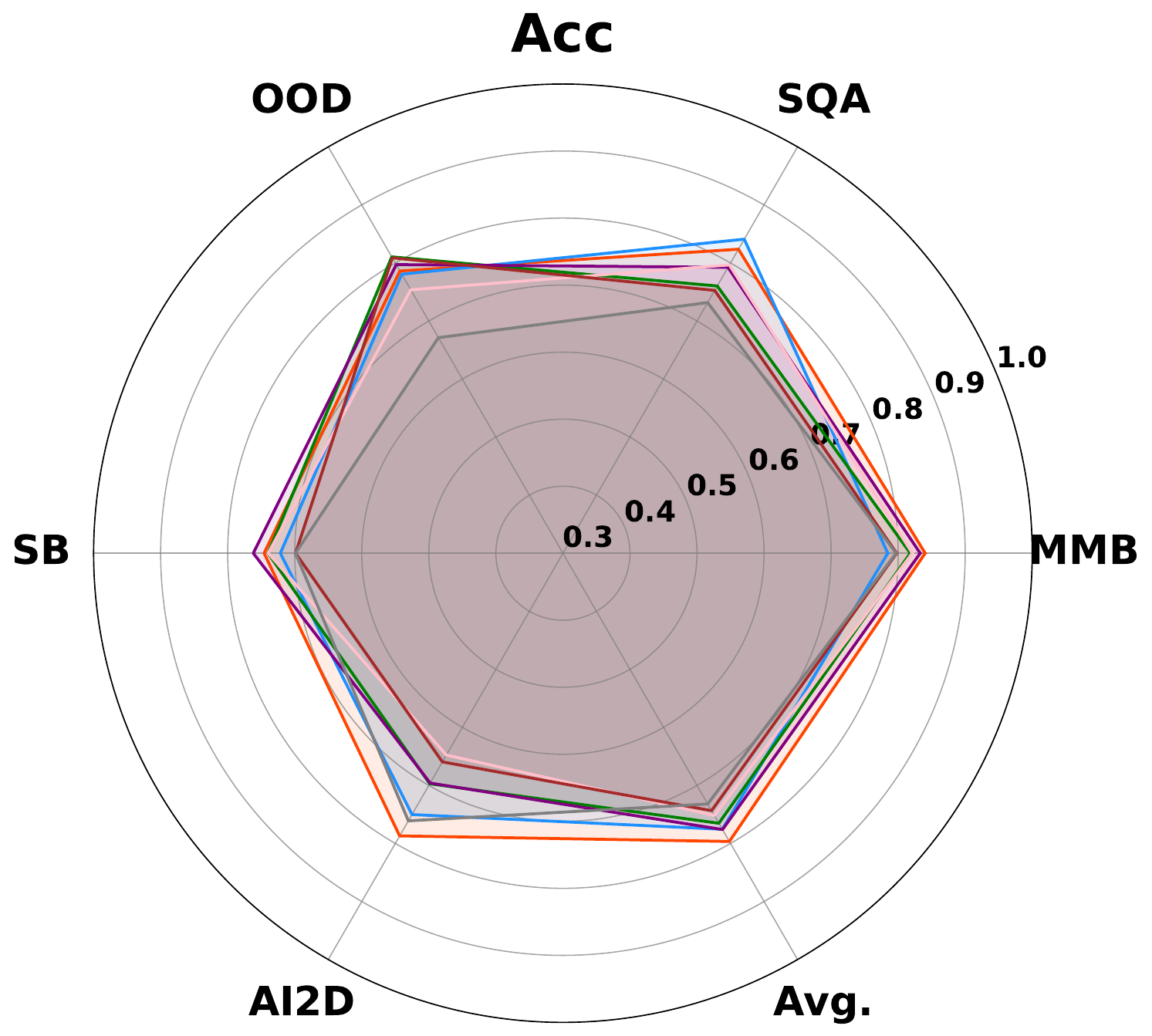} &
            \includegraphics[width=0.24\textwidth]{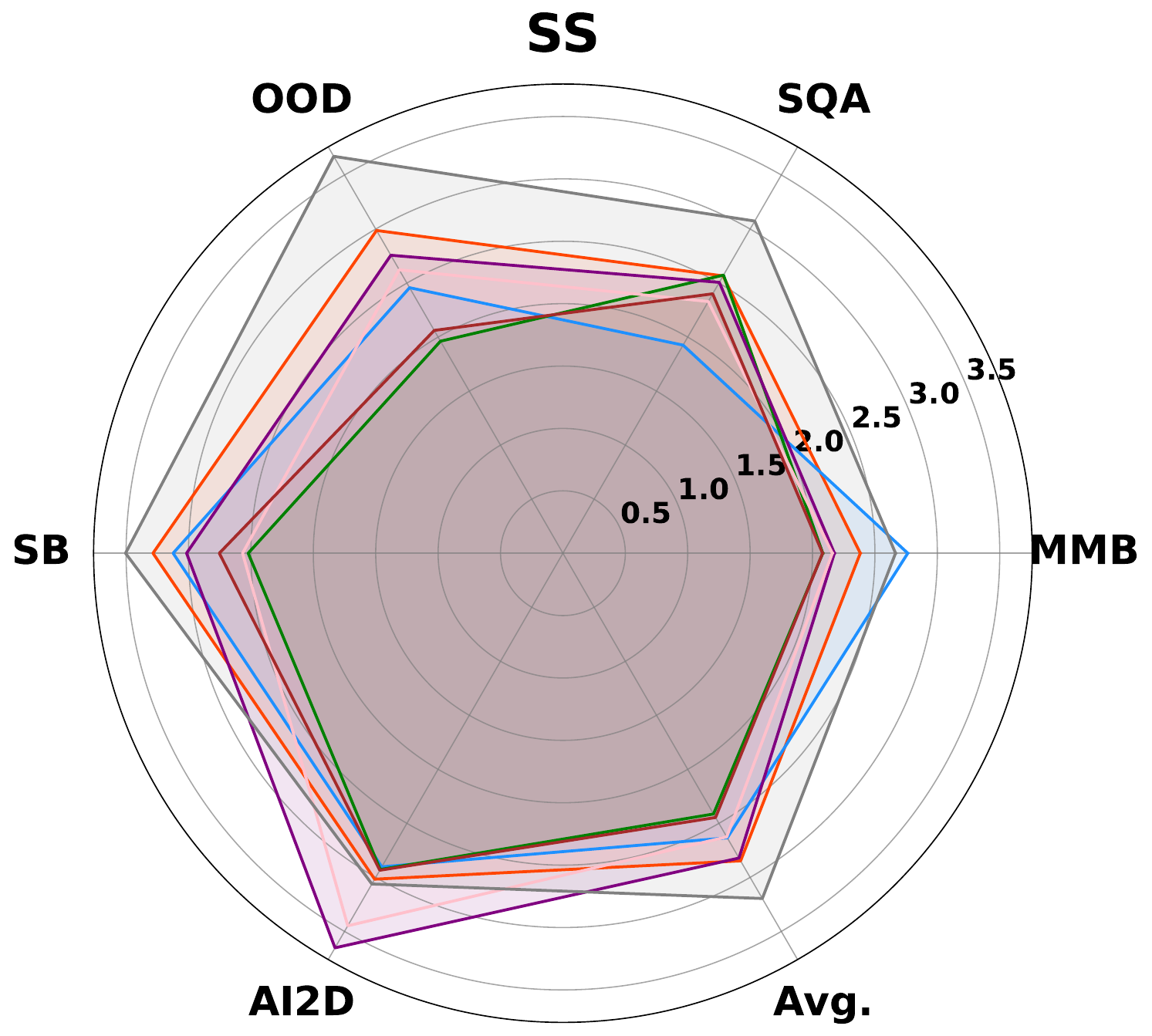} &
            \includegraphics[width=0.24\textwidth]{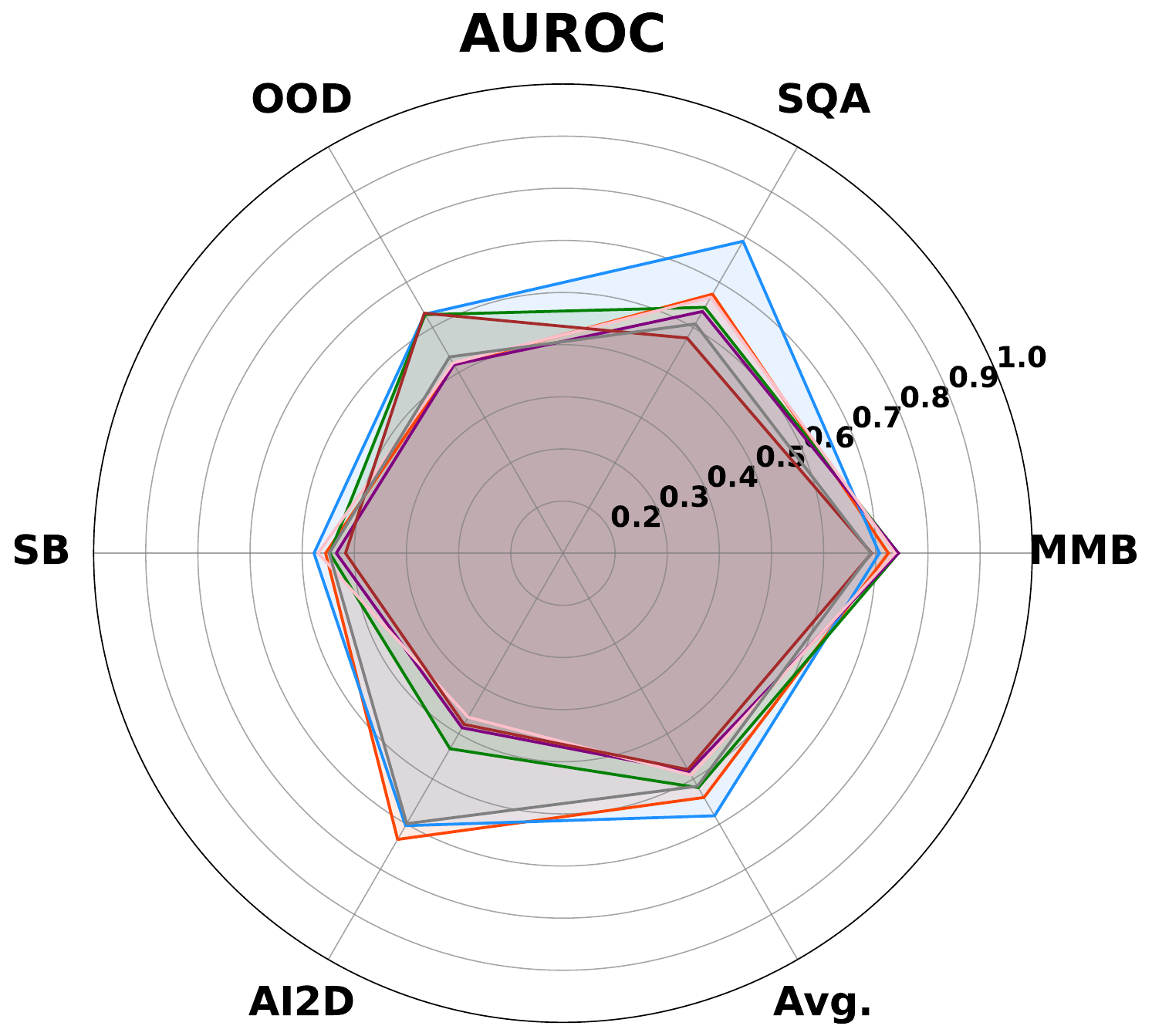} &
            \includegraphics[width=0.24\textwidth]{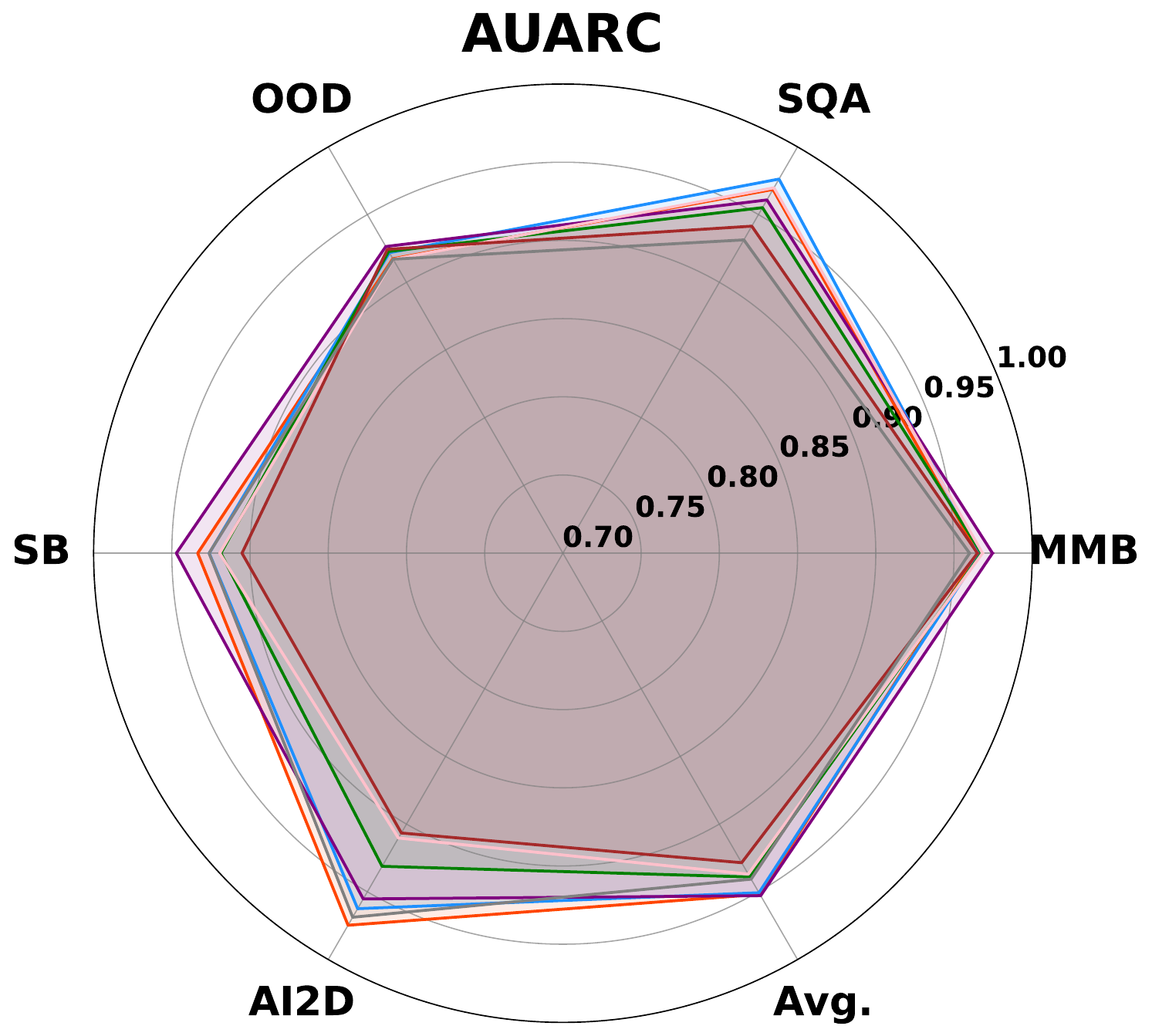} \\
        \end{tabular}
    }
    \caption{Performance comparison of additional VLMs with different model sizes (6B to 7B) across various metrics. Figures from left to right represents the performance of four models on one of the four metrics i) accuracy, ii) set size, iii) AUROC, and iv) AUARC respectively. In each figure, we have drawn the performance of models across five datasets in VLM benchmark. Each figure represents the effect of model scale (number of parameters) in its performance across different uncertainty metrics.}
    \label{fig:spider_ap_vlm}
\end{figure*}

\begin{figure*}[h!]
    \centering
    \makebox[\textwidth]{\includegraphics[width=0.4\textwidth]{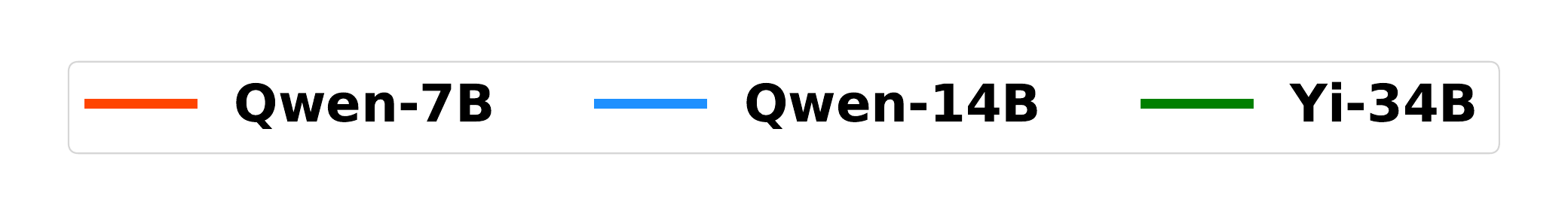}}
    \makebox[\textwidth]{
        \begin{tabular}{cccc}
            \includegraphics[width=0.24\textwidth]{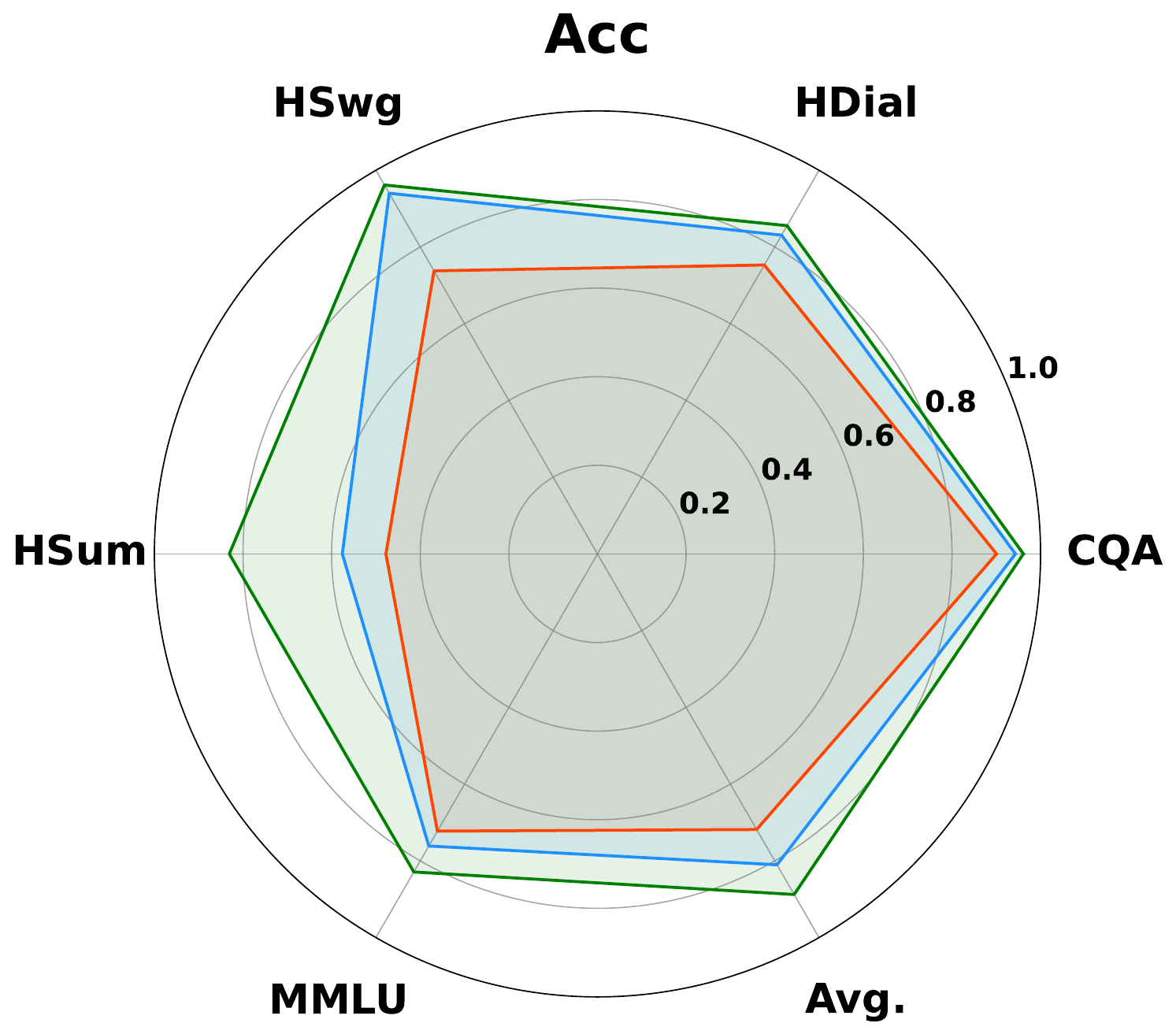} &
            \includegraphics[width=0.24\textwidth]{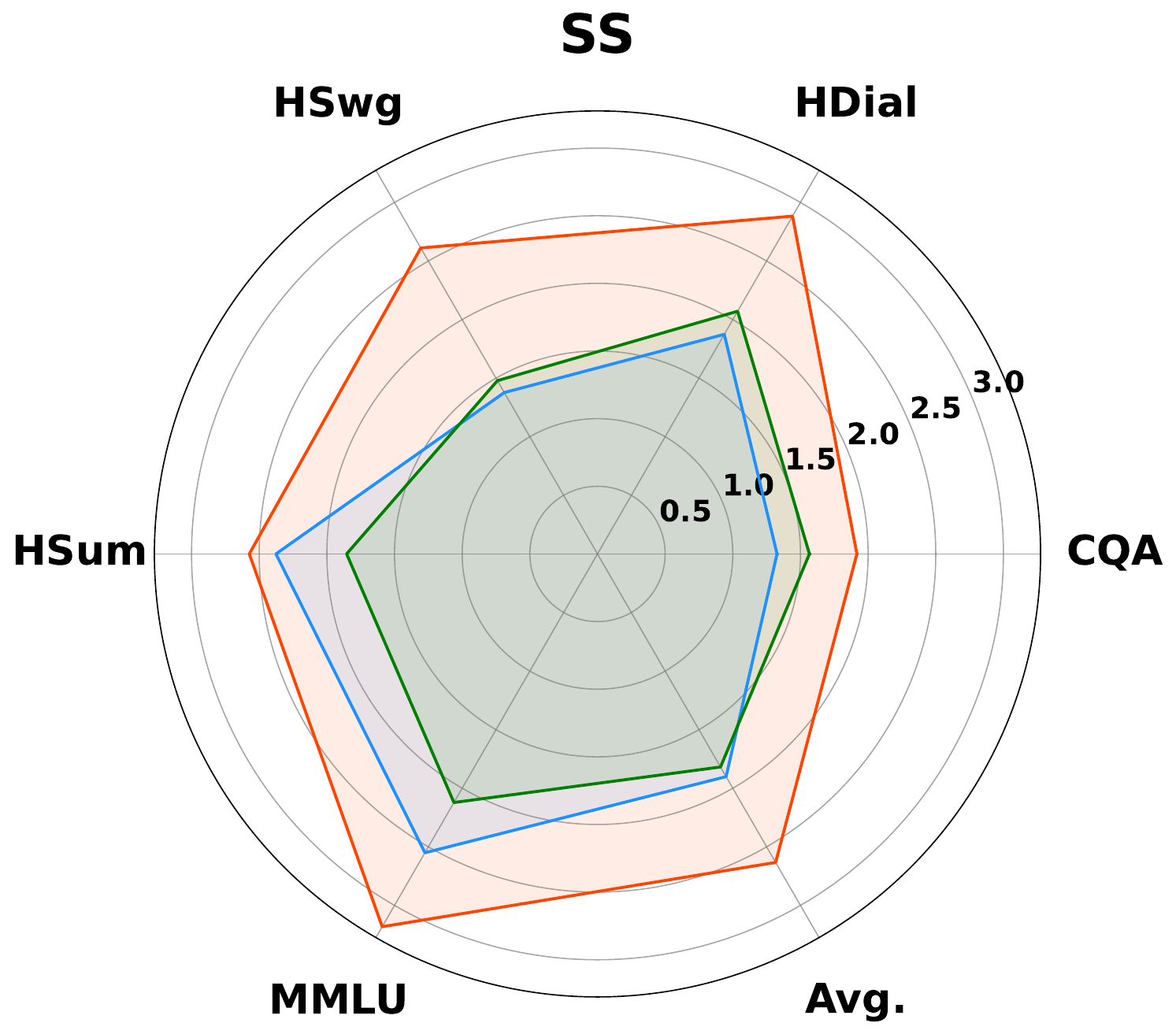} &
            \includegraphics[width=0.24\textwidth]{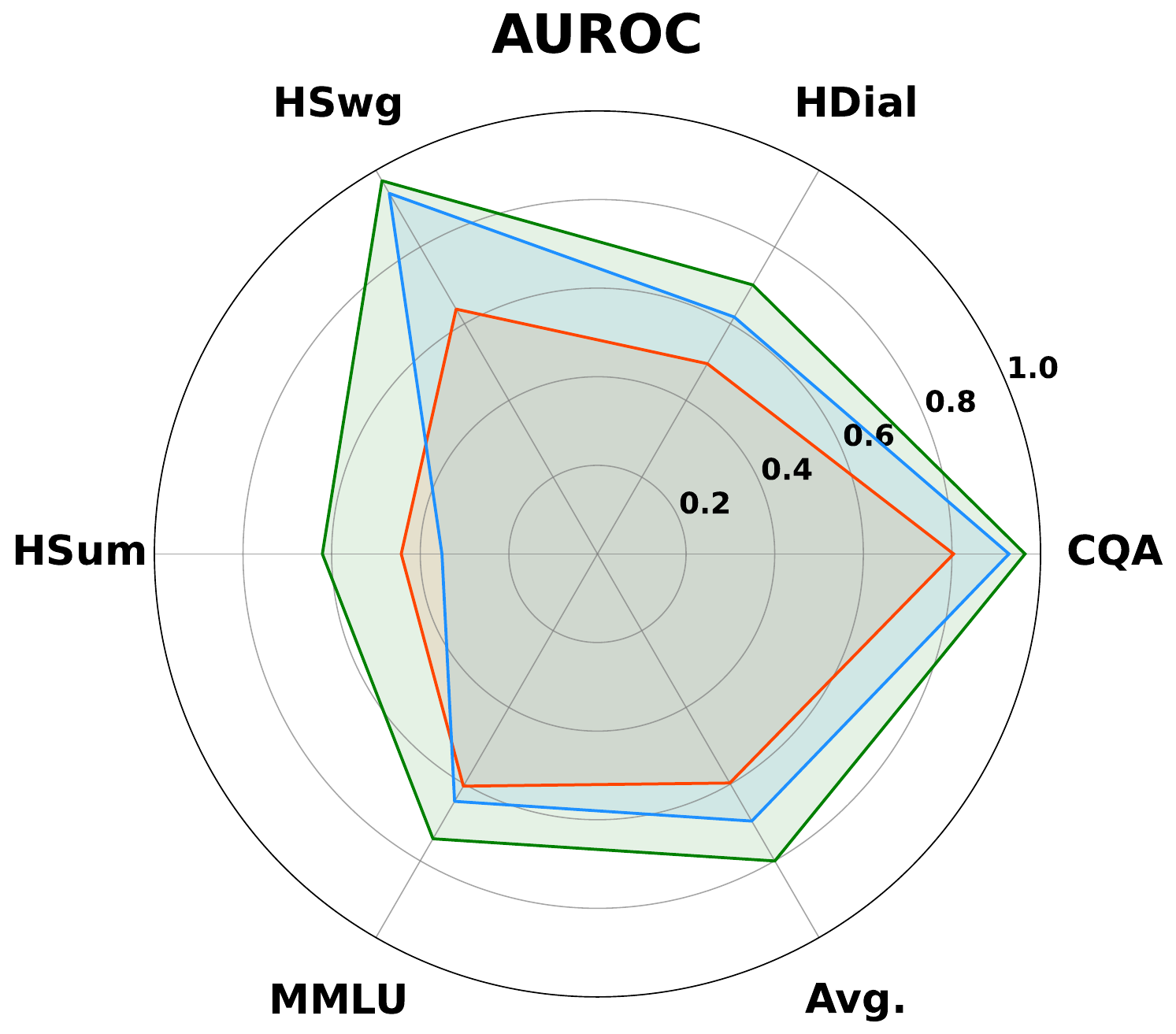} &
            \includegraphics[width=0.24\textwidth]{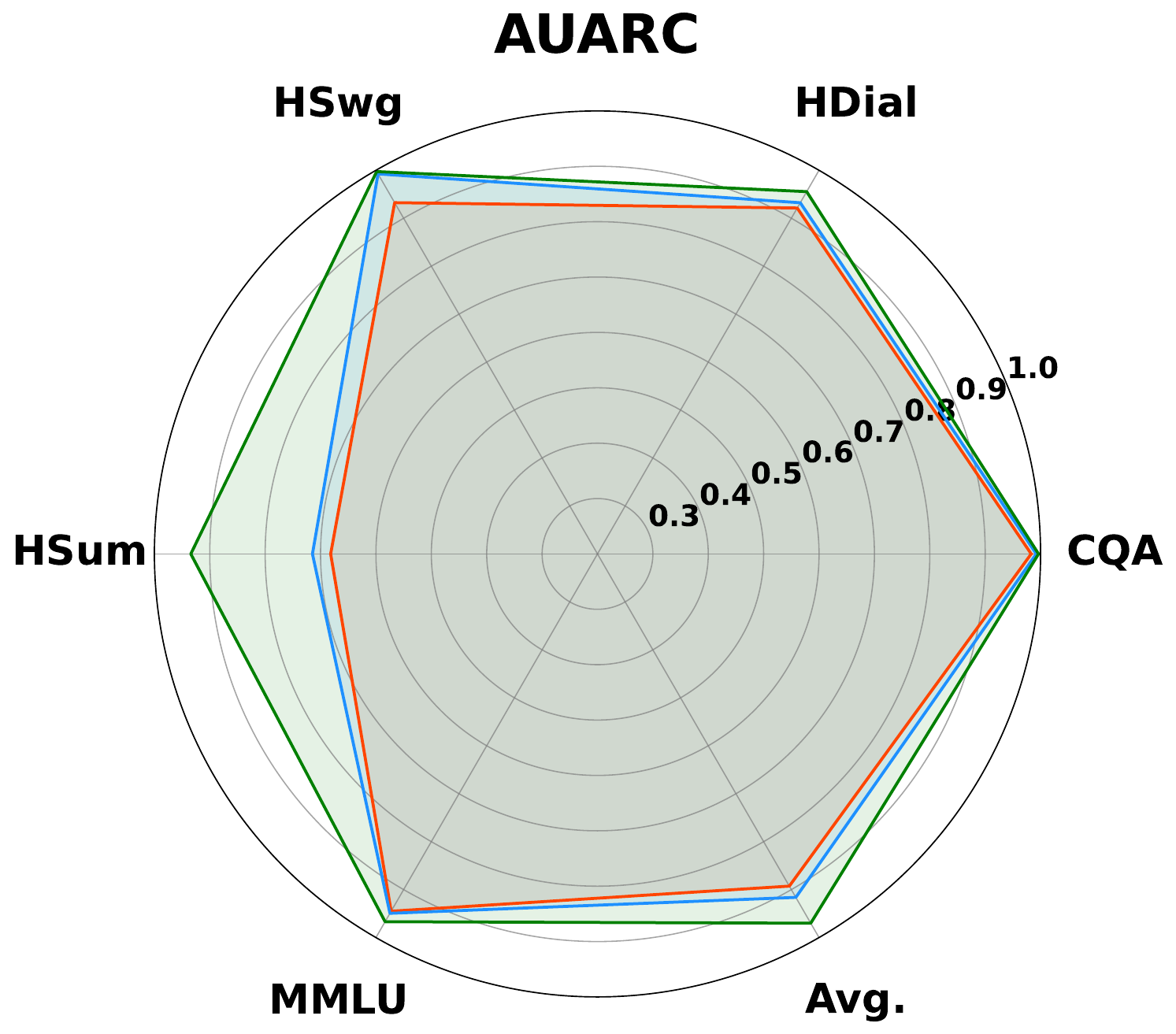} \\
        \end{tabular}
    }
    \caption{Performance comparison of LLMs with different model sizes (7B to 34B) across various metrics. Figures from left to right represents the performance of four models on one of the four metrics i) accuracy, ii) set size, iii) AUROC, and iv) AUARC respectively. In each figure, we have drawn the performance of models across five datasets in LLM benchmark. Each figure represents the effect of model scale (number of parameters) in its performance across different uncertainty metrics.}
    \label{fig:spider_llm}
\end{figure*}

\begin{figure*}[h!]
    \centering
    \makebox[\textwidth]{\includegraphics[width=0.4\textwidth]{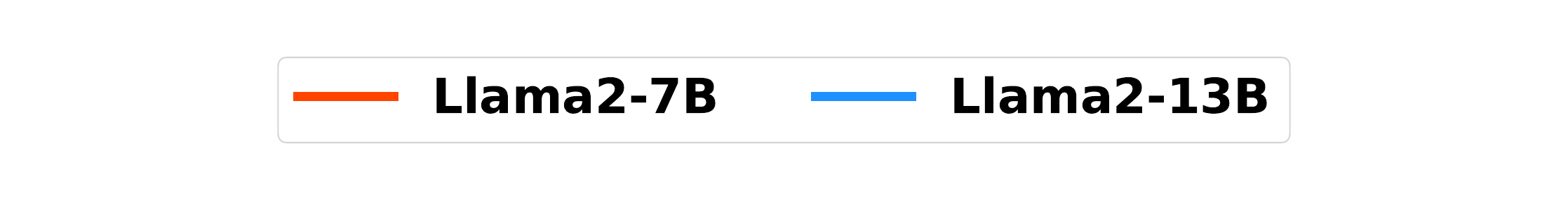}}
    \makebox[\textwidth]{
        \begin{tabular}{cccc}
            \includegraphics[width=0.24\textwidth]{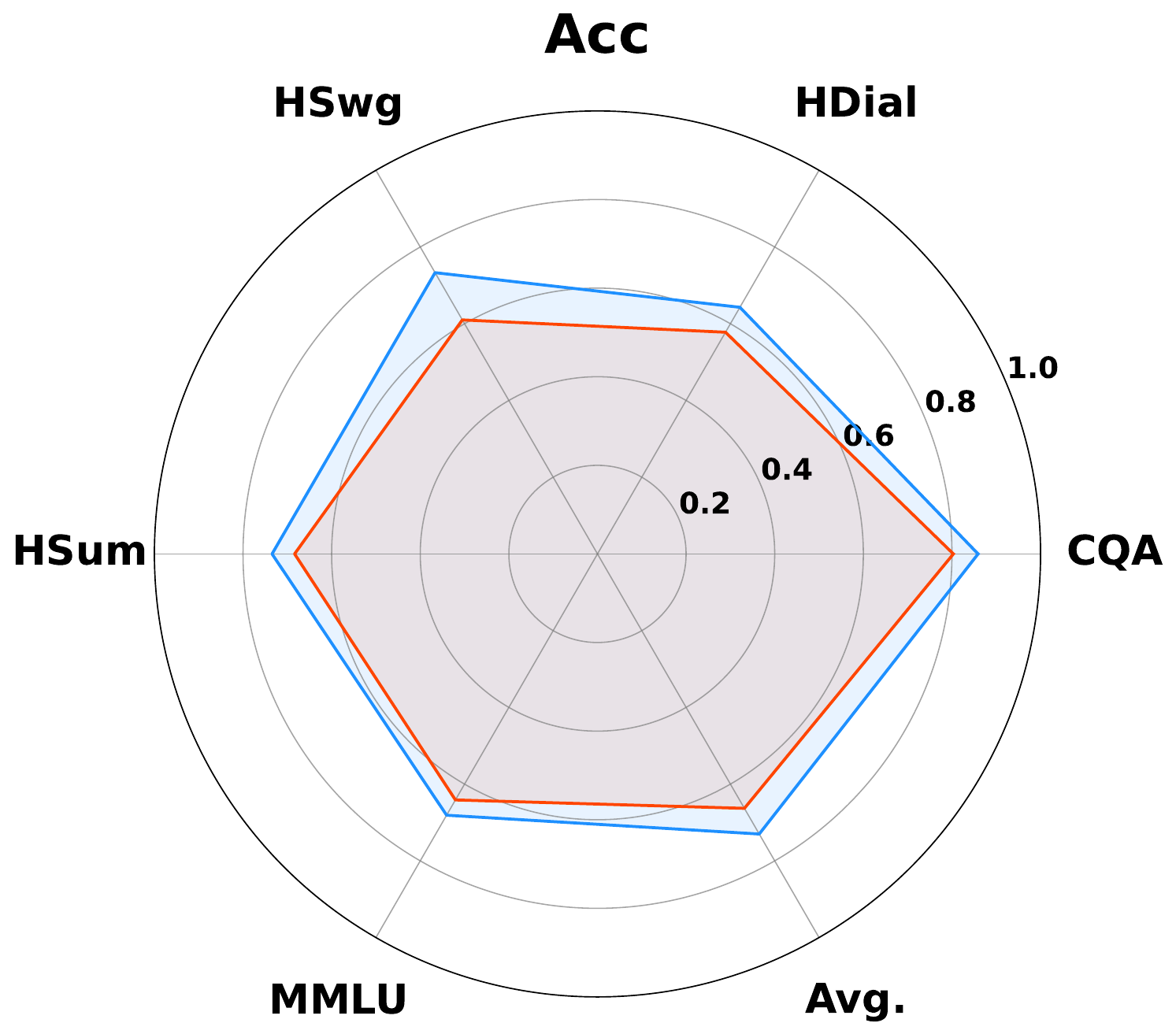} &
            \includegraphics[width=0.24\textwidth]{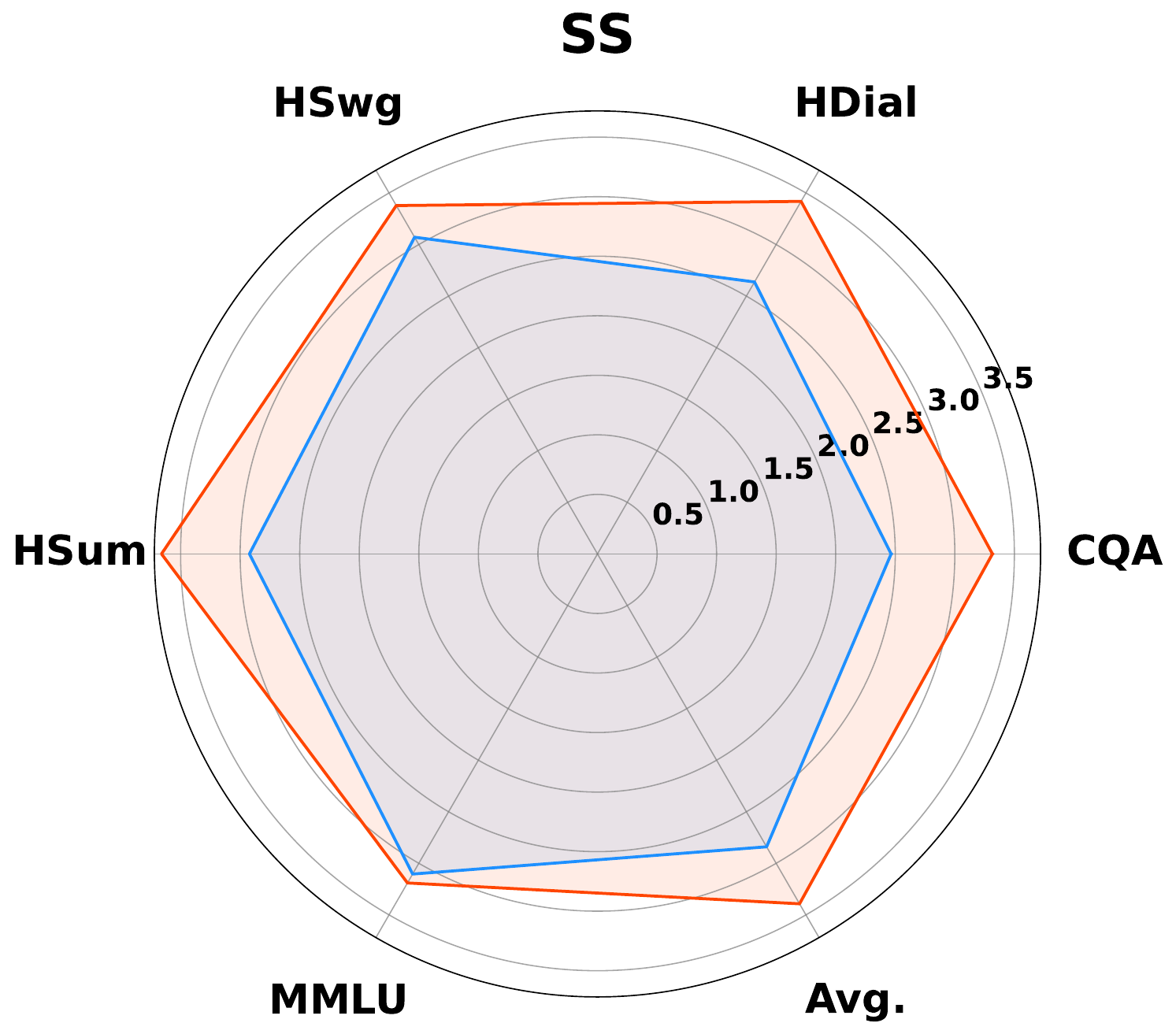} &
            \includegraphics[width=0.24\textwidth]{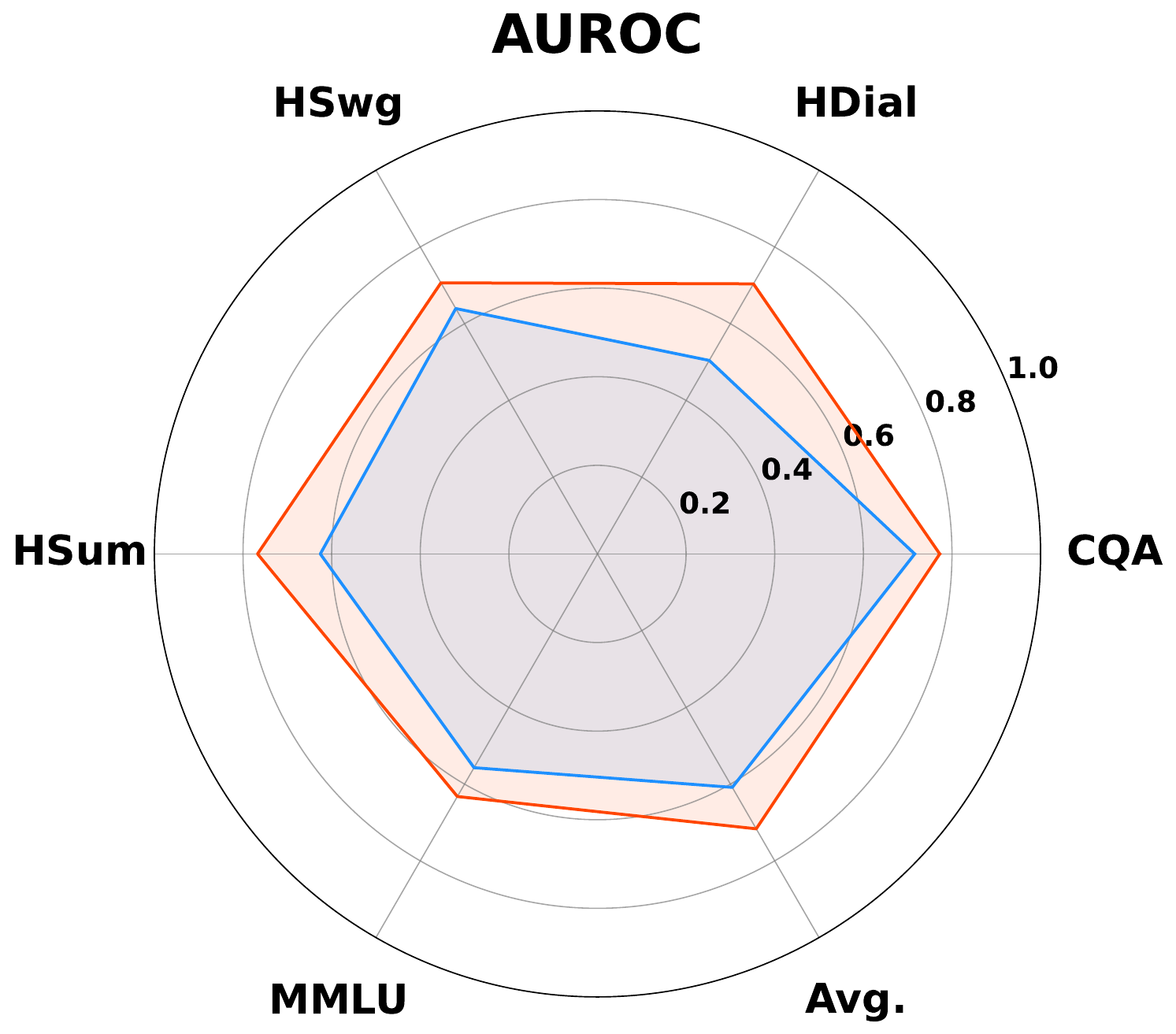} &
            \includegraphics[width=0.24\textwidth]{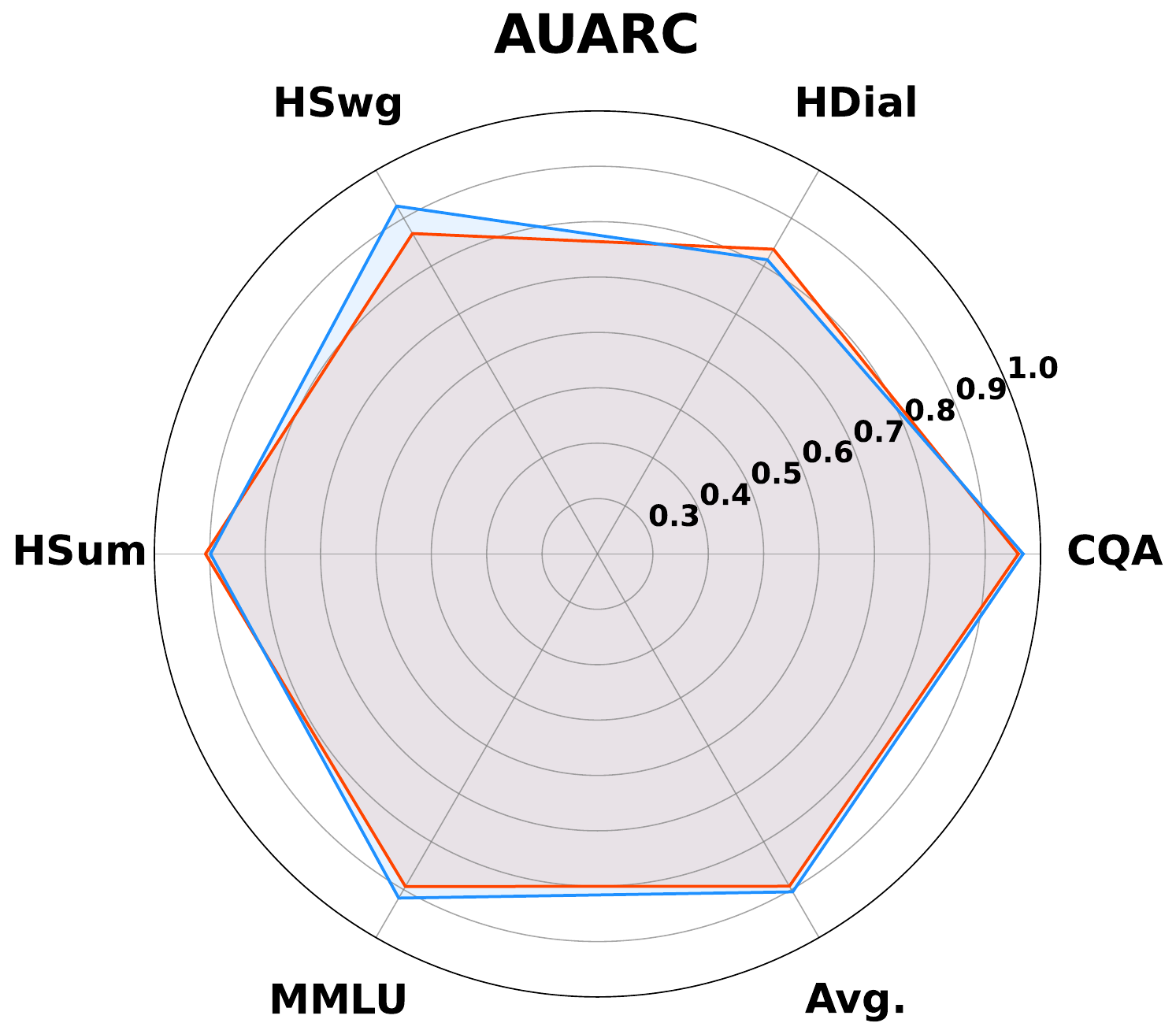} \\
        \end{tabular}
    }
    \caption{Performance comparison of Llama-2 series LLMs with different model sizes (7B and 13B) across various metrics. Figures from left to right represents the performance of two models on four metrics i) accuracy, ii) set size, iii) AUROC, and iv) AUARC respectively. In each figure, we show the performance of models across five datasets in LLM benchmark. Each figure represents the effect of model scale (number of parameters) in its performance across different uncertainty metrics.}
    \label{fig:spider_ap_llm}
\end{figure*}

\begin{figure*}[h!]
    \centering
    \makebox[\textwidth]{\includegraphics[width=0.4\textwidth]{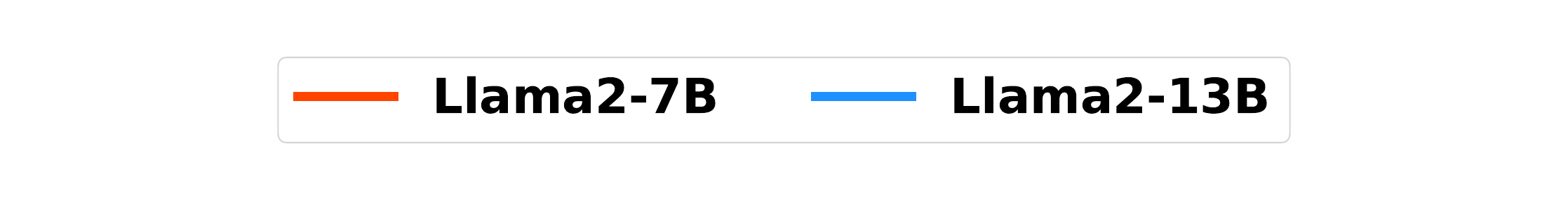}}
    \makebox[\textwidth]{
        \begin{tabular}{cccc}
            \includegraphics[width=0.24\textwidth]{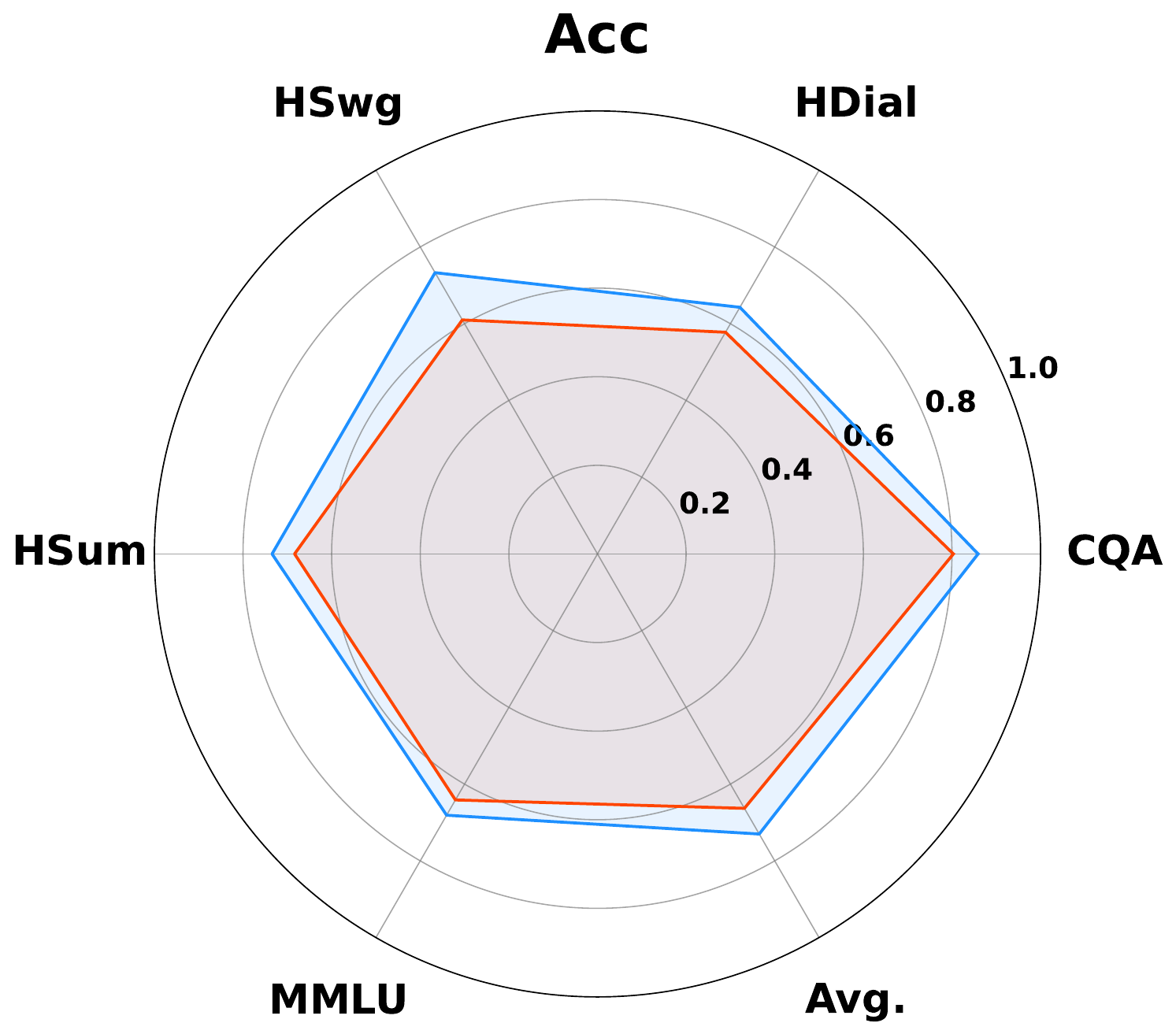} &
            \includegraphics[width=0.24\textwidth]{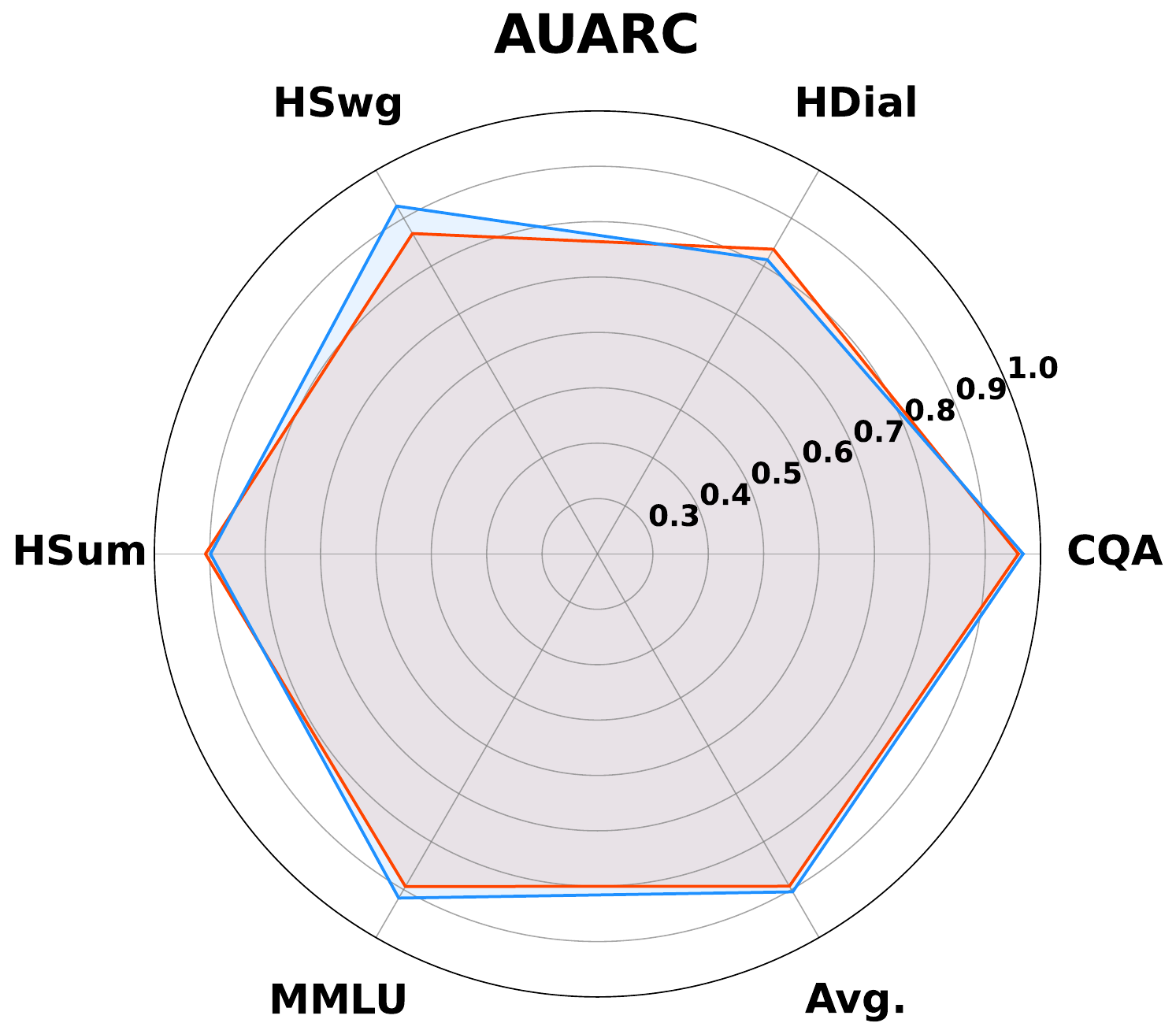} \\
        \end{tabular}
    }
    \caption{Performance comparison of Llama-2 series LLMs with different model sizes (7B and 13B) across various metrics. Figures from left to right represents the performance of two models on four metrics i) accuracy, ii) set size, iii) AUROC, and iv) AUARC respectively. In each figure, we show the performance of models across five datasets in LLM benchmark. Each figure represents the effect of model scale (number of parameters) in its performance across different uncertainty metrics.}
    \label{fig:spider_ap_llm}
\end{figure*}

\begin{figure*}[h!]
    \centering
    \makebox[\textwidth]{
        \begin{tabular}{cccc}
            \includegraphics[width=0.40\textwidth]{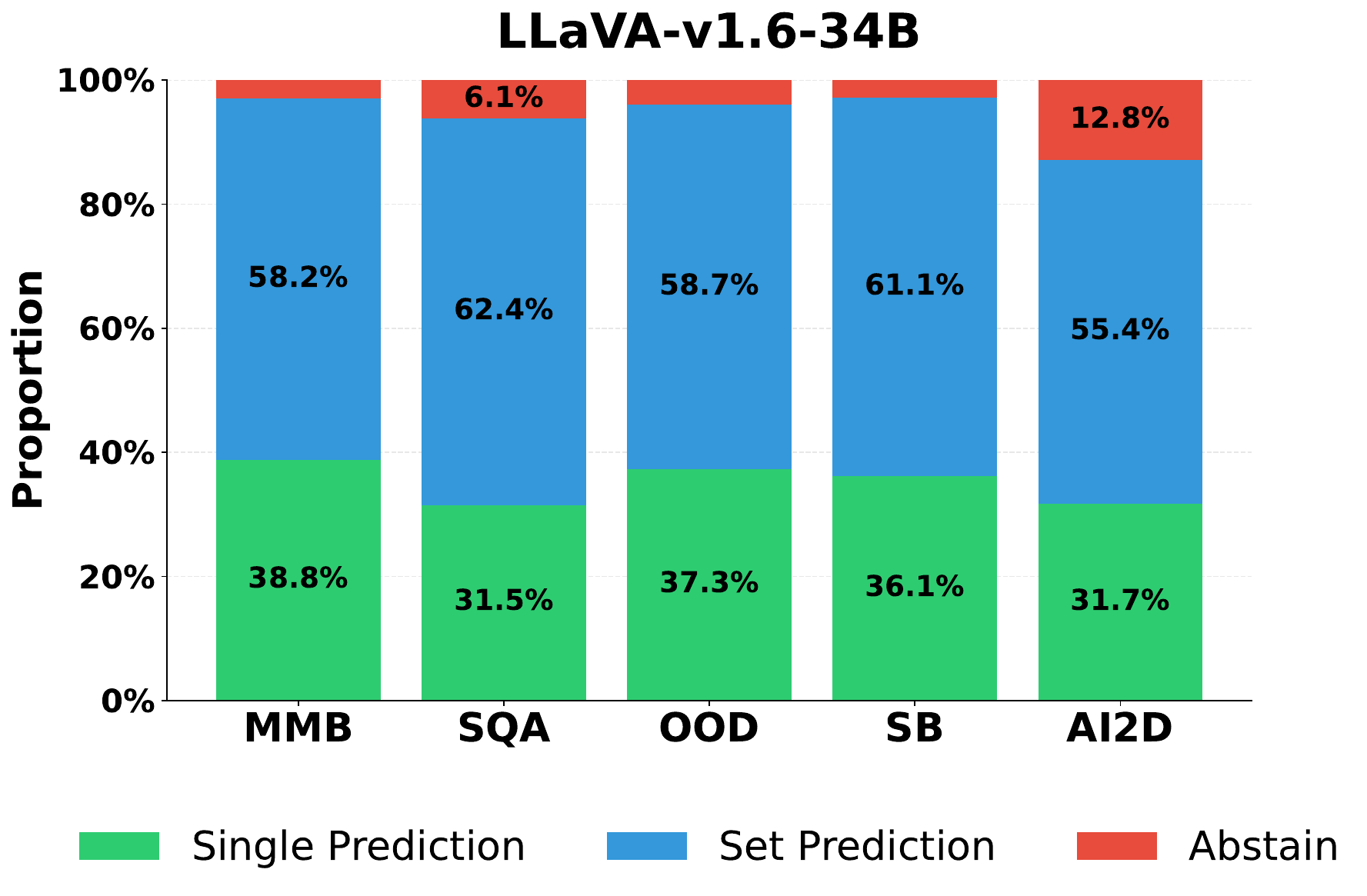} &
            \includegraphics[width=0.40\textwidth]{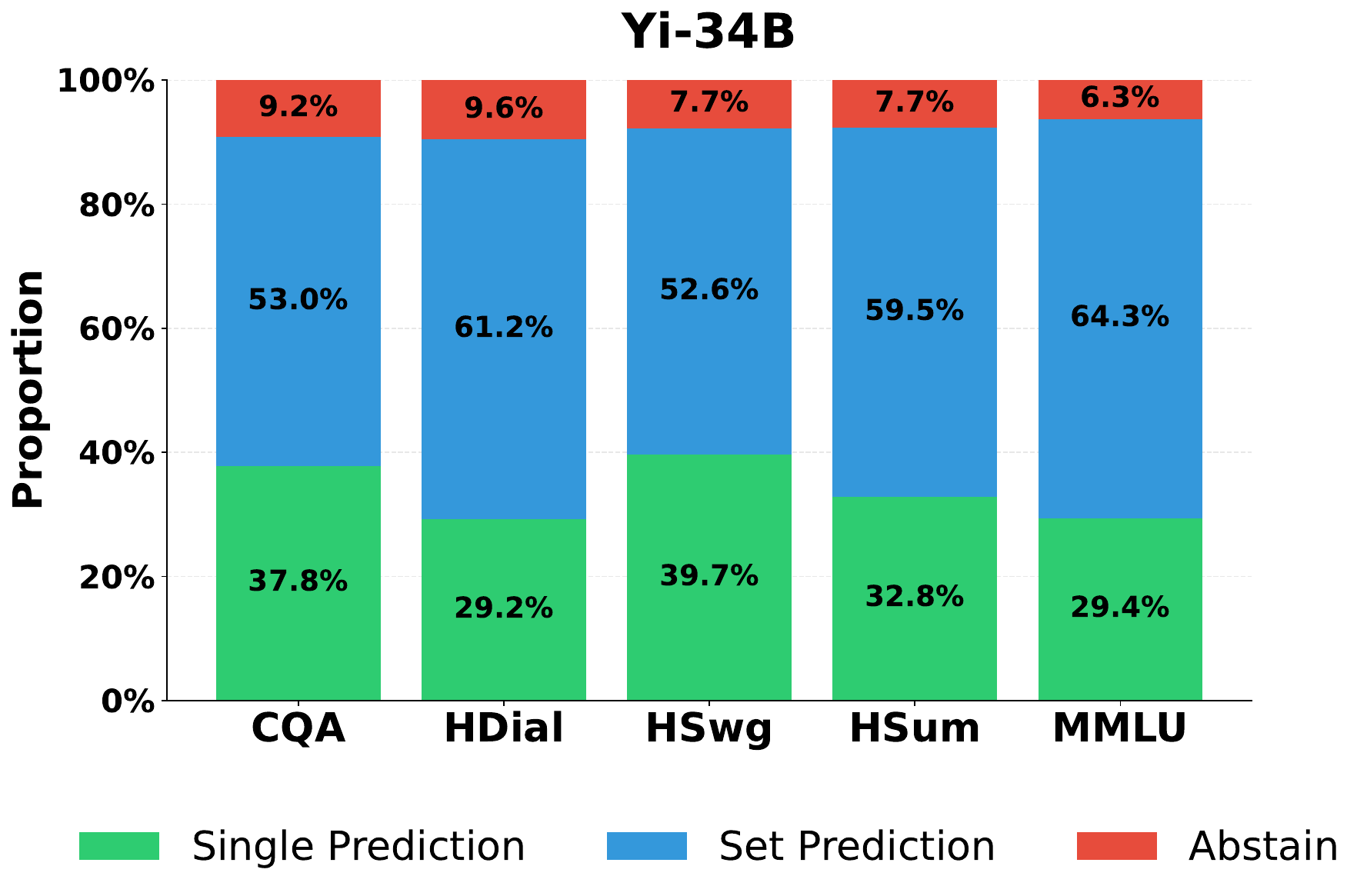} \\
        \end{tabular}
    }
    \caption{Distribution of CAP's prediction types for LLaVA-1.6-34B (VLM) and Yi-34B (LLM). The model's responses are categorized into single predictions (confident single answers), set predictions (multiple possible answers), and abstentions (declining to answer).}
    \label{fig:distribution}
\end{figure*}

\subsection{Size Distribution of Predicted Set}

The distribution of prediction types provides insights into our model's decision-making behavior across different vision-language tasks. As shown in \autoref{fig:distribution}, LLaVA-1.6-34B demonstrates a preference for set predictions across all benchmarks, with set prediction rates ranging from 55.4\% (AI2D) to 62.4\% (ScienceQA). This suggests that the model frequently identifies multiple plausible answers rather than committing to a single prediction due to underlying uncertainty in VLM response. Single predictions constitute a substantial portion of responses, varying between 31.5\% to 38.8\%, indicating scenarios where the model exhibits high confidence in a unique answer. The abstention rates show notable variation across datasets, from 6.1\% on ScienceQA to 12.8\% on AI2D, reflecting the model's ability to recognize and acknowledge uncertainty in different visual reasoning contexts. This trend repeats for Yi-34B LLM across five different tasks as well. This distribution pattern demonstrates that our selective prediction approach effectively captures different levels of model uncertainty, allowing for more nuanced and reliable responses across diverse vision-language tasks.

\end{document}